%% file: DCFlow_arxiv.tex
\documentclass{arxiv_article}

\usepackage{amsmath,amssymb,amsthm}
\usepackage{xspace}
\usepackage{fancyvrb}
\usepackage{array}
\usepackage{tabularx}
\usepackage{ragged2e}
\usepackage{enumitem}

\usepackage{booktabs}
\usepackage{colortbl}
\usepackage{mathtools}
\usepackage{xcolor}
\usepackage{algorithm}
\usepackage{algpseudocode}

\usepackage{adjustbox}
\usepackage{nicematrix}
\usepackage{makecell}
\usepackage{graphicx}
\usepackage{subfigure}
\usepackage{listings} % Added for prompt display

\theoremstyle{plain}

\usepackage{wrapfig}
\usepackage{float}

\usepackage{titletoc}

\input{math_commands.tex}

\usepackage{microtype}
\usepackage{graphicx}
\graphicspath{{figures/}{./}}
\usepackage{subcaption}
\usepackage{booktabs}
\usepackage{bm}
\usepackage{multirow}
\usepackage[table]{xcolor}
\definecolor{blue}{RGB}{0, 70, 174}

\usepackage{amsmath,amssymb,mathtools,amsthm}
\usepackage{array}
\usepackage{enumitem}

\definecolor{lightpink}{rgb}{1,0.9,0.9}
\definecolor{lightblue}{rgb}{0.8,0.9,1.0}
\definecolor{deepblue}{rgb}{0,0,0.5}
\definecolor{deepgreen}{rgb}{0,0.5,0}

\newcolumntype{C}[1]{>{\centering\arraybackslash}p{#1}}
\definecolor{oatpink}{HTML}{FDE2E4}
\definecolor{mypurple}{HTML}{4A148C}

\renewcommand\affiliationformat[2][]{ {\footnotesize $^{#1}$#2} }
\title{DC-Flow: Target-Conditioned Cyclic Peptide Design via Delayed-Coupling Flow Matching}
\author[1,2,3]{Yitian Wang}
\author[1]{Minjie Cheng}
\author[2]{Yaning Cui}
\author[2,3]{Maohua Yang}
\author[2,3]{Dongdong Wang\thanks{Corresponding authors. Email: wangdd@dp.tech, hongtengxu@ruc.edu.cn}}
\author[1,4,5]{Hongteng Xu\footnotemark[1]}

\affiliation[1]{Gaoling School of Artificial Intelligence, Renmin University of China}
\affiliation[2]{DP Technology\qquad $^{3}$Cyclon Therapeutics}
\affiliation[4]{Beijing Key Laboratory of Research on Large Models and Intelligent Governance}
\affiliation[5]{Engineering Research Center of Next-Generation Intelligent Search and Recommendation, MOE}

\newcommand{\method}{\textsc{DC-Flow}}
\newcommand{\Rec}{\mathcal{R}}
\newcommand{\B}{\mathcal{B}}
\newcommand{\Sseq}{\mathcal{S}}

\newcommand{\tbb}{\tau_{\mathrm{b}}}
\newcommand{\tseq}{\tau_{\mathrm{s}}}
\newcommand{\tsc}{\tau_{\mathrm{c}}}

\theoremstyle{plain}
\theoremstyle{remark}

\begin{document}

\abstract{Target-conditioned cyclic peptide design requires coordinating interfacial side chains with macrocyclic closure. 
Existing end-to-end generators concurrently update geometric and chemical variables from noise, assigning residues to disordered coordinates. 
Meanwhile, predictor-assisted workflows often get stuck in local traps because they lack feedback in the propose-and-filter stage and ignore 3D structural proposals during sequence search. 
To overcome the drawbacks of these two paradigms, we propose \method{}, a staged generative framework built around \textbf{D}elayed-\textbf{C}oupling Flow Matching. 
In principle, \method{} postpones sequence and side-chain activation until a target-conditioned backbone takes shape, resolving structural ambiguity before chemical assignment. 
Then, it feeds delayed activations back into the backbone, achieving co-evolution of backbone geometry and emerging residue identities by coupling their flows. 
This mechanism enables more reliable macrocyclic closure and favorable interface packing. 
For predictor-assisted workflows, \method{} supplies 3D proposals with in-trajectory guidance, enabling cascaded test-time scaling to prune redundant backbones and terminate unpromising refinement early. 
Extensive experiments show that \method{} improves cyclization fidelity, diversity, and interface energetics, substantially increasing the yield of confident cyclic peptide binders under a fixed final-candidate budget.
Our code is available at \url{https://github.com/mianqiu00/DCFlow}.}
\maketitle

\section{Introduction} 

As head-to-tail macrocycles, cyclic peptides can engage broad, shallow, polar protein-protein interaction (PPI) interfaces that often lack pockets suitable for small molecules~\citep{dougherty2019penetration,vinogradov2019macrocyclic}. 
Well-designed macrocycles can protect against protease attack, reduce access to protease-susceptible backbone states~\citep{malde2019proteinbound}, and exhibit molecular chameleonicity.\footnote{For some cyclic peptides, their polar groups remain exposed in water but fold to hide their polarity in apolar media, facilitating membrane permeation~\citep{wang2021cyclicpermeability,linker2022polar}.} 
However, target-conditioned design of cyclic peptides remains challenging due to data scarcity.
High-resolution complex structures are severely limited compared with the vast space of peptide sequences, lengths, and cyclization modes~\citep{yang2026cpsea,rettie2025afcycdesign,rettie2025rfpeptides}. 
Moreover, minor changes near the cyclization site can shift the entire backbone, yielding multiple low-energy states and binding poses. 
Therefore, a target-conditioned design model must jointly solve target docking, backbone closure, sequence design, and side-chain packing under sparse, heterogeneous supervision~\citep{zhou2025cpsde,wang2026muco}.

For cyclic peptides, existing target-conditioned design methods generally fall into two paradigms, as shown in Figure~\ref{fig:paradigm} (a) and (b).
On one hand, full-atom generators such as PepFlow~\citep{li2024pepflow}, PepGLAD~\citep{kong2024pepglad}, DiffPepBuilder~\citep{wang2024diffpepbuilder}, and cyclic-specific models like APCyc~\citep{zhao2026apcyc} (adapted from dyMEAN~\citep{kong2023dymean}) offer high throughput by amortizing design into generative trajectories. 
However, the data scarcity issue limits their generalizability, and their outputs are often geometrically plausible but fail to recover confident interfaces under cofolding. 
On the other hand, predictor-assisted methods, e.g., RFPeptides~\citep{rettie2025rfpeptides}, AfCycDesign~\citep{rettie2025afcycdesign}, and HFGuidedDesign~\citep{hu2026hfguideddesign}, integrate structure prediction into the proposal, optimization, and filtering pipeline, enriching candidates with high self-consistency. 
While effective, repeated and exhaustive cofolding incurs prohibitive computational costs, severely restricting candidate exploration under a fixed compute budget.

To bridge the gap between direct generation and predictor-assisted design and overcome their drawbacks, we propose \method{}, a staged generative framework built on a delayed-coupling flow matching strategy. 
In principle, \method{} is based on a key fact: \textbf{while side-chain interactions supply critical binding affinity at the PPI interface, the backbone scaffold dictates their optimal spatial presentation}~\citep{clackson1995hotspot}.
% Therefore, predicting sequence and side chains before the backbone scaffold takes shape lacks essential geometric context and often leads to poor designs.
Motivated by this, delayed-coupling flow matching postpones sequence and side-chain flows until the backbone scaffold takes shape, improving cyclic peptide design by enriching essential geometric context.
As shown in Figure~\ref{fig:paradigm} (c), once each of the flows is activated, it couples with the backbone in a shared network, allowing backbone geometry and residue identities to co-evolve over time.
To complement cyclic-complex supervision with general structure-conditioned sequence statistics, we augment a frozen ProteinMPNN prior~\citep{dauparas2022proteinmpnn} with trainable residual modules. 
Crucially, this staged formulation naturally unlocks \emph{cascaded test-time scaling}, an adaptive search policy that selects diverse backbones mid-generation, guides sequence sampling along the trajectory, and prunes unpromising cofolding optimization early. 

We comprehensively evaluate \method{} on the LNR benchmark~\citep{tsaban2022harnessing} as a full-atom generator.
In addition, we test it on the design tasks in the published HFGuidedDesign targets~\citep{hu2026hfguideddesign} and a challenging TL1A trimer case~\citep{zwolak2022stable} as a cofolding proposal generator. 
We also provide systematic ablations of our asynchronous staging, prior integration, and test-time search strategy. 
Experiments show that \method{} can reliably discover cyclic peptide conformations with high prediction confidence and superior binding interfaces. 
It achieves the highest cyclization success on LNR, leads cofolding baselines in strict-hit rate and mean Rosetta $\Delta G_{\mathrm{sep}}$, and further boosts these metrics through cascaded search.

\begin{figure}[t]
\centering
\includegraphics[width=\textwidth]{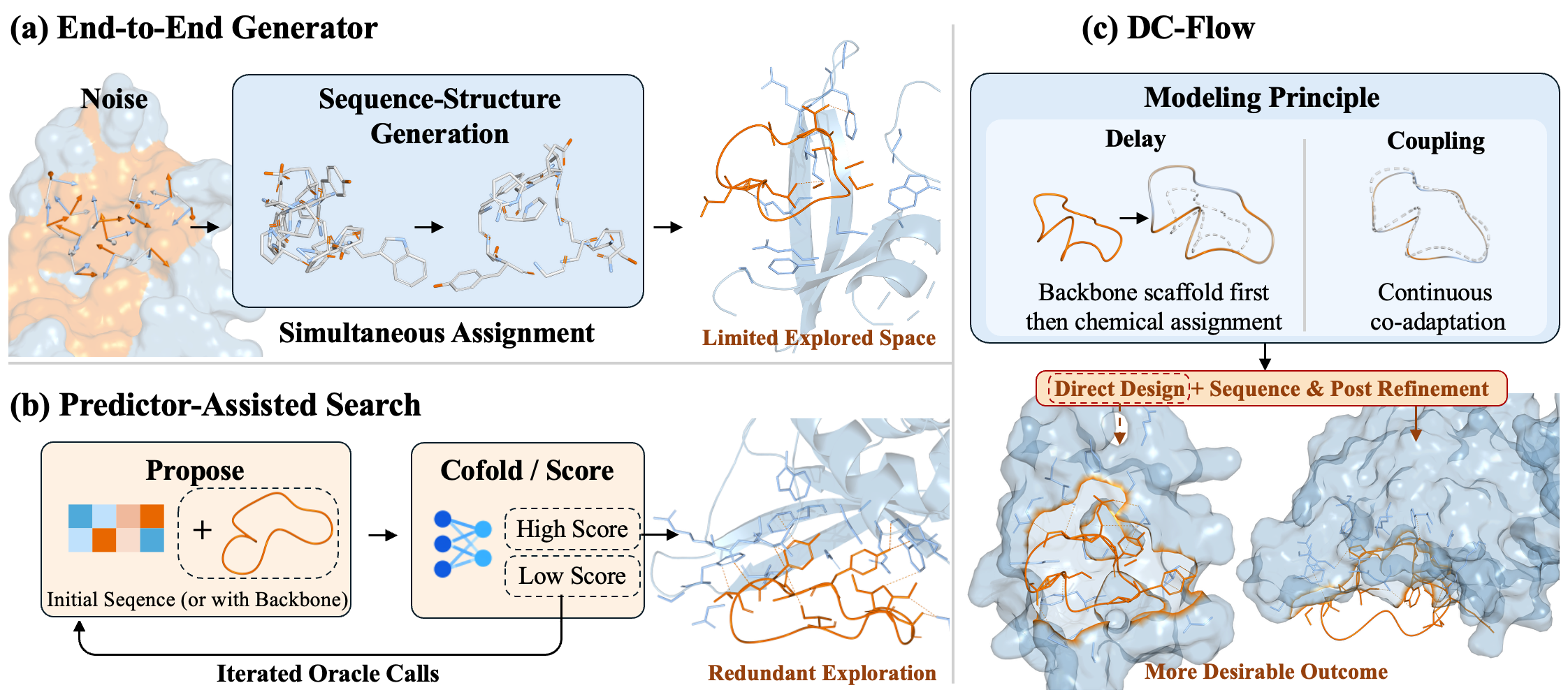}
\caption{\textbf{\method{} and Two Design Paradigms.} \textbf{(a)} End-to-end generators simultaneously update sequence and structure from noise, resulting in early disorder and limited explore space. 
\textbf{(b)} Predictor-assisted workflows rely on iterative oracle evaluations without dynamic structure-sequence feedback, resulting in redundant exploration. 
\textbf{(c)} \textsc{DC-Flow} resolves early structural ambiguity by delaying chemical assignment, then couples their flows to achieve continuous co-adaptation for both end-to-end design and predictor-assisted refinement.}
\label{fig:paradigm}
\vspace{-10pt}
\end{figure}

\vspace{-3pt}
\section{Related Work} 
\label{sec:related_work} 
\vspace{-3pt}

\textbf{Target-conditioned generative peptide design.} 
Generative approaches aim to generate protein sequences and structures conditioned on receptor pockets.
For example, APM~\citep{chen2025apm} models all-atom multi-chain complexes for conditional generation and inverse folding, and Genie~3~\citep{lin2026genie3} enhances atomistic $\mathrm{SE}(3)$-equivariant diffusion for efficient binder design. 
Focusing on peptides, PepGLAD~\citep{kong2024pepglad} embeds variable-size residues for receptor-conditioned diffusion. 
PepFlow~\citep{li2024pepflow} defines multiple flows over rigid frames, discrete residues, and side-chain torsions to support full-atom design. 
DiffPepBuilder~\citep{wang2024diffpepbuilder} utilizes an $\mathrm{SE}(3)$-equivariant diffusion model to design peptide sequences and conformations with disulfide stabilization. 
Focusing on cyclic peptides, APCyc~\citep{zhao2026apcyc} extends dyMEAN~\citep{kong2023dymean} by incorporating an expanded residue vocabulary for cyclization sites and applying Bayesian optimization. 
CPSea~\citep{yang2026cpsea} benchmarks these methods on cyclic peptide design, but it relies on pseudo-complexes rather than real ones. 
Notably, these generators simultaneously update conformations and sequences from scratch along a single trajectory, while \method{} stages the generation of backbone, sequence, and side chains in a delayed-coupling way.

\textbf{Predictor-guided design and cofolding pipelines.} 
A complementary paradigm incorporates structure predictors and cofolding models to guide candidate search, framing design as heuristic optimization over an energy landscape. 
At one end, AfCycDesign~\citep{rettie2025afcycdesign} introduces cyclic relative positional encodings into AlphaFold2~\citep{jumper2021alphafold} to establish a cyclic cofolding oracle, but its de novo design relies on time-consuming sequence hallucination and is prone to getting trapped in local optima.
At the other extreme, RFPeptides~\citep{rettie2025rfpeptides} adopts a propose-and-filter pipeline, where RFdiffusion~\citep{watson2023novo} proposes backbones, ProteinMPNN predicts sequences, and cofolding models~\citep{rettie2025afcycdesign} filter sequences. 
HFGuidedDesign~\citep{hu2026hfguideddesign} represents a middle ground by steering discrete sequence diffusion with online HighFold evaluations~\citep{zhang2024highfold}, but it operates purely in sequence space and relies on the external conformation predictors. 
\method{} bridges these approaches, generating coupled backbone-sequence candidates with in-trajectory guidance.

\textbf{Test-time scaling in biomolecular design.} 
Recently, Proteina-Complexa~\citep{didi2026complexa} combines a pretrained atomistic flow model with various searching and steering strategies, demonstrating that generative initializations benefit predictor-assisted sequence hallucination. 
In protein inverse folding, ProtInvTree~\citep{liu2026protinvtree} formulates sequence design as reward-guided tree search to increase search breadth and depth without retraining. 
MuCO~\citep{wang2026muco} addresses a sequence-conditioned cyclic peptide setting by decomposing cyclization into backbone generation, side-chain packing, and all-atom refinement. 
Building on this concept, \method{} proposes cascaded test-time scaling to concentrate expensive refinement compute on the most promising candidates.
% it selects diverse backbones mid-generation, branches guided sequences along the trajectory, and applies early pruning to downstream cofolding. 
% This policy concentrates expensive refinement compute exclusively on the most promising candidates.

\vspace{-3pt}
\section{Proposed Method}
\label{sec:method}
\vspace{-3pt}

\subsection{Problem Formulation and Modeling Principle}
\label{sec:problem_formulation}

In principle, we formulate the target-conditioned cyclic peptide design problem through the lens of conditional generative modeling, learning flow-based generators to capture the evolution of the \textit{design state} over time.
In particular, given a fixed receptor pocket $R = (S^R, \bm{X}^R)$ comprising its residue sequence $S^R$ and 3D coordinates $\bm{X}^R$, we aim to generate a cyclic peptide of length $L$ conforming to cyclization specification $\Gamma$ (e.g., head-to-tail or disulfide closure), summarized as conditions $\xi=(L,\Gamma)$. 
At time $t\in[0,1]$, the peptide design state is represented by $Y_t=(B_t, S_t, C_t)$, where $B_t$ denotes backbone scaffold, $S_t$ is a continuous carrier for residue identities over alphabet $\mathcal{A}$ yielding discrete sequence $S_1\in\mathcal{A}^L$, and $C_t$ describes side-chain conformations. 
Starting from an initial state sampled from an uninformative prior distribution, i.e., $Y_0\sim p_0$, the model with trainable parameters $\theta$ learns the conditional distribution $p_\theta(Y_1\mid R, \xi)$.

As aforementioned, at PPI interfaces, side-chain interactions determine critical binding affinity, but their spatial presentation is fundamentally dependent on backbone scaffold~\citep{clackson1995hotspot}. 
Therefore, generating all modalities (backbones, side chains, and sequences) simultaneously forces sequence and side chains to develop over disordered, unformed backbone scaffolds. 
In contrast, decoupling backbone generation with side-chain packing and sequence design restricts the exploration of the design space and thus leads to reduced search efficiency.

Motivated by the above analysis, \method{} decomposes the generative trajectory into three sequential but overlapping stages demarcated by activation thresholds $0=\tbb < \tseq < \tsc < 1$:
\begin{eqnarray}
%\begin{aligned}
p_\theta(Y_1 | R, \xi)\!=\!\!\int\!
\underbrace{p_\theta(B_{\tseq} | R, \xi)}_{\text{\scriptsize Backbone sampler}}
\underbrace{p_\theta(B_{\tsc}, S_{\tsc} | B_{\tseq}, R, \xi)}_{\text{\scriptsize Sequence generator}}
\underbrace{p_\theta(Y_1 | B_{\tseq}, B_{\tsc}, S_{\tsc}, R, \xi)}_{\text{\scriptsize Full-atom packer}}
\mathrm{d}B_{\tseq}\mathrm{d}B_{\tsc}\mathrm{d}S_{\tsc}.\!\!\!
%\end{aligned}
\label{eq:dc_factorization}
\end{eqnarray}
% \xu{explain $\tau$'s and $B_{\tau}$'s. BTW, is it reasonable to reuse $p_{\theta}$ on the right-hand side of the equation?}
Here, $\tau$'s mark the modality activation times, $B_\tau$ and $S_\tau$ denote the backbone and sequence states at activation time $t=\tau$. This formulation first establishes a receptor-conditioned backbone to resolve structural ambiguity before sequence assignment. 
It activates sequence and side chains over the formed scaffold while keeping the backbone flexible, co-evolving backbone and residue identities.

\begin{figure}[t]
\centering
\includegraphics[width=\textwidth]{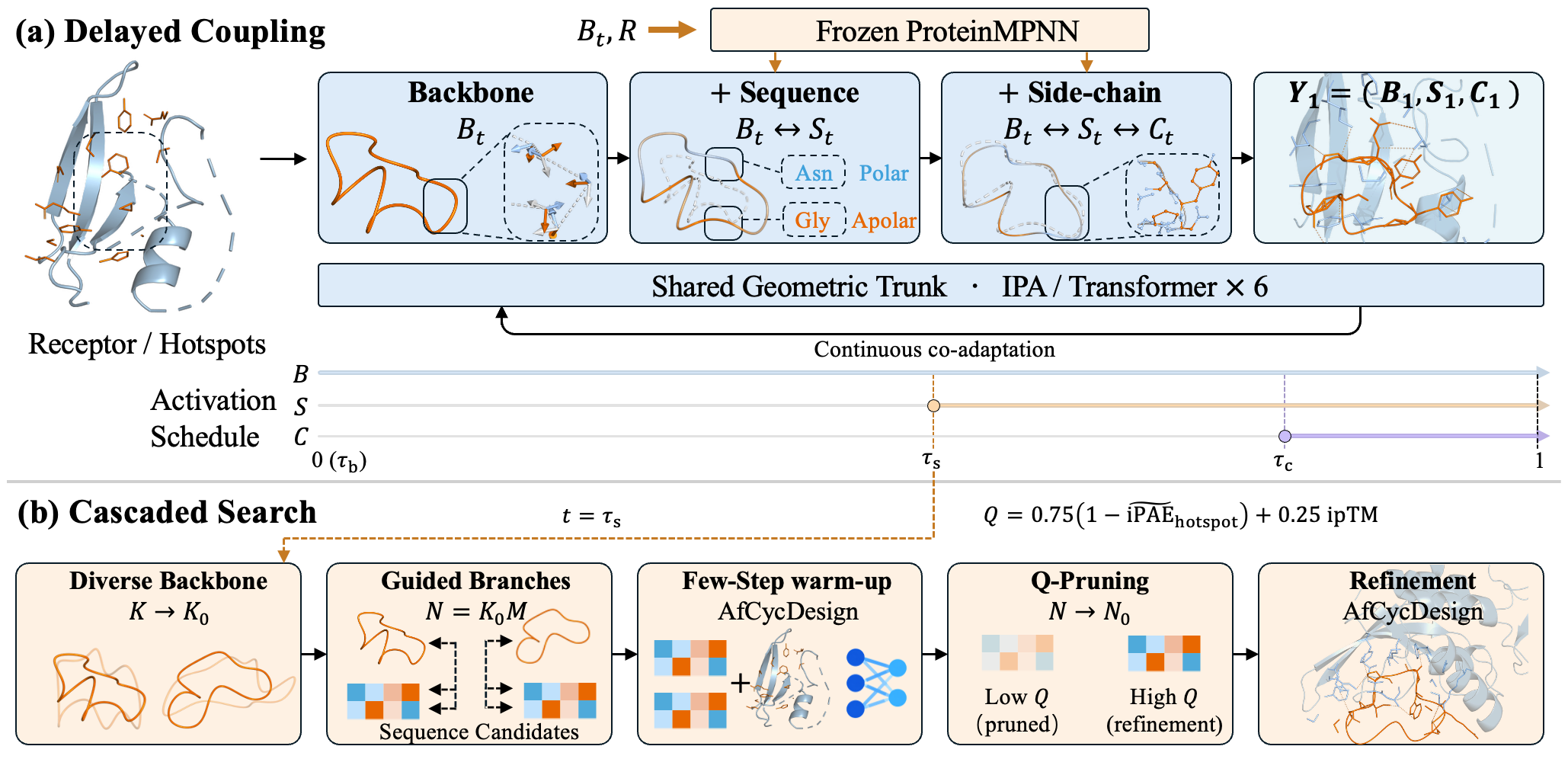}
\caption{\textbf{The overview of \method{}.} \textbf{(a) Delayed coupling:} Modality clocks activate backbone, sequence (augmented by ProteinMPNN prior), and side chains sequentially, enabling continuous co-adaptation in a shared network. 
\textbf{(b) Cascaded search:} At $\tseq$, backbones are pruned to diverse roots, each branching into $M$ oracle-guided continuations ($N=K_0 M$). After a few-step warm-up, candidates are screened by score $Q$ to restrict expensive refinement.}
\label{fig:method}
\vspace{-10pt}
\end{figure}

Figure~\ref{fig:method} provides an overview of \method{} across both standalone full-atom generation and cofolding pipelines, which contains two cohesive mechanisms: 
(i) delayed-coupling flow matching for staged backbone, sequence, and side-chain generation, conditioned on an external ProteinMPNN prior; 
and (ii) cascaded test-time scaling for efficient cofolding refinement.

\subsection{Delayed-Coupling Flow Matching}
\label{sec:delayed_coupling}

% Existing multimodal diffusion/flow models typically update backbone, sequence, and side-chain variables concurrently from $t=0$, forcing chemical identities to emerge over severely disordered coordinates. 
% \method{} instead introduces \emph{delayed coupling}. 
\method{} learns a flow-based model to generate backbones, sequences, and side chains asynchronously.
For modality $m\in\{\mathrm{b},\mathrm{s},\mathrm{c}\}
$, let $\tau_m$ denote the threshold for activating its generative flow. 
Accordingly, independent modality progress clocks are defined as $s_m(t) = \operatorname{clip}(\frac{t-\tau_m}{1-\tau_m},0,1)$, with activation boundaries $0 = \tbb < \tseq < \tsc < 1$ (by default, $0, 0.5, 0.8$), which leads to the proposed delayed-coupling mechanism that corresponds to two complementary design principles:
%\xu{I suggest use $\mathrm{b}$, $\mathrm{s}$, $\mathrm{c}$ in the whole paper and figures, for concise notations and reducing space.}

\textbf{Delaying: Resolving Structural Ambiguity.}
Before $\tseq$, the flow-based generator evolves only the backbone geometry $B_t$, while sequence inputs are replaced by uninformative mask tokens and side-chain inputs use a fixed zero-angle encoding. 
Under the reference-path assumptions in Appendix~\ref{app:delayed_proof1}, earlier backbone states are statistically less informative observations of the target structure $B_1$. 
Along this reference path, the residual sequence ambiguity, measured by conditional mutual information $D(t) \triangleq I(S_1; B_1 \mid B_t, R, \xi)$, is non-increasing as structural noise decreases, with $I(a;\, b \mid c) \triangleq H(a \mid c) - H(a \mid b, c)$ measuring the excess Bayes sequence risk.
Delaying sequence activation to $\tseq$ avoids conditioning amino-acid assignment in highly ambiguous states of the early trajectory.
% \xu{$D(t)$ is undefined. We can introduce the proposition here, leaving its proof in Appendix. (Wang: \triangleq defines the notation)} 

\textbf{Coupling: Continuous Mutual Co-Adaptation.}
Crucially, delaying sequence does \emph{not} imply a hard, decoupled two-stage pipeline (e.g., generating a rigid backbone and subsequently fixing it for inverse folding). 
Once sequence activates at $\tseq$, it enters the shared geometric trunk to jointly update the remaining trajectory alongside $B_t$. 
Formally, letting $R_B$ and $R_{BS}$ denote Bayes risks of predicting backbone target without and with sequence $S_t$, we have $R_B - R_{BS} \triangleq \Delta_t^{B \leftarrow S} \ge 0$ (Proposition~B.2, proof in Appendix~\ref{app:coupling_proof}), where the error reduction is strict whenever sequence alters the conditional mean of the target. 
This continuous coupling allows the backbone scaffold to remain flexible, actively adjusting its local geometry to accommodate specific side-chain packing requirements and emerging residue interactions. 
At $\tsc$, side-chain torsions $C_t$ activate to finalize all-atom packing.
% \xu{The theorem itself can be shown here. Leave its proof in Appendix.}

\textbf{Staged Training Objective.}
The generator is optimized end-to-end using flow matching and cross-entropy objectives. 
Let $\ell_B$ train backbone rigid frames and coordinates, $\ell_S$ denote sequence cross-entropy, $\ell_C$ supervise side-chain torsion angles, and $\ell_{\mathrm{cyc}}$ penalize cyclic closure errors (head-to-tail and disulfide constraints). 
With nonnegative loss weights $\lambda_B,\lambda_S,\lambda_C$ and indicator function $\mathbf 1[\cdot]$, modality-specific loss terms activate in accordance with their respective time thresholds. 
The grouped objective is:
\begin{equation}
\mathcal L = \mathbb E_{t\sim \mathcal{U}[0.01,0.99]}\!\left[
\lambda_B \ell_B
+ \mathbf 1[t\geq\tseq]\lambda_S\ell_S
+ \mathbf 1[t\geq\tsc]\lambda_C\ell_C
+ \ell_{\mathrm{cyc}}
\right].
\label{eq:dc_objective}
\end{equation}
% \xu{The definitions of the notations in the above equation should be given here. % 我觉得这里写的还是比较清楚吧，具体公式太长了放附录了也可以根据ref直接点过去}
The disulfide component of $\ell_{\mathrm{cyc}}$ is active only after side-chain activation; detailed mathematical definitions and decompositions of all individual loss terms are provided in Appendix~\ref{app:architecture_dynamics}.

\textbf{Leveraging ProteinMPNN as a Prior of Sequence.}
In both training and inference, \method{} incorporates a frozen ProteinMPNN~\citep{dauparas2022proteinmpnn} $F_\phi$ as a sequence prior, which successfully compensates for the scarcity of cyclic peptide complex structures.
In particular, for $t > \tseq$, the geometric trunk outputs residue features $h_{i,t}$, from which a sequence head predicts base logits $\ell^B_{i,t} = f_{\mathrm{AA}}(h_{i,t})$. 
Concurrently, $F_\phi$ evaluates the evolving intermediate backbone $B_t$ and receptor $R$, outputting log-probabilities $\ell^P_{i,t} = F_\phi(B_t; R)_i$. 
We center $\ell^P_{i,t}$ across alphabet $\mathcal{A}$ to form zero-mean residual logits $\widetilde\ell^P_{i,t}$. 
With a trainable scalar gate $g \in \mathbb R$ (initialized to $0.1$), the fused logits are
\begin{equation}
\ell_{i,t} = \ell^B_{i,t} + g\,\widetilde\ell^P_{i,t}, \qquad \text{where} \quad \widetilde\ell^P_{i,t,a} = \ell^P_{i,t,a} - \frac{1}{|\mathcal A|}\sideset{}{_{a'\in\mathcal A}}\sum\ell^P_{i,t,a'}.
\label{eq:dc_prior_fusion}
\end{equation}
Disallowed residues (e.g., non-cysteines at disulfide sites) are masked before softmax normalization.
\begin{itemize}[leftmargin=*,itemsep=0pt,topsep=0pt]
    \item \textbf{Training: Structural Transfer without Gradient Leakage.}
    During training, the sequence cross-entropy loss $\ell_S = -\frac{1}{L}\sum_{i=1}^L \log \pi_{i,t, a_i^\ast}$ supervises the fused distribution $\pi_{i,t} = \operatorname{softmax}(\ell_{i,t})$ against native residues $a_i^\ast$. 
    Crucially, evaluating ProteinMPNN on the noisy intermediate backbone $B_t$ rather than clean coordinates teaches the network sequence--structure compatibility throughout generation. 
    We apply a stop-gradient operator to $B_t$ at the input of $F_\phi$, ensuring that $\ell_S$ optimizes only \method{}'s native trunk, sequence head, and gate $g$ without differentiating through the external prior.
    \item \textbf{Inference: Dynamic In-Trajectory Guidance.}
    At inference time, residual fusion induces a product-of-experts distribution over admissible amino acids. 
    Suppressing indices $i,t$, let $\pi^B = \operatorname{softmax}(\ell^B)$ and $q^P = \operatorname{softmax}(\ell^P)$. 
    As shown in Appendix~\ref{app:mpnn_transfer_proof}, the fused categorical distribution satisfies
    \begin{equation}
    \pi(a) \propto \pi^B(a)[q^P(a)]^g,
    \label{eq:prior_product_form}
    \end{equation}
    directly scaling the injected structural log-odds by $g$.
\end{itemize}
\vspace{-3pt}
Unlike conventional pipelines (e.g., RFPeptides~\citep{rettie2025rfpeptides}) that apply ProteinMPNN as an isolated post-hoc step on fixed coordinates, \method{} embeds the prior \emph{inside the generative trajectory}. 
Discrete sequences sampled from $\pi$ feed back into the trunk at subsequent flow steps, steering continuous backbone relaxation and side-chain packing. 
This closes an active sequence--structure feedback loop during generation.

\vspace{-10pt}
\subsection{Cascaded Test-Time Scaling}
\label{sec:cascaded_search}

Although \method{} supports direct full-atom generation, cofolding~\citep{rettie2025afcycdesign} with generated proposals is still bottlenecked by expensive oracle evaluations and iterative refinement. 
We therefore introduce \emph{cascaded test-time scaling}, dynamically allocating inference compute across structural exploration, sequence branching, and selective refinement (Appendix~\ref{app:cascaded_search_theory}).

\begin{itemize}[leftmargin=1.3em,itemsep=0pt,topsep=0pt]
\item \textbf{Backbone Selection ($t=\tseq$).}
We initiate $K$ parallel backbone trajectories up to $\tseq$. 
Before activating sequence generation, we downsample to $K_0 < K$ diverse roots via kernel-based structural crowding (Appendix~\ref{app:backbone_selection}), pruning geometrically redundant backbones.
\item \textbf{Guided Sequence Branching ($t\in[\tseq, 1]$).}
Each retained root branches into $M$ independent sequence continuations ($N=K_0M$ trajectories total) while backbones continue evolving (Appendix~\ref{app:branching_checkpoints}). 
Concurrently, an external cyclic cofolding oracle periodically evaluates single-residue substitutions, biasing logits $\ell_{i,t}$ toward favorable hotspot interactions~\citep{hu2026hfguideddesign}. 
Through sequence--structure coupling, this guidance simultaneously steers subsequent backbone updates (Appendices~\ref{app:sequence_coverage} and~\ref{app:search_discovery}).
\item \textbf{Early-Pruned Cofolding Refinement.}
At $t=1$, the $N$ generated designs undergo a brief 5-step AfCycDesign warm-up. 
We then probabilistically retain $N_0 < N$ candidates weighted by an interface quality score
\begin{equation}
Q = 0.75\bigl(1-\widetilde{\mathrm{iPAE}}_{\mathrm{hotspot}}\bigr) + 0.25\,\mathrm{ipTM},
\label{eq:dc_search_quality}
\end{equation}
where $\widetilde{\mathrm{iPAE}}_{\mathrm{hotspot}}$ and $\mathrm{ipTM}$ denote the normalized hotspot interface PAE and interface TM-score, respectively. 
Only these $N_0$ survivors proceed to full refinement, subject to early stopping if the score plateaus for five consecutive soft-stage updates (Appendix~\ref{app:cascade_refinement}). 
This funnel concentrates expensive refinement on high-potential candidates, improving discovery efficiency under fixed final-candidate budget.
\end{itemize}

\vspace{-7pt}

\section{Experiments}
\label{sec:experiments}

\vspace{-5pt}

\subsection{Experimental Setup}
\label{sec:experimental_setup}

\textbf{Data Construction.}
We train \method{} and all direct full-atom baselines on the same
stereochemically filtered CPSea-derived corpus. The complete data sources,
quality-control criteria, and filtering statistics are provided in
Appendix~\ref{app:data}. Raw full-atom generation is evaluated on our
processed Large Non-Redundant (LNR) benchmark~\citep{tsaban2022harnessing},
which contains 93 receptor-peptide complexes.

To quantify generalization with respect to the training receptors, we use
MMseqs2~\citep{steinegger2017mmseqs2} to compare every LNR receptor against
all receptors in the training set. For a test receptor $\Rec$, let
\begin{equation}
    \rho(\Rec)
    \triangleq
    \sideset{}{_{\Rec'\in\mathcal D_{\mathrm{train}}}}\max
    \operatorname{SeqId}_{\mathrm{MMseqs2}}
    \left(
        \operatorname{seq}(\Rec),
        \operatorname{seq}(\Rec')
    \right),
    \label{eq:lnr_sequence_similarity}
\end{equation}
where $\mathcal D_{\mathrm{train}}$ denotes the set of training receptors,
$\operatorname{seq}(\Rec)$ is the amino-acid sequence of $\Rec$, and
$\operatorname{SeqId}_{\mathrm{MMseqs2}}$ is the sequence identity returned
by MMseqs2. We define 37 receptors with $\rho(\Rec)<0.50$ as the
out-of-distribution subset (\textsc{ood}) and the remaining 56 receptors
with $0.50\le\rho(\Rec)\le0.70$ as the in-distribution subset
(\textsc{id}). No LNR receptor has greater than 70\% sequence identity to a
training receptor. For every target, all methods receive the same receptor,
peptide-length condition, and generation budget.

The cofolding comparison uses the 12 published HFGuidedDesign targets,
which cover peptide lengths and receptor interfaces of varying
complexity~\citep{hu2026hfguideddesign}. Each method contributes the same
number of final candidates per target, and all candidates are evaluated
through a common cofolding and scoring pipeline.

\textbf{Baselines and Implementation.}
The direct full-atom baselines comprise
PepFlow~\citep{li2024pepflow},
PepGLAD~\citep{kong2024pepglad},
DiffPepBuilder~\citep{wang2024diffpepbuilder},
APM~\citep{chen2025apm},
Genie~3~\citep{lin2026genie3}, and
APCyc~\citep{zhao2026apcyc}.
All six methods are retrained using the same training corpus and data split
as \method{}, with method-specific architectures and sampling procedures
otherwise retained.
For the predictor-guided and coarse-to-refine comparison, we include
AfCycDesign~\citep{rettie2025afcycdesign} as cofolding predictor,
RFPeptides~\citep{rettie2025rfpeptides}, and
HFGuidedDesign~\citep{hu2026hfguideddesign} as proposal generator baselines.
We follow their released inference configurations and evaluate their final
designs with the same candidate budget and post-generation scoring pipeline
used for \method{}.
On LNR, each baseline generates 32 candidates per target; \method{} generates 32 candidates separately for head-to-tail (HT) and terminal disulfide (SS) cyclization. Generated complexes undergo a common closure-constrained Rosetta FastRelax protocol. On the 12-target cofolding benchmark, each configuration contributes 32 final candidates per target. Generation-time guidance, subsequent AfCycDesign optimization, and cascaded search are identified separately. Implementation details and evaluation protocols are provided in Appendix~\ref{app:implementation_metrics}.

\textbf{Evaluation Metrics.}
As a full-atom generator, we report five metrics: cyclization success (Cyc.), Rosetta total energy (Stability), sequence diversity ($D_{\mathrm{seq}}$), structural diversity ($D_{\mathrm{struct}}$), and Rosetta interface energy ($dG$)~\citep{zhao2026apcyc,alford2017rosetta}. The diversity scores measure the fractions of sequence and structure clusters among valid candidates. Higher success and diversity, and lower energies, are preferred. Cyclization criteria, validity filtering, clustering thresholds, and receptor-wise aggregation are specified in Appendix~\ref{app:full_atom_metrics}. 
As a cofolding proposal generator, we report ipTM, peptide pLDDT, interface and hotspot iPAE, composite confidence scores, and strict-hit rates~\citep{hu2026hfguideddesign}. The main comparison uses the interface-based composite $S$, whereas the search comparison uses its hotspot-based counterpart $S_{\mathrm{hotspot}}$. Rosetta evaluation additionally reports interface energy ($\Delta G_{\mathrm{sep}}$), buried surface area ($\Delta\mathrm{SASA}$), area-normalized energy ($100\Delta G_{\mathrm{sep}}/\Delta\mathrm{SASA}$), interface hydrogen bonds, buried unsatisfied polar groups, shape complementarity (SC), and packing quality (Packstat). Definitions, normalization, hit thresholds, and Mean/Best aggregation are given in Appendix~\ref{app:cofolding_metrics}.

\input{tables/main_lnr}

\subsection{Main Results}
\label{sec:main_results}

\textbf{Direct full-atom generation on LNR.}
As shown in Table~\ref{tab:lnr_main}, \method{} combines high cyclization success with complementary sequence
and conformational diversity under the common relaxation protocol.
The HT configuration achieves the most favorable interface energy and
highest sequence diversity on OOD receptors, supporting its applicability
beyond training-related targets. The SS configuration achieves the
highest structural diversity on both splits, although its sequence
diversity is lower. The All/OOD/ID results in
Appendix~\ref{app:experiment} confirm this topology-dependent balance:
HT favors sequence exploration and interface energetics, while SS
retains broader conformational variation. These results support direct
generation as a source of diverse cyclic-complex hypotheses for
downstream evaluation.

\textbf{Cofolding-based comparison.}
In the predictor-assisted workflow, \method{} achieves the strongest
Mean and Best confidence scores among the compared pipelines.
Its strict-hit rate reaches 36.7\%, compared with 33.3\% for
HFGuidedDesign and 29.2\% for RFPeptides.
The guidance and refinement configurations in
Appendix~\ref{app:experiment} clarify this result: unrefined proposals
remain substantially less confident, and the combination of
generation-time guidance and subsequent optimization performs best.
This supports the complementary role of predictor assistance in
turning the generator's geometry-conditioned sequence proposals into
confident predicted complexes.

\input{tables/main_cofold.tex}

\begin{figure}[ht]
\centering
\includegraphics[width=\textwidth]{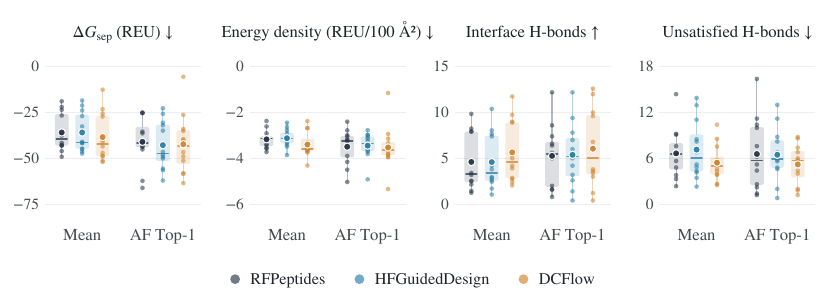}
\caption{\textbf{Target-wise Rosetta interface quality across 12 targets.}
Mean averages four sampled candidates per target; AF Top-1 uses the
AF-selected candidate. Each candidate's scores are averaged over five
models. Rounded bands show interquartile ranges, with median lines and
1.5-IQR whiskers; dots denote targets and larger circles denote
across-target means. Energy density is
$100\Delta G_{\mathrm{sep}}/\Delta\mathrm{SASA}$.
Complete metrics are reported in Table~\ref{tab:app_cofold_rosetta}.}
%\xu{The subfigures should be in the same shape and size.} % Resolved}
\label{fig:cofold_rosetta}
\end{figure}

Figure~\ref{fig:cofold_rosetta} provides a physical assessment of these
outputs. \method{} achieves more favorable energy per buried interface
area, more interface hydrogen bonds, and fewer buried unsatisfied
polar groups in both the Mean and AF Top-1 summaries
(Table~\ref{tab:app_cofold_rosetta}). The improvements therefore extend
beyond confidence scores to complementary measures of interface quality.

\vspace{-5pt}
\subsection{Ablation Study and Analysis of Test-Time Scaling}
\label{sec:ablation}
\vspace{-5pt}

We examine the three generative mechanisms through controlled ablations, with a separately trained model for each setting, retaining the receptor conditions, sampling-step budget, and candidate counts:
\begin{itemize}[leftmargin=1.3em,itemsep=2pt,topsep=2pt]
    \item \textbf{Sequence activation schedule} varies $\tseq\in\{0,0.2,0.5\}$ while retaining $\tsc=0.8$, assessing how the timing of chemical assignment balances early backbone disambiguation against conformational exploration; $\tseq=0$ corresponds to fully synchronous generation from noise.
    \item \textbf{w/o coupling} freezes the backbone scaffold at $B_{\tseq}$ for all $t\geq\tseq$, mimicking conventional decoupled two-stage pipelines (i.e., fixed-scaffold inverse folding) to quantify the benefit of continuous mutual co-adaptation.
    \item \textbf{w/o ProteinMPNN prior} disables the external prior by setting $g=0$ in \eqref{eq:dc_prior_fusion}, isolating the generator's native sequence trunk to evaluate how effectively the structural prior compensates for complex-data scarcity.
\end{itemize}

\input{tables/ablation_lnr.tex}

\textbf{Continued coupling and the structural prior.}
Freezing the backbone causes the largest energetic deterioration:
for OOD HT designs, interface energy worsens from $-28.37$ to
$-4.54$ REU while sequence and structural diversity remain nearly
unchanged. Backbone adaptability therefore contributes to chemical
compatibility beyond maintaining candidate diversity.
Removing the ProteinMPNN prior also weakens overall interface energies
for both topologies, although its diversity effects vary
(Appendix~\ref{app:ablation}). This is consistent with a structural prior
that favors compatible residue choices without uniformly increasing
the number of sequence modes.

\textbf{Delayed activation balances exploration and compatibility.}
Simultaneous backbone--sequence activation weakens HT interface
energetics, particularly on OOD targets, whereas the SS response is
mixed. Earlier activation at $\tseq=0.2$ increases overall sequence
and structural diversity relative to the default
(Appendix~\ref{app:ablation}). The default $\tseq=0.5$ instead gives
the highest overall cyclization success for both topologies and
the best HT energy scores. Thus the schedule controls a trade-off
between exploration and geometric compatibility; the benefit of
delayed coupling cannot be reduced to maximizing diversity alone.

\textbf{Cascaded search enriches the final candidate pool.}
We next evaluate the complete search policy with the generator and
guidance/refinement backends unchanged. With 32 final candidates per
target, search raises the strict-hit rate from 32.6\% to 75.5\%
(Figure~\ref{fig:cofold_scaling}). Mean and Best both improve, with
larger changes in Mean, showing that search enriches the returned
candidate population as well as improving the strongest designs.
The additional results in Appendix~\ref{app:ablation} show more
favorable Rosetta energies and packing, alongside weaker
hydrogen-bond metrics. These findings support the cascade as an integrated
policy for allocating exploration and refinement, with its clearest
benefit in the yield of confident predicted interfaces.

\begin{figure}[!ht]
\centering
\includegraphics[width=\textwidth]{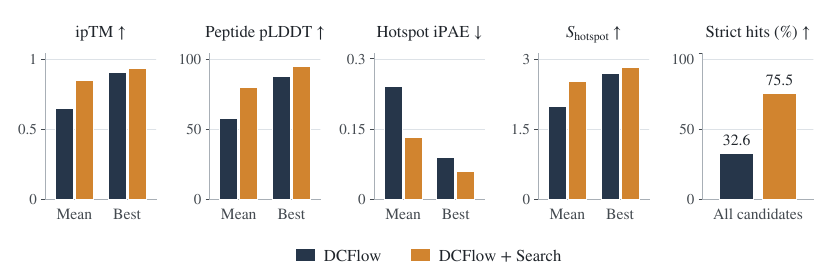}
\caption{\textbf{Effect of inference-time search on cofolding.}
Both \method{} configurations use guidance and AfCycDesign optimization.
Search improves Mean and Best on all four confidence metrics and
increases the strict-hit rate from 32.6\% to 75.5\%.
Best averages target-wise Top-1 candidates selected by
$S_{\mathrm{hotspot}}$; strict hits are measured over all 384 candidates.
Detailed results appear in Table~\ref{tab:app_scaling_main}.}
%\xu{The title of each subfigure should be in one line, for a better layout.} % Resolved}
\label{fig:cofold_scaling}
\end{figure}

\section{Conclusion and Future Work}
\label{conclusion}

We presented \method{}, a target-conditioned cyclic peptide design framework combining delayed coupling, a frozen ProteinMPNN residual prior, and cascaded test-time search. Its staged trajectory introduces sequence and side-chain information after initial backbone formation while preserving continuous co-adaptation, supporting both direct full-atom generation and predictor-assisted design. Experiments show a favorable balance of cyclization success, diversity, and interface energetics, while integrated search enriches the final candidate pool with confident predicted complexes.

\textbf{Limitations and Future Work.}
Despite geometric filtering, CPSea-derived supervision remains largely synthetic and may bias learned conformational and interface preferences. Reliance on AF2-Multimer confidence for guidance and evaluation may also favor predictor-specific solutions without corresponding improvements in experimental performance. Moreover, the current cofolding pipeline does not enforce terminal disulfide closure, limiting the use of \method{}'s SS design capability in this workflow. We plan to pair generated complex hypotheses with wet-lab measurements in an iterative design--experiment loop, using experimental feedback to guide dataset curation and model updates. This aims to improve the reliability of direct full-atom design and reduce dependence on costly cofolding refinement, accelerating the workflow as experimental evidence accumulates.

\clearpage
\input{declarations}

\bibliography{DCFlow}

\clearpage
\appendix
\startcontents[appendix]

\section*{Appendix Contents}
\begingroup
  \hypersetup{linkcolor=black}
  \setlength{\parskip}{4pt}
  \printcontents[appendix]{}{1}{\setcounter{tocdepth}{2}}
\endgroup

\clearpage

\section{Notation and Modeling Assumptions}
\label{app:notation}
\input{tables/app_notation}

\clearpage
\section{Theory of Delayed-Coupling Flow Matching}
\label{app:delayed_coupling_theory}
\input{tables/app_math}

\clearpage
\section{DC-Flow Formulation and Architecture}
\label{app:architecture}
\input{tables/app_architecture.tex}

\clearpage
\section{CPSea Quality Filtering}
\label{app:data}
\input{tables/app_data}

\clearpage
\section{Detailed Implementation and Evaluation Metrics}
\label{app:implementation}
\input{tables/app_implementation.tex}

\clearpage
\section{Additional Experiment Results}
\label{app:experiment}
\input{tables/app_lnr}
\input{tables/app_cofold}

\clearpage
\section{Additional Ablation Results}
\label{app:ablation}
\input{tables/app_ablation_lnr}
\input{tables/app_cofold_scaling}

\clearpage
\section{Challenging TL1A-Trimer Case}
\label{sec:tl1a_case}
\input{tables/app_tl1a.tex}

\clearpage
\section{Visualization}
\label{sec:visualization}

\begin{figure}[!ht]
    \centering
    \includegraphics[width=0.85\linewidth]{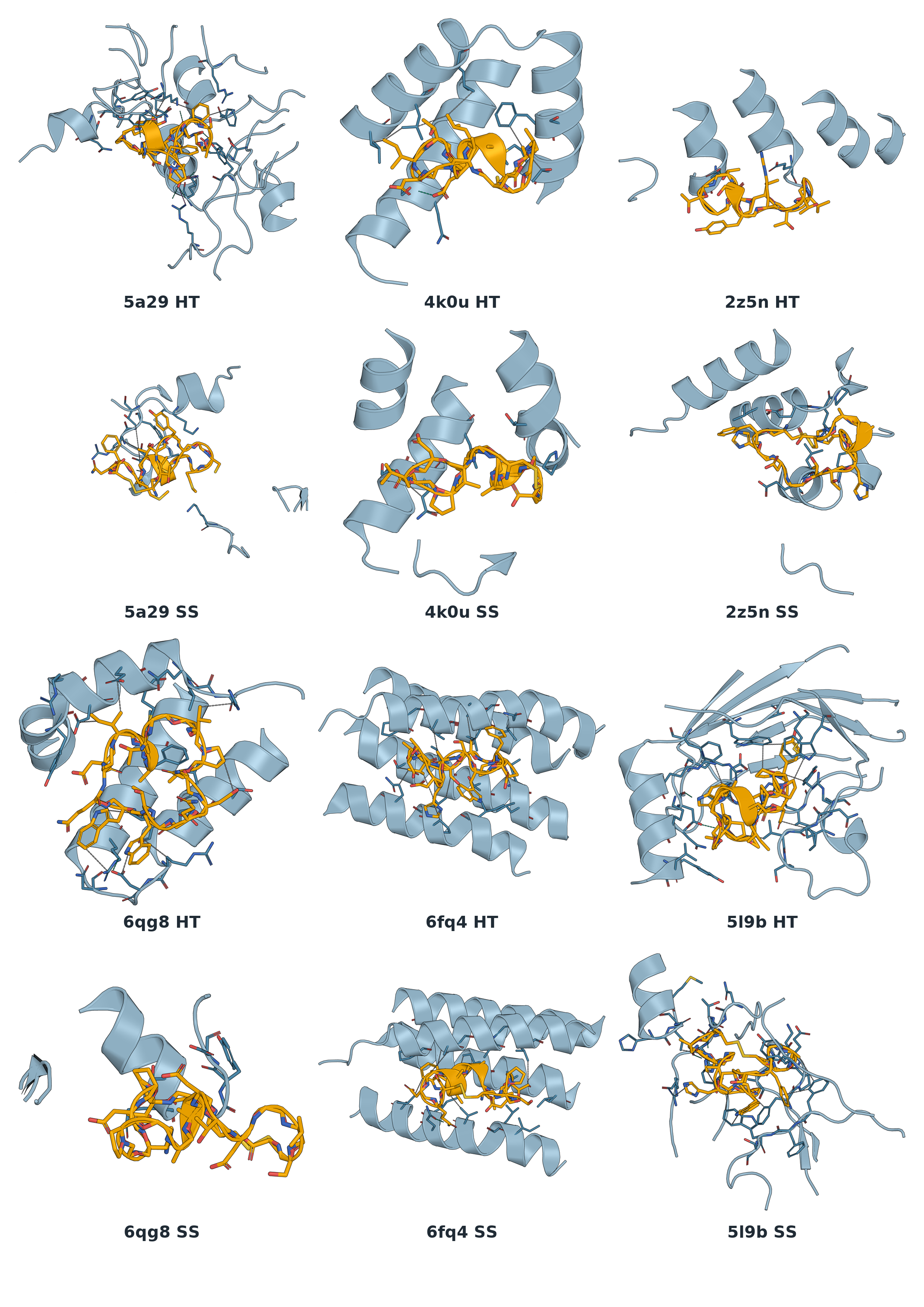}
    \caption{\textbf{Cyclic peptide designs generated directly by \method{} across 6 targets and 2 cyclization modes in LNR.} Local interface views show the target in light blue and the designed peptide in orange. Dashed lines indicate hydrogen-bond/polar contacts (green), salt bridges (magenta), and hydrophobic contacts (gray), inferred from relaxed model geometry.}
    \label{fig:app_visulization1}
\end{figure}

\clearpage
\begin{figure}[!ht]
    \centering
    \includegraphics[width=0.85\linewidth]{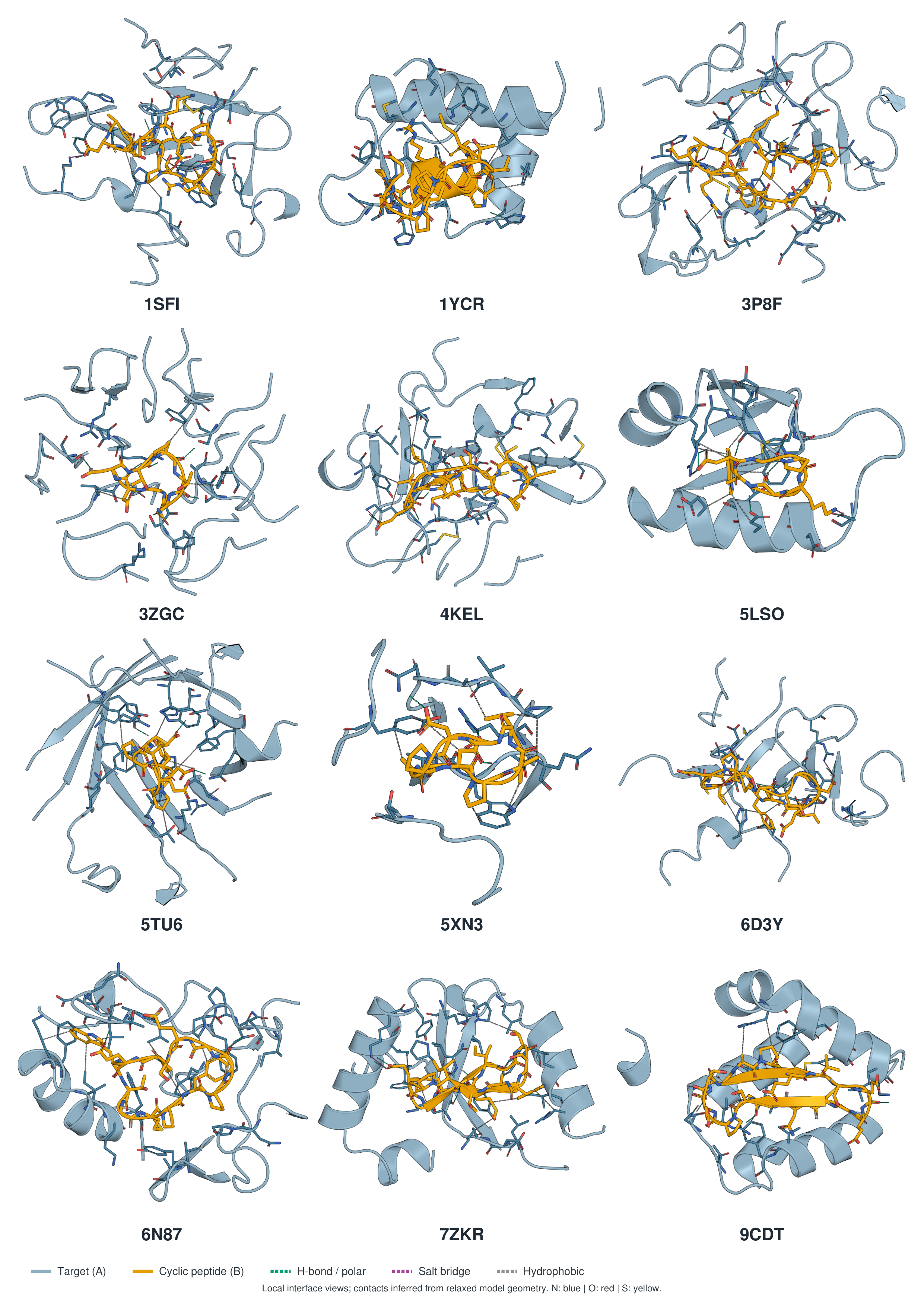}
    \caption{\textbf{Cyclic peptide designs generated by \method{} (Search) across 12 targets in HFGuidedDesign test set.} Local interface views show the target in light blue and the designed peptide in orange. Dashed lines indicate hydrogen-bond/polar contacts (green), salt bridges (magenta), and hydrophobic contacts (gray), inferred from relaxed model geometry.}
    \label{fig:app_visulization2}
\end{figure}

\end{document}

%% file: math_commands.tex
\usepackage{amsmath,amsfonts,bm}

\def\eqref#1{equation~\ref{#1}}
\def\Eqref#1{Equation~\ref{#1}}
\def\1{\bm{1}}

\DeclareMathAlphabet{\mathsfit}{\encodingdefault}{\sfdefault}{m}{sl}
\SetMathAlphabet{\mathsfit}{bold}{\encodingdefault}{\sfdefault}{bx}{n}

%% file: tables/main_lnr.tex
\begin{table*}[ht]
\centering
\small

\caption{\textbf{Direct full-atom cyclic peptide generation on LNR.} Evaluated across in-distribution (ID) and out-of-distribution (OOD) targets. Both \method{} variants attain 100\% cyclization success on ID targets. \method{}(SS) achieves the highest closure fidelity and structural diversity ($D_{\text{struct}}$) across splits, whereas \method{}(HT) provides superior binding energetics (Rosetta dG) and sequence exploration ($D_{\text{seq}}$), highlighting their complementary topological profiles.}
\label{tab:lnr_main}

\setlength{\tabcolsep}{3pt}
\begin{tabular*}{\textwidth}{@{\extracolsep{\fill}} l *{5}{rr} @{}}
\toprule
\multirow{2}{*}{Method}
& \multicolumn{2}{c}{Cyc. $\uparrow$}
& \multicolumn{2}{c}{Stability $\downarrow$}
& \multicolumn{2}{c}{$D_{\mathrm{seq}}\uparrow$}
& \multicolumn{2}{c}{$D_{\mathrm{struct}}\uparrow$}
& \multicolumn{2}{c}{Rosetta $dG\downarrow$} \\
\cmidrule(lr){2-3}
\cmidrule(lr){4-5}
\cmidrule(lr){6-7}
\cmidrule(lr){8-9}
\cmidrule(lr){10-11}
& \textsc{ood} & \textsc{id}
& \textsc{ood} & \textsc{id}
& \textsc{ood} & \textsc{id}
& \textsc{ood} & \textsc{id}
& \textsc{ood} & \textsc{id} \\
\midrule

PepFlow
& 92.9 & 94.6
& \textbf{-961.10} & -590.27
& 0.686 & 0.619
& 0.948 & 0.670
& -24.78 & -17.24 \\

PepGLAD
& \underline{97.5} & 87.4
& -720.20 & -579.45
& \underline{0.895} & \textbf{0.872}
& 0.932 & \underline{0.817}
& -15.59 & -13.50 \\

DiffPepBuilder$^{\dagger}$
& 92.2 & 94.6
& -939.25 & -574.79
& 0.446 & 0.286
& \underline{0.979} & 0.732
& -20.66 & -18.05 \\

APM
& 97.3 & \textbf{100.0}
& -869.51 & -497.36
& 0.173 & 0.120
& 0.938 & 0.754
& \underline{-25.06} & -16.50 \\

Genie~3$^{\dagger}$
& 83.6 & 81.6
& -726.29 & -399.28
& 0.253 & 0.242
& 0.823 & 0.796
& -14.07 & -10.16 \\

APCyc
& 95.4 & 73.9
& -713.21 & \textbf{-616.07}
& 0.867 & 0.831
& 0.849 & 0.705
& -22.79 & \textbf{-19.91} \\

\midrule

\textbf{\method (HT)}
& \underline{97.5} & \textbf{100.0}
& \underline{-940.05} & \underline{-604.28}
& \textbf{0.940} & \underline{0.845}
& 0.934 & 0.459
& \textbf{-28.37} & \underline{-19.14} \\

\textbf{\method (SS)}
& \textbf{97.6} & \textbf{100.0}
& -807.36 & -509.06
& 0.668 & 0.205
& \textbf{0.991} & \textbf{0.943}
& -21.17 & -15.39 \\

\bottomrule
\end{tabular*}%
\par\vspace{0.35em}
\parbox{\textwidth}{\footnotesize
$^{\dagger}$ DiffPepBuilder and Genie~3 require the Rosetta side-chain packer.
}

\vspace{-3pt}
\end{table*}

%% file: tables/main_cofold.tex
\begin{table}[!ht]
\centering
\caption{\textbf{Cofolding performance across 12 targets.}
All methods use AFCyc optimization; HFGuidedDesign and \method{} also use
generation-time guidance. Mean summarizes all 384 candidates; Best averages
the target-wise Top-1 candidates selected by composite score.}
\label{tab:cofold_main}
\begingroup
\small
\setlength{\tabcolsep}{3pt}
\renewcommand{\arraystretch}{1.10}
\begin{tabular*}{\textwidth}{@{\extracolsep{\fill}}llrrrrrr@{}}
\toprule
Method & Summary
& ipTM $\uparrow$
& \shortstack{Peptide\\pLDDT $\uparrow$}
& \shortstack{Interface\\iPAE $\downarrow$}
& \shortstack{Hotspot\\iPAE $\downarrow$}
& \shortstack{Composite\\$S\uparrow$}
& \shortstack{Strict hits\\(\%) $\uparrow$} \\
\midrule
\multirow{2}{*}{RFPeptides}
& Mean & 0.666 & 61.21 & \underline{0.334} & \underline{0.234} & 1.944
& \multirow{2}{*}{29.2} \\
& Best & 0.877 & 84.90 & 0.170 & 0.115 & 2.556 & \\
\addlinespace[3pt]
\multirow{2}{*}{HFGuidedDesign}
& Mean & \underline{0.674} & \underline{61.60} & \underline{0.334}
& \underline{0.234} & \underline{1.957} & \multirow{2}{*}{\underline{33.3}} \\
& Best & \underline{0.889} & \underline{86.57} & \underline{0.161}
& \underline{0.104} & \underline{2.593} & \\
\addlinespace[3pt]
\multirow{2}{*}{\textbf{\method}}
& Mean & \textbf{0.683} & \textbf{62.50} & \textbf{0.326}
& \textbf{0.228} & \textbf{1.982} & \multirow{2}{*}{\textbf{36.7}} \\
& Best & \textbf{0.903} & \textbf{87.26} & \textbf{0.146}
& \textbf{0.091} & \textbf{2.629} & \\
\bottomrule
\end{tabular*}
\par\vspace{3pt}
\parbox{\textwidth}{\footnotesize\raggedright
Note: Both iPAE metrics are normalized. Strict-hit rates
use all 384 candidates (12 targets $\times$ 32 samples); full settings and evaluation details are in
Table~\ref{tab:app_cofold_main}.}
\endgroup
\vspace{-5pt}
\end{table}

%% file: tables/ablation_lnr.tex
\begin{table*}[!ht]
\centering
\small
\caption{\textbf{Ablations on LNR.}
The default $\tseq=0.5$ achieves the highest overall cyclization success within each topology and the best HT Stability and Rosetta dG scores.
Freezing backbone motion substantially worsens Stability for both topologies, while removing the MPNN prior weakens interface scores and removing the delay produces topology-dependent effects.
Earlier activation at $\tseq=0.2$ improves overall sequence and structure diversity, indicating a trade-off between diversity and interface scores. OOD/ID correspond to 37/56 receptors.}
\label{tab:ablation}

\setlength{\tabcolsep}{3pt}
\begin{tabular*}{\textwidth}{@{\extracolsep{\fill}} l *{5}{rr} @{}}
\toprule

\multirow{2}{*}{Inference setting}
& \multicolumn{2}{c}{Cyc. $\uparrow$}
& \multicolumn{2}{c}{Stability $\downarrow$}
& \multicolumn{2}{c}{$D_{\mathrm{seq}}\uparrow$}
& \multicolumn{2}{c}{$D_{\mathrm{struct}}\uparrow$}
& \multicolumn{2}{c}{Rosetta $dG\downarrow$} \\
\cmidrule(lr){2-3}
\cmidrule(lr){4-5}
\cmidrule(lr){6-7}
\cmidrule(lr){8-9}
\cmidrule(lr){10-11}

& \textsc{ood} & \textsc{id}
& \textsc{ood} & \textsc{id}
& \textsc{ood} & \textsc{id}
& \textsc{ood} & \textsc{id}
& \textsc{ood} & \textsc{id} \\
\midrule

w/o coupling (HT)
& 97.1 & 99.5
& 112.60 & -255.79
& \underline{0.937} & 0.824
& 0.933 & 0.439
& -4.54 & -12.92 \\

w/o coupling (SS)
& 97.1 & 99.6
& -167.87 & -249.78
& 0.639 & 0.213
& 0.989 & \underline{0.936}
& -7.97 & -12.95 \\

\midrule

w/o MPNN prior (HT)
& 97.3 & 99.8
& \underline{-914.12} & -583.86
& 0.933 & 0.836
& 0.948 & 0.427
& \underline{-27.56} & \underline{-17.36} \\

w/o MPNN prior (SS)
& 97.1 & 99.7
& -808.71 & -503.13
& 0.639 & 0.271
& 0.948 & \underline{0.936}
& -18.88 & -16.37 \\

\midrule

w/o delay, $\tseq=0$ (HT)
& 97.1 & 99.3
& -905.72 & -591.57
& 0.934 & 0.829
& 0.940 & 0.453
& -20.40 & -17.17 \\

w/o delay, $\tseq=0$ (SS)
& 97.2 & 99.7
& -838.82 & -537.96
& 0.668 & 0.219
& 0.982 & 0.924
& -19.61 & -16.92 \\

\midrule

$\tseq=0.2$ (HT)
& \textbf{97.6} & 99.8
& -913.84 & \underline{-592.83}
& \textbf{0.940} & \textbf{0.903}
& 0.934 & 0.590
& -23.38 & -16.30 \\

$\tseq=0.2$ (SS)
& \underline{97.5} & \underline{99.9}
& -812.85 & -513.08
& 0.672 & 0.205
& \textbf{1.000} & \textbf{0.943}
& -19.59 & -16.36 \\

\midrule

$\tseq=0.5$ (HT)
& \underline{97.5} & \textbf{100.0}
& \textbf{-940.05} & \textbf{-604.28}
& \textbf{0.940} & \underline{0.845}
& 0.934 & 0.459
& \textbf{-28.37} & \textbf{-19.14} \\

$\tseq=0.5$ (SS)
& \textbf{97.6} & \textbf{100.0}
& -807.36 & -509.06
& 0.668 & 0.205
& \underline{0.991} & \textbf{0.943}
& -21.17 & -15.39 \\

\bottomrule
\end{tabular*}%
\end{table*}

%% file: declarations.tex
\subsection*{AI Use Statement}
In this work, generative AI tools were used to assist with language polishing, phrasing improvements, and the \LaTeX{} formatting and typesetting of mathematical notations and proofs. Additionally, generative AI was utilized to assist with code organization and formatting; all code changes were manually inspected, executed, and verified by the authors. All scientific ideation, methodological designs, theoretical formulations, and mathematical proofs (Appendix~\ref{app:dc_math}) were derived and completed by the authors. Model architectures, experimental implementations, and data analyses were likewise executed exclusively by the authors. The authors manually inspected and verified all AI-formatted text revisions against their original ones, and take full responsibility for the content, technical claims, code, and findings presented in this paper.

\subsection*{Ethics Statement}
This work introduces \method{}, a computational generative framework for target-conditioned cyclic peptide binder design aimed at facilitating therapeutic peptide discovery against challenging protein-protein interaction (PPI) interfaces. The research was conducted entirely \textit{in silico} and relies exclusively on publicly accessible, anonymized structural repositories, including the Protein Data Bank (PDB), the CPSea database, and the open-source LNR benchmark. No human participants, personally identifiable health information, or vertebrate animals were involved in this study; therefore, no Institutional Review Board (IRB) or animal care approval was required. The methodology is designed as a foundational tool for computational biophysics and molecular design and is not intended for, nor tested on, the generation of biological toxins, dual-use hazardous agents, or pathogen engineering. The authors foresee no direct negative societal, discriminatory, or biosecurity implications from this work.

\subsection*{Reproducibility Statement}
To ensure the reproducibility of all findings, models, and algorithms in this paper, we have provided:
\begin{itemize}[leftmargin=*,itemsep=1pt,topsep=1pt]
    \item \textbf{Code and Checkpoints:} An anonymized codebase implementing \method{}—including delayed-coupling flow matching, the ProteinMPNN prior integration, and the cascaded search pipeline—is provided in the supplementary material.
    \item \textbf{Architectures and Hyperparameters:} Detailed network structures (geometric trunk, IPA modules), loss weights, and training schedules are fully specified in Appendix~\ref{app:architecture} and Appendix~\ref{app:implementation}.
    \item \textbf{Data Curation and Benchmarks:} Complete stereochemical filtering scripts for the CPSea corpus (amide/disulfide geometry audits) and the MMseqs2-based sequence-identity splits for the LNR benchmark are documented in Appendix~\ref{app:data}.
    \item \textbf{Evaluation and Seeds:} All sampling trajectories use fixed 200-step schedules. Evaluation protocols for AlphaFold-Multimer-v3 and Rosetta FastRelax are detailed in Appendix~\ref{app:implementation}.
    \item \textbf{Theoretical Proofs:} Rigorous proofs and regularity assumptions for all theoretical propositions (delayed activation, coupling benefit, and search coverage) are provided in Appendix~\ref{app:dc_math}.
\end{itemize}

%% file: tables/app_notation.tex
Table~\ref{tab:notation} summarizes the principal symbols used in the main text and theoretical analysis. Additional architectural variables are defined locally in Appendix~\ref{app:architecture}. Time-indexed variables denote intermediate generative states, with $t=1$ indicating the endpoint. The receptor remains fixed throughout each trajectory, while peptide generation is conditioned on the specified length and cyclization mode.

\begin{table}[htbp]
\centering
\small
\caption{Principal notation for generation, prior integration, and cascaded search.}
\label{tab:notation}
\begin{tabular}{@{}p{0.30\textwidth}p{0.66\textwidth}@{}}
\toprule
Symbol & Meaning\\
\midrule
$R$ & Receptor structure and its representation\\
$L,\Gamma,\xi$ & Peptide length, cyclization specification, and joint conditions $\xi=(L,\Gamma)$\\
$\mathcal A$ & Amino-acid alphabet\\
$t\in[0,1]$ & Global generation time, distinct from the local modality clocks\\
$B_t$ & Peptide backbone frames and carbonyl-orientation state at time $t$\\
$S_t,S_1$ & Continuous sequence carrier during generation and decoded endpoint sequence $S_1\in\mathcal A^L$\\
$C_t$ & Side-chain torsion state; full heavy-atom coordinates are reconstructed from $Y_1$\\
$Y_t$ & Joint peptide state $(B_t,S_t,C_t)$\\
$p_0,p_\theta$ & Initial sampling distribution and learned conditional generative distribution\\
$\tbb,\tseq,\tsc$ & Backbone, sequence, and side-chain activation times; defaults $0$, $0.5$, and $0.8$\\
$m,s_m(t)$ & Modality index $m\in\{\mathrm{b},\mathrm{s},\mathrm{c}\}$ and its local progress clock\\
$\theta,\phi$ & Trainable generator parameters and frozen ProteinMPNN parameters\\
$F_\phi$ & Frozen ProteinMPNN structural sequence prior\\
$\ell^B_{i,t},\ell^P_{i,t},\widetilde\ell^P_{i,t}$ & Native sequence logits, prior log-probabilities, and centered prior residuals at residue $i$\\
$\ell_{i,t},\pi_{i,t}$ & Fused sequence logits and the resulting categorical distribution\\
$\pi^B,q^P$ & Native and prior distributions over admissible amino acids; position and time indices suppressed\\
$g$ & Trainable scalar gate controlling the prior contribution; initialized to $0.1$\\
$\ell_B,\ell_S,\ell_C,\ell_{\mathrm{cyc}}$ & Backbone, sequence, side-chain, and cyclic-closure loss terms\\
$\lambda_B,\lambda_S,\lambda_C$ & Nonnegative weights of the modality-specific losses\\
$K,K_0$ & Number of initial backbone hypotheses and retained roots, with $K_0<K$\\
$M,N$ & Continuations per retained root and total proposals $N=K_0M$\\
$N_0$ & Number of candidates retained after refinement warm-up, with $N_0<N$\\
$Q$ & Interface quality score for probabilistic early screening, defined in \eqref{eq:dc_search_quality}\\
$Z=(R,\xi)$ & Receptor and design conditions grouped for the theoretical analysis\\
$\overline{B}_t$ & Reference backbone observation used in the information-theoretic analysis\\
$H_t$ & Complete sampler checkpoint, including active states and memory needed for continuation\\
\bottomrule
\end{tabular}
\end{table}

The temporal factorization in \eqref{eq:dc_factorization} specifies a computational sampling schedule, not a physical backbone-first binding pathway or independence among modalities. The information decomposition in \eqref{eq:information_conservation} is derived under the reference-path assumptions in Appendix~\ref{app:delayed_proof1}. More generally, the theoretical results state conditions under which delayed activation, continued coupling, prior transfer, and structural diversification can be beneficial; their practical effects are evaluated empirically. In cascaded refinement, $Q$ is used for early candidate selection, whereas early stopping uses the composite score reported by AfCycDesign.

%% file: tables/app_math.tex
\providecommand{\Rec}{R}
\providecommand{\B}{B}
\providecommand{\Sseq}{S}
\ifdefined\dcappproposition\else
  \newtheorem{dcappproposition}{Proposition}[section]
  \newtheorem{dcappcorollary}[dcappproposition]{Corollary}
\fi

\label{app:dc_math}
This appendix analyzes the four mechanisms used in
Sections~\ref{sec:delayed_coupling}--\ref{sec:cascaded_search}:
delayed activation, continued coupling, residual-prior transfer, and
cascaded inference-time allocation. Each result specifies its probability
law, assumptions, and implication for the implemented method. We write
$Z=(\Rec,\xi)$ for the receptor and design conditions and use natural
logarithms. The finite sequence alphabet ensures that the sequence entropies
below are finite.

The first two parts concern a joint endpoint distribution $p^\star$ and a
specified interpolation or training coupling. The third concerns prediction
risk on a fixed distribution of intermediate contexts. The fourth concerns
the learned sampler and its conditional continuation distributions.
Transferring a statement between these settings requires the stated
additional assumptions.

\paragraph{Staged generation as composition of conditional kernels.}
\Eqref{eq:dc_factorization} gives a compact stage-wise decomposition in terms of the modality states exposed at the activation boundaries. For the implemented sampler, continuation can additionally depend on persistent source variables and other sampling memory; we therefore use the complete checkpoint \(H_t\) below to state the corresponding transition law exactly. Let $H_t$ be the complete sampler checkpoint at time $t$. It contains the
visible active state and any source variables or other memory needed to
continue sampling. This distinction matters when updates depend on a fixed
source such as $x_0$. Write $\mu_0(\mathrm dh_0\mid Z)$ for the initialized
checkpoint law and $P^\theta_{a:b}$ for the conditional transition kernel
between times $a$ and $b$, including intermediate categorical draws.
If $D(h_1)=Y_1$ denotes terminal decoding, then for any measurable output set
$E$,
\begin{equation}
\begin{aligned}
p_\theta(Y_1\in E\mid Z)
={}&\int \mu_0(\mathrm dh_0\mid Z)\,
P^\theta_{0:\tseq}(\mathrm dh_S\mid h_0,Z)\\
&\quad {}\times
P^\theta_{\tseq:\tsc}(\mathrm dh_C\mid h_S,Z)\,
P^\theta_{\tsc:1}(\mathrm dh_1\mid h_C,Z)\,
\mathbf 1\{D(h_1)\in E\}.
\end{aligned}
\label{eq:app_staged_kernel_composition}
\end{equation}
This is the marginal law induced by the three stages. A product of the
three kernels describes their joint path law before integrating out
intermediate states. Replacing complete checkpoints by $\B_{\tseq}$ or
$(\B_{\tsc},\Sseq_{\tsc})$ alone additionally requires that those variables
are sufficient for the future sampler. Temporal staging does not assume
that the final backbone, sequence, and side chains are independent.

% ============================================================
\subsection{Delayed activation: structural information and sequence risk}
\label{app:delayed_proof1}

\subsubsection{Additional information beyond receptor conditioning}

Let
\begin{equation}
J_B=I(\Sseq_1;\B_1\mid Z),\qquad
D(t)=I(\Sseq_1;\B_1\mid\overline{\B}_t,Z),\qquad
G(t)=I(\Sseq_1;\overline{\B}_t\mid Z).
\label{eq:app_information_quantities}
\end{equation}
The bar distinguishes a reference backbone observation from the complete
learned sampler state. In particular, the analysis does not assume that
the receptor is uninformative. The residual quantity satisfies
\begin{equation}
J_B=\mathbb E\!\left[
D_{\mathrm{KL}}\!\left(
p^\star(\Sseq_1\mid\B_1,Z)
\,\Vert\,p^\star(\Sseq_1\mid Z)
\right)\right].
\label{eq:jb_kl}
\end{equation}
Thus $J_B>0$ precisely when backbone conditioning changes sequence
preferences on a set of positive probability, beyond the information
already contained in $Z$.

\paragraph{Assumption D (ordered reference observations).}
For every $t$, assume
$\Sseq_1\perp\overline{\B}_t\mid(\B_1,Z)$.
For $t_1<t_2$, assume there is a stochastic kernel that maps the
observation at $t_2$ to one with the conditional law of the observation at
$t_1$, without access to $\B_1$ or $\Sseq_1$.
Equivalently, the two observation laws admit a coupling with conditional
Markov chain
\[
\Sseq_1\longrightarrow\B_1\longrightarrow
\overline{\B}_{t_2}\longrightarrow\overline{\B}_{t_1}
\quad\text{given }Z.
\]
Assume the observation at zero is independent of the endpoints given
$Z$, and $\overline{\B}_1=\B_1$.

\begin{dcappproposition}[Information ordering and optimal sequence loss]
\label{prop:1}
Under Assumption D, $G(t)$ is non-decreasing and $D(t)$ is non-increasing,
with
\begin{equation}
J_B=G(t)+D(t),\qquad
D(0)=J_B,\quad D(1)=0.
\label{eq:information_conservation}
\end{equation}
Define the optimal sequence log-loss using the intermediate observation by
\begin{equation}
\mathcal L^\star_{\mathrm{seq}}(t)
=\inf_q\mathbb E[-\log q(\Sseq_1\mid\overline{\B}_t,Z)].
\end{equation}
Then
\begin{equation}
\begin{aligned}
\mathcal L^\star_{\mathrm{seq}}(t)
&=H(\Sseq_1\mid\overline{\B}_t,Z),\\
\mathcal L^\star_{\mathrm{seq}}(t)
-\mathcal L^\star_{\mathrm{seq}}(1)
&=D(t).
\end{aligned}
\label{eq:app_delay_risk_gap}
\end{equation}
For $t_1<t_2$, the loss reduction equals $G(t_2)-G(t_1)$ and is strict
exactly when this increment is positive.
\end{dcappproposition}

\begin{proof}
Conditional independence and the chain rule give
\begin{equation}
\begin{aligned}
I(\Sseq_1;\B_1,\overline{\B}_t\mid Z)
&=I(\Sseq_1;\B_1\mid Z)\\
&=I(\Sseq_1;\overline{\B}_t\mid Z)
 +I(\Sseq_1;\B_1\mid\overline{\B}_t,Z).
\end{aligned}
\end{equation}
This proves the decomposition. Conditional data processing under the
degradation coupling gives $G(t_1)\leq G(t_2)$. The endpoint identities
follow from the endpoint assumptions.

For any conditional predictor $q$, its log-loss decomposes as
\begin{equation}
\begin{aligned}
\mathbb E[-\log q(\Sseq_1\mid\overline{\B}_t,Z)]
={}&H(\Sseq_1\mid\overline{\B}_t,Z)\\
&+\mathbb E D_{\mathrm{KL}}\!\left(
p^\star(\Sseq_1\mid\overline{\B}_t,Z)
\,\Vert\,q(\Sseq_1\mid\overline{\B}_t,Z)\right).
\end{aligned}
\end{equation}
Minimizing the nonnegative KL term proves the first risk identity.
Furthermore,
$H(\Sseq_1\mid\overline{\B}_t,Z)
=H(\Sseq_1\mid Z)-G(t)$, and the analogous endpoint expression subtracts
$J_B$. Their difference is $D(t)$. The strictness statement follows by
subtracting risks at the two times.
\end{proof}

\subsubsection{A reference path satisfying the assumption}
\label{app:delayed_proof2}

Consider a Euclidean backbone representation with
\begin{equation}
\overline{\B}_t=(1-t)\B_0+t\B_1,\qquad
\B_0\sim\mathcal N(0,I),\qquad
\B_0\perp(\B_1,\Sseq_1,Z).
\label{eq:app_linear_path}
\end{equation}
For $t>0$, the invertible rescaling
$V_t=\overline{\B}_t/t=\B_1+\sigma_t\B_0$ has
$\sigma_t=(1-t)/t$. For $0<t_1<t_2$, generate
\begin{equation}
V_2=\B_1+\sigma_{t_2}\varepsilon_2,\qquad
V_1=V_2+\sqrt{\sigma_{t_1}^2-\sigma_{t_2}^2}\,\varepsilon',
\end{equation}
where $\varepsilon_2,\varepsilon'$ are independent standard Gaussian
variables independent of $(\B_1,\Sseq_1,Z)$.
The conditional marginal laws of $V_1,V_2$ equal those of
$V_{t_1},V_{t_2}$, and $V_1$ is degraded from $V_2$.
Consequently, this path satisfies Assumption D. Under this constructed
coupling,
\begin{equation}
D(t_1)-D(t_2)
=I(\Sseq_1;V_2\mid V_1,Z)\geq0.
\label{eq:exact_ambiguity_reduction}
\end{equation}
The actual shared-source interpolation need not be Markov across time;
the argument compares its observation channels through an equivalent
coupling.

\paragraph{Implication for delayed activation.}
The quantity $D(t)$ is an exact excess Bayes log-loss from using an
intermediate backbone instead of the endpoint backbone. If $G(t)$ grows
before $\tseq$, later sequence assignment has a lower information-limited
risk. For finite predictors, write their risk as
$\widehat{\mathcal L}(t)=\mathcal L^\star_{\mathrm{seq}}(t)+e(t)$,
where $e(t)\geq0$ is prediction excess risk. Then
\begin{equation}
\widehat{\mathcal L}(t_1)-\widehat{\mathcal L}(t_2)
=G(t_2)-G(t_1)+e(t_1)-e(t_2).
\label{eq:app_delay_finite_predictor}
\end{equation}
Thus the gain persists whenever the information increment exceeds any
increase in excess risk. For a factorized residue head, the same argument
applies to each residue and the losses sum; it need not model the full
sequence joint distribution.

\paragraph{Scope.}
The Gaussian construction applies to Euclidean coordinates, or whitened
coordinates on the relevant translation subspace. It does not prove
information ordering for rotations, the complete nonlinear learned
trajectory, or a guided trajectory that already uses sequence feedback.
Those cases require their own version of Assumption D. Nor does a fixed
$\tseq$ guarantee a confident or physically valid backbone. If $J_B=0$,
this information-based benefit vanishes. An unrestricted synchronous
network can in principle learn to ignore early uninformative variables;
the result supports explicit timing as an inductive bias, not universal
dominance over all synchronous architectures.

% ============================================================
\subsection{Continued coupling: sequence-conditioned backbone prediction}
\label{app:coupling_proof}

Fix an active time $t\geq\tseq$ and a joint law of
$(U_t^B,\B_t,\Sseq_t,Z)$.
Here $U_t^B$ is the backbone regression target in a common Euclidean or
tangent-space representation. Assume
$\mathbb E\|U_t^B\|^2<\infty$. Let
\begin{equation}
\begin{aligned}
m_B&=\mathbb E[U_t^B\mid\B_t,Z],&
\mathcal R_B&=\inf_f\mathbb E\|U_t^B-f(\B_t,Z)\|^2,\\
m_{BS}&=\mathbb E[U_t^B\mid\B_t,\Sseq_t,Z],&
\mathcal R_{BS}&=\inf_f\mathbb E\|U_t^B-f(\B_t,\Sseq_t,Z)\|^2.
\end{aligned}
\end{equation}
The infima range over square-integrable measurable predictors under the
same target law.

\begin{dcappproposition}[Available benefit from continued coupling]
\label{prop:2}
The optimal risks satisfy
\begin{equation}
\mathcal R_B-\mathcal R_{BS}
=\mathbb E\|m_{BS}-m_B\|^2
\triangleq\Delta_t^{B\leftarrow S}\geq0.
\label{eq:conditional_pythagorean}
\end{equation}
The decrease is strict if and only if
$\mathbb P(m_{BS}\neq m_B)>0$.
\end{dcappproposition}

\begin{proof}
Conditional expectations attain the two optimal squared risks.
Decompose
$U_t^B-m_B=(U_t^B-m_{BS})+(m_{BS}-m_B)$.
The second summand is measurable with respect to
$\sigma(\B_t,\Sseq_t,Z)$, whereas the first has zero conditional mean.
Their expected inner product is zero. Expanding the squared norm yields
\eqref{eq:conditional_pythagorean}; a squared norm has strictly
positive expectation exactly under the stated condition.
\end{proof}

\paragraph{When is the condition substantive?}
For the linear reference path and $t<1$, the velocity target is
$U_t^B=(\B_1-\B_t)/(1-t)$, giving
\begin{equation}
\Delta_t^{B\leftarrow S}
=\frac{1}{(1-t)^2}
\mathbb E\left\|
\mathbb E[\B_1\mid\B_t,\Sseq_t,Z]
-\mathbb E[\B_1\mid\B_t,Z]
\right\|^2.
\label{eq:backbone_refinement_gain}
\end{equation}
Thus strict gain requires sequence information to change the preferred
endpoint geometry, not merely to correlate with it.
For example, if
\[
U_t^B=m_B+A_t\zeta_t+\varepsilon_t,
\qquad
\mathbb E[\zeta_t\mid\B_t,Z]=0,\qquad
\mathbb E[\varepsilon_t\mid\B_t,\Sseq_t,Z]=0,
\]
where $A_t$ is measurable given $(\B_t,Z)$ and
$\zeta_t=\zeta(\Sseq_t,\B_t,Z)$, then
$\Delta_t^{B\leftarrow S}=\mathbb E\|A_t\zeta_t\|^2$.
Nonzero sequence-dependent geometric corrections therefore give a
strict benefit. Positive terminal mutual information alone does not
establish this conditional-mean condition.

\paragraph{Finite models and late side-chain information.}
If fitted predictors have excess risks $e_B,e_{BS}\geq0$ above their
respective Bayes risks, then
\begin{equation}
\widehat{\mathcal R}_B-\widehat{\mathcal R}_{BS}
=\Delta_t^{B\leftarrow S}+e_B-e_{BS}.
\end{equation}
A fitted coupled predictor is better when its additional approximation
or estimation error is smaller than the available conditioning gain.
At $t\geq\tsc$, the same orthogonal decomposition applies to the nested
information sets $(\B_t,\Sseq_t,Z)$ and $(\B_t,\Sseq_t,C_t,Z)$.

\paragraph{Implication for the architecture.}
Freezing the backbone after $\tseq$ prevents it from realizing further
sequence-dependent corrections. Continued evolution retains that
possibility whenever $\Delta_t^{B\leftarrow S}>0$.
The proposition compares regression information under a fixed law; it
does not by itself prove lower terminal docking error or a lower
force-field energy for the trained sampler. Together with the previous
part, it supports delaying access to immature modalities while preserving
their subsequent influence.

% ============================================================
\subsection{ProteinMPNN prior: transfer mechanism and prediction benefit}
\label{app:mpnn_transfer_proof}
\label{app:prior_tree_theory}

\subsubsection{Residual fusion and admissible support}

Fix an active prediction context, omitting residue and time subscripts.
The prior below is supplied by the frozen ProteinMPNN model
used in the Method~\citep{dauparas2022proteinmpnn}.
Let $\mathcal A_i\subseteq\mathcal A$ be the allowed amino acids at residue
$i$. On this support, let $p(a)=\pi^B(a)>0$ be the native categorical
distribution and $q(a)=q^P(a)>0$ the normalized frozen-prior distribution.
Restricting and renormalizing a full-alphabet prior to $\mathcal A_i$
does not change the fused law below. Let
$z(a)=\log q(a)$ and
$\widetilde\ell^P_a=\ell^P_a-|\mathcal A|^{-1}
\sum_{a'\in\mathcal A}\ell^P_{a'}$.

\begin{dcappproposition}[Exact transfer of prior log odds]
\label{prop:mpnn_transfer}
For every finite scalar gate $g$,
\begin{equation}
p_g(a)
=\frac{p(a)q(a)^g}{\sum_{b\in\mathcal A_i}p(b)q(b)^g},
\qquad
\log\frac{p_g(a)}{p_g(b)}
=\log\frac{p(a)}{p(b)}
+g\log\frac{q(a)}{q(b)}.
\label{eq:app_product_distribution}
\end{equation}
Moreover, $p_0=p$. At a fixed context with unrestricted base logits,
every strictly positive distribution on $\mathcal A_i$ can be represented
for any fixed finite $g$.
\end{dcappproposition}

\begin{proof}
In the softmax of $\ell^B+g\widetilde\ell^P$, centering contributes the
same scalar to every category and cancels between numerator and
denominator. Writing each uncentered exponential in terms of its
normalized distribution cancels the remaining normalizers, yielding
the first identity. Taking a ratio gives the second.
The case $g=0$ is immediate. For a desired distribution $r$, choose
$\ell^B_a=\log r(a)-g\widetilde\ell^P_a$ on the allowed support.
Its fused softmax is $r$.
\end{proof}

Centering is therefore a softmax-invariant parameterization; it does not
itself improve categorical probabilities. When $g>0$, the prior reinforces
its preferred categories. A negative unconstrained gate reverses these
preferences. The pointwise construction is not a claim that a finite
shared network realizes arbitrary corrections at all contexts. The
architectural nesting at $g=0$, however, is exact.

\subsubsection{A sufficient condition for strictly lower cross-entropy}

The product identity explains how transfer occurs, but does not alone
show that it helps. We next give a prediction-risk criterion.

Let $X$ denote an active residue-time context sampled from a fixed
distribution $\nu$, with true conditional residue law $r_X$ supported
on the admissible alphabet, native
prediction $p_X$, and frozen prior $q_X$. Hold the base predictor and
context distribution fixed while varying the global gate. Set
$z_X(a)=\log q_X(a)$ and
\begin{equation}
\mathcal E(g)
=\mathbb E_{X\sim\nu}\mathbb E_{a\sim r_X}[-\log p_{g,X}(a)].
\end{equation}

\paragraph{Assumption P (bounded, aligned structural preferences).}
Assume $p_X,q_X$ have full support on $\mathcal A_i$, finite base risk, and
\begin{equation}
\max_a z_X(a)-\min_a z_X(a)\leq D_P<\infty
\quad\text{for }\nu\text{-almost every }X.
\end{equation}
Define the average prior-alignment gap
\begin{equation}
\delta_P=
\mathbb E_X\!\left[
\mathbb E_{a\sim r_X}z_X(a)
-\mathbb E_{a\sim p_X}z_X(a)\right].
\label{eq:app_prior_alignment}
\end{equation}
For positive-prior transfer, assume $\delta_P>0$.
This means that the true residue distribution receives a larger expected
prior log score than the current native prediction does.

\begin{dcappproposition}[Cross-entropy benefit of an aligned prior]
\label{prop:app_prior_risk_gain}
Under Assumption P,
\begin{equation}
\begin{aligned}
\mathcal E(g)-\mathcal E(0)
&=\mathbb E_X\left[
\log\mathbb E_{a\sim p_X}e^{g z_X(a)}
-g\mathbb E_{a\sim r_X}z_X(a)\right],\\
\mathcal E'(0)&=-\delta_P,\qquad
\mathcal E''(g)=
\mathbb E_X\operatorname{Var}_{a\sim p_{g,X}}[z_X(a)]
\leq D_P^2/4.
\end{aligned}
\label{eq:app_prior_risk_derivatives}
\end{equation}
Consequently,
\begin{equation}
\mathcal E(g)
\leq\mathcal E(0)-g\delta_P+\frac{D_P^2g^2}{8}.
\label{eq:app_prior_risk_bound}
\end{equation}
If $D_P>0$ and $\delta_P>0$, every
$0<g<8\delta_P/D_P^2$ gives a strict reduction, and the choice
$g=4\delta_P/D_P^2$ yields a reduction of at least
$2\delta_P^2/D_P^2$.
\end{dcappproposition}

\begin{proof}
Take the negative logarithm of the normalized product in
\eqref{eq:app_product_distribution}, and average under $r_X$.
Differentiation of the log normalizer gives its tilted mean and variance.
To justify differentiation even if the absolute log scores vary across
contexts, subtract $\min_a z_X(a)$ separately in each context; the risk
difference and all derivatives are unchanged, while the centered scores
are bounded in $[0,D_P]$.

For a random variable $W\in[a,b]$, the inequality
$(W-a)(b-W)\geq0$ gives
$\operatorname{Var}(W)\leq(\mathbb EW-a)(b-\mathbb EW)
\leq(b-a)^2/4$.
Apply this to the tilted prior score to obtain the curvature bound.
Taylor's theorem then gives
\eqref{eq:app_prior_risk_bound}; the asserted interval and
quantitative reduction follow by completing the quadratic.
\end{proof}

\paragraph{Interpretation and training.}
The assumption is a calibration condition on current intermediate
backbones, not on clean structures alone. The result supplies an explicit
way in which externally learned preferences can reduce sequence loss:
they must correct directions in which the native predictor underweights
compatible residues. It also explains why the gate should be learned.
For the same native predictor, $g=0$ remains available; if
$\delta_P<0$, sufficiently small negative gates decrease the fixed-context
risk instead. If the prior is uniform on the support, it has no effect.
If $p_X=r_X$, then $\delta_P=0$ and no strict improvement over the Bayes
predictor is possible.

During joint training, the parameter class includes the native model
and, under Assumption P at a given native predictor, a strictly better
positive-gate prediction rule. This is a population-risk existence
statement, not a guarantee that finite-data optimization finds that rule.
The numerical gate above is an analytical witness, not a recommended
replacement for the trained gate.

\paragraph{Gradient paths and dynamic use.}
The implementation evaluates
$\operatorname{stopgrad}(F_\phi(\B_t;\Rec))$.
The stop-gradient operator blocks that computational gradient path,
while $\phi$ stays frozen. It does not make the mathematical function
$F_\phi$ constant in its coordinates. Its values can change as $\B_t$
changes, and fused categorical samples can influence later backbone
updates through the shared network.

The risk result fixes $\nu$ and the base predictions. It directly
describes prediction on sampled training contexts. During inference,
changing $g$ can also change future contexts, so terminal cofolding
confidence requires a separate trajectory-level argument or empirical
evaluation. Frozen pretrained parameters transfer learned preferences;
they do not create information beyond the input backbone in the
information-theoretic sense.

% ============================================================
\subsection{Cascaded search: structural selection, sequence coverage, and refinement}
\label{app:cascaded_search_theory}

This part analyzes the cascade as a sequence of conditional decisions.
The structural selector reduces redundant conditioning contexts before
sequence expansion. Quality-based selection occurs later, after short
AfCycDesign trials. These operations use different surrogates and have
different possible failure modes.

\subsubsection{Branchable checkpoints and conditional continuation laws}
\label{app:branching_checkpoints}

Suppose the backbone, sequence, and side-chain sources are independent
given $Z$, and inactive peptide sequence and side-chain inputs are
masked before $\tseq$. Assume also that no external sequence-dependent
guidance is applied during this prefix.
Then the prefix checkpoint relevant to backbone evolution can be
constructed without using the unrevealed sequence source.
Replacing that source at $\tseq$ leaves the generated prefix unchanged.
This follows by induction over the prefix updates: each update uses only
the previous prefix state, the conditions, and permitted prefix randomness.

Generate a root pool $\mathcal H=(H_1,\ldots,H_K)$ at $\tseq$.
Each $H_i$ contains $\B_{\tseq}^{(i)}$ and all memory needed for a valid
continuation, including source-dependent sampler variables.
Any stored unused sequence variables are replaced by the child's fresh
source; they are not held fixed as conditioning variables in $P_i$ below.
For a retained index set $A$, let
\begin{equation}
P_i(s)=
\mathbb P_\theta(\Sseq_1=s\mid H_i,Z;
\text{fixed continuation and guidance protocol}).
\label{eq:app_continuation_law}
\end{equation}
Conditional on $\mathcal H$ and $A$, assume distinct children use
independent fresh sequence sources and continuation randomness.
The external oracle is deterministic or uses independent randomness for
each child, with no cross-child adaptation.
Then the $M$ children of root $i$ are independent draws from $P_i$.
Their backbone states continue to evolve; no side-chain branching factor
is added to the cofolding policy.

The selected roots are generally dependent because selection is
history-dependent and without replacement. The sequence identities below
condition on the realized root pool and selected set; they do not
assume independent selected roots.

\subsubsection{Structural potentials and reduced redundancy}
\label{app:backbone_selection}

Let $d_{ij}$ be the declared distance between visible backbones of
roots $i,j$, preserving the task's residue correspondence.
Distances in a common receptor frame measure pose as well as shape;
independent peptide alignment measures shape alone.
The proof applies to either explicitly chosen distance but does not
identify them.

Let $\kappa_{ij}=\kappa(d_{ij})\in[0,1]$ be symmetric, with $\kappa$
non-increasing. No positive-semidefinite-kernel or determinantal-process
assumption is required. For a nonempty selected set $A$, define
\begin{equation}
c_i(A)=\frac{1}{|A|}\sum_{j\in A}\kappa_{ij},
\qquad i\in\mathcal U_A:=\{1,\ldots,K\}\setminus A.
\label{eq:app_crowding}
\end{equation}
Let $w(c)>0$ be bounded and non-increasing on $[0,1]$. Select
\begin{equation}
\mu_A(i)
=\frac{w(c_i(A))}
{\sum_{j\in\mathcal U_A}w(c_j(A))},
\qquad i\in\mathcal U_A.
\label{eq:app_root_selection}
\end{equation}
The first root is selected uniformly; subsequent draws are without
replacement until $|A|=K_0$.
This formulation covers a family of probability-based crowding rules.
The evaluated implementation must specify its distance, similarity,
and weighting function.

Potential-based reweighting at an intermediate generation checkpoint is
related to Feynman--Kac steering~\citep{singhal2025fksteering}.
Here the potential also depends on the roots already selected.
\Eqref{eq:app_root_selection} defines the conditional choice law
being analyzed; it is not asserted to sample exactly from a prescribed
terminal reward distribution.

\begin{dcappproposition}[One-step reduction in structural crowding]
\label{prop:app_crowding_reduction}
Fix a root pool and a nonempty selected set $A$. Let $u_A$ be uniform
on $\mathcal U_A$. Then
\begin{equation}
\mathbb E_{\mu_A}c_i(A)-\mathbb E_{u_A}c_i(A)
=\frac{\operatorname{Cov}_{u_A}
(c_i(A),w(c_i(A)))}{\mathbb E_{u_A}w(c_i(A))}
\leq0.
\label{eq:app_crowding_gain}
\end{equation}
The inequality is strict when two remaining crowding values differ and
the weighting function strictly distinguishes them.
Every remaining root retains positive selection probability.
\end{dcappproposition}

\begin{proof}
The equality follows by writing the reweighted expectation as
$\mathbb E_{u_A}[cw]/\mathbb E_{u_A}w$.
For independent indices $I,J\sim u_A$,
\begin{equation}
2\operatorname{Cov}_{u_A}(c,w(c))
=\mathbb E[(c_I-c_J)(w(c_I)-w(c_J))]\leq0,
\end{equation}
because the differences have opposite signs.
Strictness follows if a pair gives a strictly negative product;
positivity follows from $w>0$.
\end{proof}

For example, $w(c)=\exp(-\beta c)$ with $\beta>0$ has the
regularization interpretation
\begin{equation}
\mu_A=\arg\min_{\mu}
\left\{\mathbb E_{\mu}c_i(A)
+\beta^{-1}D_{\mathrm{KL}}(\mu\Vert u_A)\right\}.
\label{eq:app_selector_regularization}
\end{equation}
A Lagrange multiplier gives this form; strict convexity gives uniqueness.
It tends to uniform selection as $\beta\downarrow0$ and concentrates on
minimum-crowding roots as $\beta\to\infty$.
The covariance result is more general and does not require using this
particular family in the implementation.

\paragraph{Scope and cost.}
The comparison uses the same root pool and selection history for both
choices. It proves a local reduction in expected redundancy, not global
optimality of the selected set or domination of a separate adaptive run.
Low crowding alone establishes neither physical validity nor sequence
diversity. With a cached similarity matrix, crowding sums can be updated
in $O(K)$ work per root, or $O(KK_0)$ total.
If a distance costs $O(L)$, the matrix costs $O(K^2L)$ time and $O(K^2)$
storage; on-demand distances to selected roots cost $O(KK_0L)$ time and
$O(K)$ auxiliary storage.

\subsubsection{From geometric redundancy to conditional sequence coverage}
\label{app:sequence_coverage}

Fix in advance a map
$\psi:\mathcal A^L\to\{1,\ldots,J\}$ describing sequence families.
It may be the identity on exact sequences or a declared partition by
sequence similarity or another fixed criterion. Let
\begin{equation}
\pi_i(v)=P_i(\psi(\Sseq_1)=v),\qquad
o_{ij}=\sum_{v=1}^J\pi_i(v)\pi_j(v).
\label{eq:app_sequence_overlap}
\end{equation}
Thus $o_{ij}$ is the probability that independent continuations from
roots $i,j$ fall in the same sequence family.
For a selected set $A$ of size $k$, write $N=kM$.
Let $D_A$ count distinct families among the $N$ children, and $C_A$
count unordered pairs of children in the same family.

\begin{dcappproposition}[Exact coverage and a collision-based certificate]
\label{prop:app_sequence_coverage}
Under conditional independence of the children,
\begin{equation}
\mathbb E[D_A\mid\mathcal H,A]
=\sum_{v=1}^J
\left[1-\prod_{i\in A}(1-\pi_i(v))^M\right],
\label{eq:app_exact_coverage}
\end{equation}
and
\begin{equation}
\mathbb E[C_A\mid\mathcal H,A]
=\binom M2\sum_{i\in A}o_{ii}
+M^2\sum_{\substack{i<j\\i,j\in A}}o_{ij}.
\label{eq:app_expected_collisions}
\end{equation}
Consequently, $\mathbb E D_A\geq N-\mathbb E C_A$.
If structural similarity controls between-root sequence overlap as
\begin{equation}
o_{ij}\leq a_S\kappa_{ij}+\varepsilon_S
\quad(i\neq j),\qquad a_S,\varepsilon_S\geq0,
\label{eq:app_overlap_envelope}
\end{equation}
then
\begin{equation}
\begin{aligned}
\mathbb E[D_A\mid\mathcal H,A]\geq{}&
N-\binom M2\sum_{i\in A}o_{ii}\\
&-M^2\left(
a_S\sum_{\substack{i<j\\i,j\in A}}\kappa_{ij}
+\varepsilon_S\binom k2\right).
\end{aligned}
\label{eq:app_coverage_certificate}
\end{equation}
\end{dcappproposition}

\begin{proof}
A family $v$ is absent exactly when every child misses it.
Conditional independence gives the product in
\eqref{eq:app_exact_coverage}; sum the inclusion indicators.
Within root $i$, $\binom M2$ pairs have collision probability $o_{ii}$.
Between roots $i,j$, $M^2$ pairs have collision probability $o_{ij}$.
Linearity of expectation proves the collision identity.
If family $v$ occurs $n_v$ times, then
$(n_v-1)_+\leq\binom{n_v}{2}$.
Summing gives $N-D_A\leq C_A$.
Substitute \eqref{eq:app_overlap_envelope} for the final bound.
\end{proof}

\Eqref{eq:app_overlap_envelope} is a cross-modal calibration
assumption: backbone similarity must be informative about overlap between
conditional sequence distributions. Coupling and the backbone-conditioned
prior provide a mechanism for such dependence, but do not prove this
bound. Arbitrarily large $a_S$ or $\varepsilon_S$ can make the
certificate vacuous and do not establish useful cross-modal calibration.
With the within-root term held fixed or controlled, reducing the kernel
sum improves the coverage lower bound.

A stronger local calibration yields a reduction in actual expected
sequence collisions. Fix $A$ with $r=|A|>0$, and suppose
\begin{equation}
\begin{gathered}
o_{ii}=d_A\quad\text{for all }i\in\mathcal U_A,\\
o_{ij}=a_S\kappa_{ij}+b_A+e_{ij},
\quad |e_{ij}|\leq\varepsilon_S,\quad
i\in\mathcal U_A,\ j\in A.
\end{gathered}
\label{eq:app_local_overlap_calibration}
\end{equation}
Let
$\Delta_c(A)=\mathbb E_{u_A}c_i(A)-\mathbb E_{\mu_A}c_i(A)\geq0$.

\begin{dcappcorollary}[Reduced sequence redundancy under calibration]
\label{cor:app_sequence_deredundancy}
Under \eqref{eq:app_local_overlap_calibration}, the expected
number of new colliding child pairs after adding one root satisfies
\begin{equation}
\mathbb E_{\mu_A}[\Delta C_i]
-\mathbb E_{u_A}[\Delta C_i]
\leq-M^2r\bigl(a_S\Delta_c(A)-2\varepsilon_S\bigr).
\label{eq:app_local_collision_gain}
\end{equation}
It is strictly smaller under the crowding selector if
$a_S\Delta_c(A)>2\varepsilon_S$.
\end{dcappcorollary}

\begin{proof}
For candidate $i$,
$\mathbb E[\Delta C_i]=\binom M2o_{ii}+M^2\sum_{j\in A}o_{ij}$.
The diagonal and $b_A$ terms cancel between selection laws.
The kernel term is $-M^2r a_S\Delta_c(A)$.
The sum of errors lies in $[-r\varepsilon_S,r\varepsilon_S]$ for each
$i$, so its two expectations differ by at most $2r\varepsilon_S$.
\end{proof}

\paragraph{An exact coverage case.}
Suppose roots belong to structural classes, $\kappa_{ij}$ is one for
roots in the same class and zero otherwise, and each class corresponds
to a distinct sequence family with probability one.
Roots from a class not yet selected have crowding zero and maximal
weight $w(0)$. If $n_{\mathrm{new}}$ of the $n$ remaining roots belong
to unseen classes, then
\begin{equation}
\mathbb P_{\mu_A}(\text{new sequence family})
=\frac{n_{\mathrm{new}}w(0)}
{n_{\mathrm{new}}w(0)+
 \sum_{i:\,\mathrm{seen\ class}}w(c_i(A))}
\geq\frac{n_{\mathrm{new}}}{n}.
\end{equation}
The right-hand side is the uniform-choice probability.
Thus this separated conditional-family model directly converts
geometric selection into a greater chance of adding a new family at
the current decision. It is an example, not a universal model of the
trained generator.

\paragraph{What breadth and branching each contribute.}
Let $I$ be uniform over $A$ and $V\sim\pi_I$ a random child's family.
For $\bar\pi_A=k^{-1}\sum_{i\in A}\pi_i$,
\begin{equation}
H(\bar\pi_A)
=\frac1k\sum_{i\in A}H(\pi_i)
+\frac1k\sum_{i\in A}D_{\mathrm{KL}}(\pi_i\Vert\bar\pi_A).
\label{eq:app_sequence_mixture_entropy}
\end{equation}
This is $H(V)=H(V\mid I)+I(V;I)$.
Distinct root-conditioned laws contribute between-root information;
$M$ children explore variability within each law.
If all $P_i$ are identical, structural selection cannot change the
sequence sampling law. Also, sequence distributions may already have
full support. Broader exploration here means finite-budget coverage and
less repetition, not necessarily a larger mathematical support.

\subsubsection{When broader exploration improves candidate discovery}
\label{app:search_discovery}

For a fixed root set and a fixed per-child continuation/refinement
protocol, let $p_i(q)$ be the probability that a child of root $i$
achieves quality at least $q$. Conditional independence gives
\begin{equation}
\mathbb P\!\left(\max_{i\in A,\,m\leq M}V_{i,m}\geq q
\;\middle|\;\mathcal H,A\right)
=1-\prod_{i\in A}(1-p_i(q))^M.
\label{eq:app_exact_hit_probability}
\end{equation}
Here $V$ is a declared final cofolding score or hit indicator.
Diversity is useful only if it changes the success opportunities being
sampled; this identity alone does not favor diverse roots.

To make that requirement explicit, partition conditional sequence
exploration into regions $v$. For a chosen quality threshold, let
$\Theta_v\in[0,1]$ be region $v$'s latent success probability under a
fixed child protocol.

\paragraph{Assumption Q (distinct, unresolved quality opportunities).}
The $\Theta_v$ are independent and identically distributed across regions,
and independent of the geometry-only allocation.
Conditional on them, child successes are independent Bernoulli draws
with parameter $\Theta_v$ in region $v$.
Allocate $n_v$ children to region $v$, with $\sum_vn_v=N$ and fixed
per-child computation, without observing the latent $\Theta_v$.

\begin{dcappproposition}[Diversification reduces an averaged miss probability]
\label{prop:app_region_discovery}
Under Assumption Q, let $m(n)=\mathbb E[(1-\Theta)^n]$, with $m(0)=1$.
Then
\begin{equation}
P_{\mathrm{miss}}(n_1,\ldots,n_J)=\prod_{v=1}^Jm(n_v).
\label{eq:app_region_miss}
\end{equation}
For integers $a,b\geq1$,
\begin{equation}
m(a)m(b)\leq m(a+b).
\label{eq:app_region_split}
\end{equation}
Splitting $a+b$ trials from one region into $a$ and $b$ trials in two
independent regions therefore cannot increase the averaged miss
probability. The split factor decreases strictly when $\Theta$ is
non-degenerate.
For fixed $N$ and a fixed number of available exchangeable regions,
a balanced allocation minimizes the miss probability.
\end{dcappproposition}

\begin{proof}
Conditional on the latent probabilities, all trials fail with probability
$\prod_v(1-\Theta_v)^{n_v}$.
Independence gives the product of moments.
Let $X=1-\Theta$ and let $X'$ be an independent copy. Since both powers
are increasing on $[0,1]$,
\begin{equation}
m(a+b)-m(a)m(b)
=\tfrac12\mathbb E[(X^a-X'^a)(X^b-X'^b)]\geq0.
\end{equation}
For non-degenerate $X$, the integrand is strictly positive with positive
probability.
For balancing, Cauchy--Schwarz gives
$m(n)^2\leq m(n-1)m(n+1)$ for $n\geq1$.
Thus the increments of $\log m(n)$ are non-decreasing when the moments
are positive. If $n_u\geq n_v+2$, moving one trial from $u$ to $v$ cannot
increase $\log m(n_u)+\log m(n_v)$.
Repeated moves yield a balanced allocation.
If a positive moment vanishes, $\Theta=1$ almost surely and the miss
probability is already zero for any nonempty allocation.
\end{proof}

\paragraph{Relation to structural selection.}
The overlap assumptions link lower crowding to less redundant conditional
sequence exploration. Assumption Q then gives a regime where allocating
trials across distinct sequence regions improves discovery probability
at the same leaf budget. The assumptions play different roles; neither
follows from geometry alone. The comparison also fixes per-child
computation and does not yet include quality pruning.

This is an average over unresolved region quality, not a guarantee for
every fixed assignment of success rates.
If all regions have the same deterministic rate, spreading trials leaves
success probability unchanged. If high-quality regions are already
known, concentration can be preferable. Quality-based refinement
addresses that later information regime.
No monotonic score gain from increasing $K$ or reducing crowding alone
is asserted.

\subsubsection{Quality selection, early stopping, and total compute}
\label{app:cascade_refinement}

After $b=5$ initial AfCycDesign updates on each of $N=K_0M$ leaves,
let
\begin{equation}
Q_j=0.75\bigl(1-\widetilde{\mathrm{iPAE}}_{\mathrm{hotspot},j}\bigr)
+0.25\,\mathrm{ipTM}_j
\end{equation}
be the screening score. Let positive weights $v_j=v(Q_j)$ increase with
$Q_j$. The actual weighting rule and replacement policy must be reported.
The result below considers sequential selection of $N_0$ distinct leaves
using these fixed weights.

To separate screening error from generation, couple each candidate to
its potential outcome under the same full-budget refinement protocol.
Let $\mathcal G$ contain candidates that would reach the specified
quality threshold under that protocol; future optimizer randomness can
be fixed in advance for this definition.
The selector's random choices are independent of those potential outcomes
conditional on its observed scores.

\begin{dcappproposition}[Retention of potentially successful candidates]
\label{prop:app_quality_retention}
Condition on screening scores and potential outcomes, with
$\mathcal G\neq\varnothing$. Define
\begin{equation}
\omega_{\mathcal G}
=\frac{\sum_{j\in\mathcal G}v_j}{\sum_{j=1}^Nv_j}.
\end{equation}
Then
\begin{equation}
\mathbb P(A_Q\cap\mathcal G=\varnothing
\mid Q_1,\ldots,Q_N,\mathcal G)
\leq(1-\omega_{\mathcal G})^{N_0}.
\label{eq:app_retention_bound}
\end{equation}
\end{dcappproposition}

\begin{proof}
Conditioned on having selected no member of $\mathcal G$, its total
remaining weight is unchanged, while the denominator can only decrease.
Hence the conditional probability of selecting a member at the next
draw is at least $\omega_{\mathcal G}$.
Multiply the conditional probabilities of missing it at all $N_0$ draws.
\end{proof}

The bound is useful when early scores assign appreciable weight to
candidates that benefit from further optimization.
For example, if all members of $\mathcal G$ score at least as highly as
all nonmembers, monotonicity of $v$ gives
$\omega_{\mathcal G}\geq|\mathcal G|/N$.
A high early score does not establish this calibration automatically.
Positive weights preserve a chance of retaining each candidate but do
not eliminate false negatives.

Let $\delta_{\mathrm{stop}}$ be the probability that a retained
potentially successful candidate exists but the actual stopping and
discrete-completion policy returns no success.
Combining retention with this definition gives
\begin{equation}
P_{\mathrm{cascade\ hit}}
\geq
\mathbb E\!\left[
\mathbf 1\{\mathcal G\neq\varnothing\}
\left(1-(1-\omega_{\mathcal G})^{N_0}\right)
\right]-\delta_{\mathrm{stop}}.
\label{eq:app_cascade_hit_bound}
\end{equation}
The integrand is defined as zero if $\mathcal G$ is empty.
This separates useful-candidate generation, retention, and stopping losses.
Five soft-stage steps without improvement do not by themselves bound
$\delta_{\mathrm{stop}}$; later improvement remains possible.
The patience rule is therefore an empirical budget policy, not a proven
convergence test.

\paragraph{Compute accounting.}
Let $T_{\max}$ be the full per-leaf optimizer cap including the $b$
screening updates. If survivors resume their optimizer states and receive
$\ell_j\leq T_{\max}-b$ additional updates, then
\begin{equation}
U_{\mathrm{cascade}}
=bN+\sum_{j\in A_Q}\ell_j
\leq bN+N_0(T_{\max}-b).
\label{eq:app_optimizer_budget}
\end{equation}
Relative to fully optimizing all $N$ leaves,
\begin{equation}
NT_{\max}-U_{\mathrm{cascade}}
=(N-N_0)(T_{\max}-b)
+\sum_{j\in A_Q}(T_{\max}-b-\ell_j).
\end{equation}
The first term saves refinement on rejected leaves; the second saves
work from stopping survivors early. If optimization is restarted,
repeated work must be added. Different step types or recycle counts
require measured costs rather than treating every update as equal.

With prefix cost $c_{\mathrm{pre}}$ and continuation cost
$c_{\mathrm{cont}}$, the end-to-end budget is
\begin{equation}
\begin{aligned}
C_{\mathrm{total}}={}&
Kc_{\mathrm{pre}}+K_0Mc_{\mathrm{cont}}
+C_{\mathrm{guide}}+C_{\mathrm{select}}\\
&+C_{\mathrm{screen}}+C_{\mathrm{refine}}
+C_{\mathrm{final\ evaluation}}.
\end{aligned}
\label{eq:app_total_budget}
\end{equation}
The resource-allocation principle is related to early-stopping methods
such as Hyperband~\citep{li2018hyperband}, but this optimizer, its scores,
and probabilistic selection do not inherit Hyperband's guarantees.
AfCycDesign refinement is not a flow or diffusion trajectory.

\paragraph{Implications for the integrated policy.}
Structural selection reduces conditional geometric redundancy.
When that redundancy tracks overlap between conditional sequence laws,
it can reduce repeated exploration and improve a finite-budget coverage
certificate. Under distinct unresolved quality opportunities, broader
coverage can improve discovery probability. As quality evidence becomes
available, weighted retention assigns additional refinement to promising
leaves, with the stated retention bound.
The policy trades exploration against optimization cost; this analysis
neither attributes every empirical gain to a separate stage nor proves
optimality of the chosen hyperparameters.

The results concern computational prediction and search. They do not
make AlphaFold confidence a guarantee of experimental binding, or imply
that full-atom generation cannot benefit from a different scaling policy.
The evaluated scope is the integrated cofolding cascade in
Section~\ref{sec:cascaded_search}.

%% file: tables/app_architecture.tex
\begin{figure}[!ht]
\centering
\includegraphics[width=0.9\textwidth]{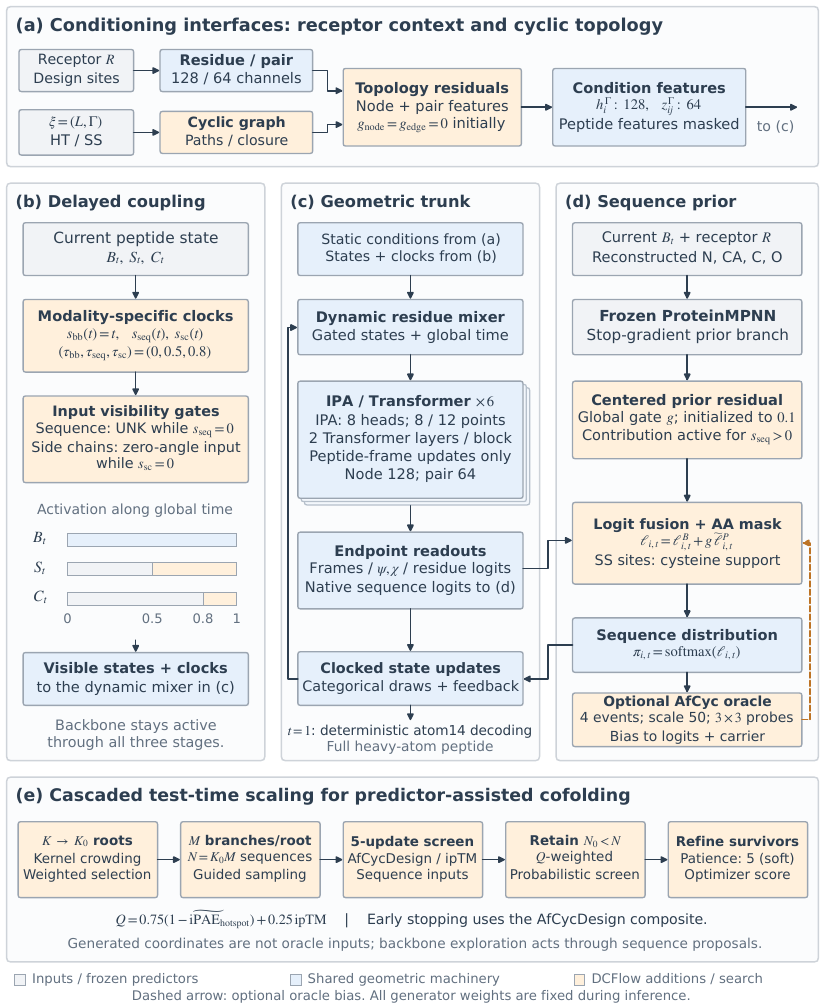}
\caption{\textbf{Detailed architecture and inference interfaces of \method{}.}
(a) Receptor features and cyclic graph features form topology-conditioned
residue and pair representations.
(b) Modality clocks control both state progression and input visibility.
(c) A shared geometric trunk predicts backbone, angular, and sequence
endpoints; sampled sequences re-enter subsequent predictions while the
backbone remains active.
(d) Frozen ProteinMPNN predictions from the current intermediate backbone
enter sequence logits through a trainable residual gate. Optional oracle
feedback supplies persistent sequence biases.
(e) The cofolding cascade selects backbone roots, branches guided
continuations, and allocates further refinement after a short screen.}
\label{fig:dcflow_architecture_detailed}
\end{figure}

This appendix specifies how the three mechanisms in Section~\ref{sec:method}
are instantiated within a shared generative network. Figure~\ref{fig:dcflow_architecture_detailed}
shows the conditioning interfaces, delayed modality activation, geometric
trunk, and residual sequence prior, together with their connection to the
cofolding cascade. The theoretical analysis is given in
Appendix~\ref{app:dc_math}; training settings, refinement protocols, and
evaluation metrics are reported in Appendix~\ref{app:implementation_metrics}.

\subsection{State Representation and Modality-Specific Paths}
\label{app:architecture_state}

\paragraph{Geometric and chemical variables.}
We retain the main-text state $Y_t=(B_t,S_t,C_t)$.
For peptide residue $i\in\{1,\ldots,L\}$, let
$x_{i,t}\in\mathbb R^3$ be its C$\alpha$ translation and
$O_{i,t}\in\mathrm{SO}(3)$ its local orientation, defined using
N--C$\alpha$--C. We use $O_{i,t}$ to distinguish a residue rotation
from the receptor $R$ and the screening score $Q$.
The backbone state $B_t$ also includes a carbonyl-orientation channel
$\psi_{i,t}$ computed from N--C$\alpha$--C--O. This channel controls
the local carbonyl oxygen and is distinct from the conventional
backbone dihedral defined with the next residue's nitrogen.
The side-chain state $C_t$ consists of up to four torsions
$\chi_{i,t}=(\chi_{i,t}^{1},\ldots,\chi_{i,t}^{4})$ per residue,
with residue-specific validity masks.

The implementation represents $S_t$ by a continuous carrier
$u_t\in\mathbb R^{L\times|\mathcal A|}$ and a sampled categorical sequence
$a_t\in\mathcal A^L$. The carrier is mapped to probabilities by softmax;
it is not itself constrained to the probability simplex.
For a residue identity $a$, define
\begin{equation}
\operatorname{Encode}(a)=10\,\operatorname{onehot}(a)-5\mathbf 1,
\label{eq:arch_sequence_encoding}
\end{equation}
where $\mathbf 1$ is the all-ones vector over the alphabet.
The encoded category therefore has value $+5$ and the remaining categories
have value $-5$. The terminal discrete sequence is denoted $S_1$ in the
main text.

\paragraph{Training paths.}
Only peptide states are perturbed. Translation sources are centered
Gaussian samples, orientations are sampled uniformly on $\mathrm{SO}(3)$,
angular sources are uniform on the circle, and $u_0$ is Gaussian with
standard deviation 5. Native endpoint states carry subscript 1.
Using the modality clocks $s_m(t)$ defined in
Section~\ref{sec:delayed_coupling}, the interpolation paths are
\begin{equation}
\begin{aligned}
x_t &= (1-s_{\mathrm{bb}}(t))x_0+s_{\mathrm{bb}}(t)x_1,\\
u_t &= (1-s_{\mathrm{seq}}(t))u_0+s_{\mathrm{seq}}(t)u_1,\\
O_t &= \operatorname{Geo}_{\mathrm{SO}(3)}
       (O_0,O_1;s_{\mathrm{bb}}(t)),\\
\psi_t &= \operatorname{Geo}_{\mathbb S^1}
          (\psi_0,\psi_1;s_{\mathrm{bb}}(t)),\\
\chi_t &= \operatorname{Geo}_{\mathbb S^1}
          (\chi_0,\chi_1;s_{\mathrm{sc}}(t)).
\end{aligned}
\label{eq:arch_interpolation}
\end{equation}
Here $\operatorname{Geo}_{\mathrm{SO}(3)}$ denotes geodesic interpolation
between rotations, and $\operatorname{Geo}_{\mathbb S^1}$ denotes circular
interpolation applied componentwise to the valid angular channels.
Thus the carbonyl channel follows the backbone clock, whereas side-chain
torsions follow the delayed side-chain clock.
Intermediate residue identities are sampled from
$\operatorname{softmax}(u_{i,t})$.

\subsection{Receptor and Cyclization Conditioning}
\label{app:architecture_conditioning}

\paragraph{Static residue and pair features.}
The input contains the designated receptor pocket and peptide design
positions. A residue encoder combines observed identities, local heavy-atom
coordinates, and periodic backbone-angle features into 128-dimensional
representations. A pair encoder combines residue-pair identities,
same-chain relative positions clipped to $\pm32$, atom-distance features,
and orientation features into 64-dimensional representations.
The native peptide's identities and geometric features are masked in these
static encoders during training. Observed pair geometry is retained only
when both residues belong to the visible structural context.
Peptide--receptor interactions are subsequently modeled using the evolving
peptide frames in geometric attention.

The receptor and peptide participate jointly in attention. Receptor hidden
features can therefore change, but its coordinates, identities, and angular
states remain fixed. Static condition features are computed once per
sampling trajectory, while the current peptide state is incorporated at
each network evaluation.

\paragraph{Cyclic graph residuals.}
The cyclization specification $\Gamma$ augments sequential peptide edges
with the prescribed closing edge: an N--C linkage for head-to-tail (HT)
closure or an SG--SG linkage for disulfide (SS) closure.
For each peptide residue pair, the graph encodes shortest-path distance,
whether the selected shortest path crosses a closing edge, and the type
of any direct closing edge. Equal-length paths prefer a path without
the closing edge. These features supplement the original linear relative
positions, allowing the network to distinguish sequence separation from
chemical adjacency.

Let $h_i^{0}$ and $z_{ij}^{0}$ denote the static residue and pair features.
Write $c_\Gamma$ for the topology category, $v_i^\Gamma$ for a closure-site
indicator, and $d_{ij}^\Gamma,r_{ij}^\Gamma,e_{ij}^\Gamma$ for graph
distance, path type, and direct closing-edge type. With learned embeddings
$E_{\mathrm{mode}},E_{\mathrm{site}},E_d,E_r,E_e$, conditioning is injected as
\begin{equation}
\begin{aligned}
h_i^\Gamma
&=h_i^0+g_{\mathrm{node}}
  \bigl(E_{\mathrm{mode}}(c_\Gamma)+E_{\mathrm{site}}(v_i^\Gamma)\bigr),\\
z_{ij}^\Gamma
&=z_{ij}^0+g_{\mathrm{edge}}m_{ij}^{\mathrm{pep}}
  \bigl(E_d(d_{ij}^\Gamma)+E_r(r_{ij}^\Gamma)+E_e(e_{ij}^\Gamma)\bigr).
\end{aligned}
\label{eq:arch_topology_residuals}
\end{equation}
Here $m_{ij}^{\mathrm{pep}}$ is one for peptide--peptide pairs and zero
otherwise. The node residual is applied across residue tokens, with
closure-site indicators identifying the relevant peptide positions.
The unconstrained scalar gates $g_{\mathrm{node}}$ and $g_{\mathrm{edge}}$
are initialized to zero, preserving the pretrained condition features at
initialization.
At specified SS sites, the admissible sequence categories are restricted
to cysteine before prediction probabilities are normalized.

\subsection{Delayed Visibility and the Shared Geometric Trunk}
\label{app:architecture_trunk}

\paragraph{Input visibility and clock conditioning.}
The delayed schedule affects interpolation, network inputs, supervision,
and numerical updates. When $s_{\mathrm{seq}}(t)=0$, peptide residue
identities supplied to the trunk are replaced by an unknown-residue token.
When $s_{\mathrm{sc}}(t)=0$, the side-chain inputs are replaced by fixed
zero angles before periodic encoding. The latter is a zero-angle
representation rather than a zeroed feature vector.
Consequently, inactive random sequence and side-chain states do not enter
the trunk as sample-dependent chemical context.

In addition to the global-time embedding, the network receives a residual
clock feature
\begin{equation}
c_t^{\mathrm{clock}}=
W_{\mathrm{clock}}
\left[
E_{\mathrm{seq}}(s_{\mathrm{seq}}(t));
E_{\mathrm{sc}}(s_{\mathrm{sc}}(t))
\right],
\label{eq:arch_clock_embedding}
\end{equation}
where each $E$ produces a 128-dimensional time embedding, semicolons
denote concatenation, and
$W_{\mathrm{clock}}\in\mathbb R^{128\times256}$ is a bias-free,
zero-initialized projection. The dynamic feature mixer combines these
clock features with the static conditions, current residue identities,
angular encodings, and global time.

\paragraph{Geometric updates and endpoint readouts.}
Six shared blocks process the joint receptor--peptide representation.
Each block applies invariant point attention (IPA), residual normalization,
two Transformer layers, a residue transition, and a local rigid-frame
update. Pair representations are updated between blocks, except after
the final block. IPA combines scalar feature similarity, pair-feature
biases, and distances between points transformed by the current frames.
Its geometric interactions therefore use the evolving peptide pose in
the fixed receptor coordinate system.

Frame updates are restricted to valid peptide residues within every block.
This constraint preserves receptor geometry throughout the forward pass,
rather than restoring it only after the pass.
The final frames provide backbone endpoint predictions; separate
multilayer perceptrons produce 20 residue logits and five angular endpoints
(one carbonyl and four side-chain channels).
These heads share the same trunk. Once sequence or side chains become
visible, their features can affect subsequent backbone predictions without
requiring a separate backbone-refinement network.

\begin{table}[t]
\centering
\small
\caption{Architecture settings of the shared generator. The frozen
ProteinMPNN branch is external to the six geometric blocks.}
\label{tab:dcflow_architecture_settings}
\begin{tabular}{@{}lp{0.48\linewidth}@{}}
\toprule
Component & Configuration\\
\midrule
Residue / pair channels & 128 / 64\\
Geometric blocks & 6\\
IPA heads & 8\\
IPA query/key / value points & 8 / 12 per head\\
Transformer layers / heads & 2 per block / 4\\
Transformer feed-forward width & 128; dropout 0\\
Sequence head & $128\rightarrow128\rightarrow128\rightarrow20$\\
Angular head & $128\rightarrow128\rightarrow128\rightarrow5$\\
Additional clock projection & $256\rightarrow128$; zero initialization\\
Topology residual gates & Two trainable scalars; zero initialization\\
ProteinMPNN residual gate & Global scalar $g$; initialized to $0.1$\\
\bottomrule
\end{tabular}
\end{table}

\subsection{Dynamic ProteinMPNN Residual Prior}
\label{app:architecture_prior}

The ProteinMPNN branch reconstructs peptide N, C$\alpha$, C, and O
coordinates from the \emph{current} frames and carbonyl angles, and combines
them with the observed receptor backbone. It evaluates sequence-independent,
backbone-conditioned residue probabilities using the frozen model
$F_\phi$~\citep{dauparas2022proteinmpnn}.
Neither the native peptide coordinates nor the predicted clean endpoint
is supplied to this branch.

The 20 standard-amino-acid log-probabilities are centered across categories
and fused with the native sequence head according to
\eqref{eq:dc_prior_fusion}. Subtracting the category mean removes a
softmax-invariant offset; it does not change relative residue preferences.
The gate $g$ is a learned global scalar, and the prior contribution is
enabled only when $s_{\mathrm{seq}}(t)>0$.
This activation controls the contribution: the implementation can still
evaluate the frozen branch before activation and multiply its residual by
zero.

ProteinMPNN parameters and its coordinate-input path are excluded from
backpropagation. The fused sequence loss trains the native trunk,
sequence head, and gate $g$. After fusion, disallowed residue categories
are masked before softmax normalization.
Sampling the resulting distribution changes the sequence context provided
at subsequent steps, allowing transferred structural preferences to
influence later backbone and side-chain predictions through the shared
trunk.

\subsection{Supervision, Numerical Updates, and Reconstruction}
\label{app:architecture_dynamics}

\paragraph{Supervision follows modality activation.}
The grouped objective in \eqref{eq:dc_objective} combines translation endpoint error, rotational matching, reconstructed-backbone error, sequence cross-entropy, angular matching, and cyclic-closure supervision. Given the predicted endpoints $\widehat{Y}_1 = (\widehat{B}_1, \widehat{S}_1, \widehat{C}_1)$ and native structures $Y_1 = (B_1, S_1, C_1)$, the individual loss terms are formally defined as follows:

\begin{itemize}[leftmargin=*, nosep]
    \item \textbf{Backbone Loss ($\ell_B$):} Combines translation, rotation, reconstructed atoms, and carbonyl orientation via $\lambda_B \ell_B = \lambda_{\text{trans}} \ell_{\text{trans}} + \lambda_{\text{rot}} \ell_{\text{rot}} + \lambda_{\text{bb}} \ell_{\text{bb}} + \lambda_{\psi} \ell_{\psi}$, where
    \begin{equation}
    \begin{aligned}
        \ell_{\text{trans}} &= \frac{1}{L} \sum_{i=1}^L \|\widehat{x}_{i,1} - x_{i,1}\|_2^2, &
        \ell_{\text{rot}} &= \frac{1}{L} \sum_{i=1}^L \|\mathrm{Log}(R_{i,1}^T \widehat{R}_{i,1})\|_2^2, \\
        \ell_{\text{bb}} &= \frac{1}{3L} \sum_{i=1}^L \sum_{a \in \{\text{N}, \text{C}\alpha, \text{C}\}} \|\widehat{x}_{i,a} - x_{i,a}\|_2^2, \quad &
        \ell_{\psi} &= \frac{1}{L} \sum_{i=1}^L \left( 1 - \cos(\widehat{\psi}_{i,1} - \psi_{i,1}) \right).
    \end{aligned}
    \end{equation}
    $\mathrm{Log}(\cdot)$ denotes the $\mathrm{SO}(3)$ logarithmic map penalizing Lie-algebra tangent deviations; $\ell_\psi$ supervises the $\text{N--C}\alpha\text{--C--O}$ dihedral throughout generation.

    \item \textbf{Sequence Loss ($\ell_S$):} Supervised via cross-entropy for $t \geq \tseq$ against native residues $a_i^*$:
    \begin{equation}
        \ell_S = - \frac{1}{L} \sum_{i=1}^L \log \pi_{i,t,a_i^*}, \quad \text{where } \pi_{i,t} = \mathrm{softmax}(\ell_{i,t}).
    \end{equation}
    Input visibility and ProteinMPNN prior fusion require $t > \tseq$.

    \item \textbf{Side-Chain Angular Loss ($\ell_C$):} Supervises valid channels $\chi = (\chi^1, \dots, \chi^4)$ under mask $m_{i,k} \in \{0, 1\}$:
    \begin{equation}
        \lambda_C \ell_C = \lambda_{\text{angle}} \ell_{\text{angle}} + \lambda_{\text{torsion}} \ell_{\text{torsion}}, \quad \text{with } \ell_{\text{angle}} = \frac{\sum_{i=1}^L \sum_{k=1}^4 m_{i,k} \left( 1 - \cos(\widehat{\chi}_{i,1}^k - \chi_{i,1}^k) \right)}{\sum_{i=1}^L \sum_{k=1}^4 m_{i,k}}.
    \end{equation}
    $\ell_{\text{torsion}}$ provides auxiliary endpoint supervision. Angular terms activate at $t \geq \tsc$ using global time scaling.

    \item \textbf{Cyclic Closure Loss ($\ell_{\text{cyc}}$):} Penalizes closing-atom distance errors on reconstructed heavy atoms:
    \begin{equation}
    \begin{aligned}
        \ell_{\text{cyc}} &= \lambda_{\text{HT}} \mathbf{1}[\Gamma=\mathrm{HT}]\ell_{\text{cyc}}^{\text{HT}} + \mathbf{1}[t \geq \tsc,\Gamma=\mathrm{SS}] \lambda_{\text{SS}} \ell_{\text{cyc}}^{\text{SS}}, \\ 
        \text{where} & \quad
        \begin{cases}
            \ell_{\text{cyc}}^{\text{HT}} = \left( \|\widehat{x}_{1, \text{N}} - \widehat{x}_{L, \text{C}}\|_2 - 1.33\,\text{\AA} \right)^2, \\
            \ell_{\text{cyc}}^{\text{SS}} = \left( \|\widehat{x}_{1, \text{SG}} - \widehat{x}_{L, \text{SG}}\|_2 - 2.03\,\text{\AA} \right)^2.
        \end{cases}
    \end{aligned}
    \end{equation}
    The disulfide term is active only when $t \geq \tsc$, as sulfur coordinates require side-chain packing. These terms supervise closure distances rather than bond-angle or steric constraints.
\end{itemize}

Following the loss configuration in Appendix~\ref{app:model_implementation}, the loss weights are assigned as: $\lambda_{\text{trans}} = 0.5$, $\lambda_{\text{rot}} = 0.5$, $\lambda_{\text{bb}} = 0.25$, $\lambda_{\psi} = 1.0$, $\lambda_S = 1.0$, $\lambda_{\text{angle}} = 1.0$, $\lambda_{\text{torsion}} = 0.5$, $\lambda_{\text{HT}} = 0.25$, and $\lambda_{\text{SS}} = 0.25$.

\paragraph{Source-dependent sampling and feedback.}
At each sampling step, the network predicts clean endpoints and the
sampler advances the active states. Translation and sequence-carrier
increments use the fixed source states, as specified in
\eqref{eq:app_source_updates}; they are not obtained by replacing
the source with the current intermediate state.
Rotations use endpoint-directed geodesic updates, and angular updates
follow the backbone and side-chain clocks.
Predicted sequence probabilities are sampled and encoded before the
carrier update; a categorical sequence is then drawn from the updated
carrier for the next trunk evaluation.
The evaluated schedule uses 200 time points on $[0.01,1]$,
corresponding to 199 update intervals and one final endpoint prediction.

The complete checkpoint $H_t$ introduced in
Appendix~\ref{app:branching_checkpoints} therefore includes the current
state and source-dependent sampling memory. Branching from a backbone
prefix retains this context while providing independent randomness for
the child continuations.
At the endpoint, residue-specific templates and predicted frames and
angles deterministically reconstruct up to 14 heavy atoms per residue.
This atom14 reconstruction is a geometric decoder, not an additional
atom-coordinate diffusion model. The full-atom designation refers to
heavy atoms; hydrogens are not explicitly generated.

\subsection{Interfaces to Guided Cofolding and Cascaded Search}
\label{app:architecture_search_interface}

\paragraph{Direct and predictor-assisted outputs.}
Direct generation returns the reconstructed peptide--receptor complex,
which is subsequently assessed using the common physical protocol.
In the cofolding workflow, the generated \emph{sequence} is passed to
the external predictor and AfCycDesign optimizer.
The generated peptide coordinates are not provided as templates or
initial poses. Structural exploration therefore affects this workflow
through the sequence distributions induced by different backbone
hypotheses.

\paragraph{In-trajectory sequence feedback.}
When guidance is enabled, the external cyclic oracle evaluates the
current sequence and single-residue alternatives at four events during
generation. Each event probes at most three positions and three
alternative identities per position, using a base scale of 50.
The resulting score-dependent biases modify later sequence logits,
carrier updates, and the final categorical prediction.
Oracle coordinates do not replace the peptide state, and generation-time
guidance does not differentiate through the predictor.
Sequence feedback can nevertheless influence later backbone updates
through the shared trunk.
The event schedule and bias rule are specified in
Appendix~\ref{app:afcycdesign_refinement_protocol}.

\paragraph{Cascade checkpoints and selection criteria.}
At $\tseq$, $K$ backbone prefixes are probabilistically reduced to
$K_0$ roots using structural crowding. Each retained root produces
$M$ guided continuations, yielding $N=K_0M$ generated sequences.
The backbones continue evolving during these continuations.
All $N$ sequences receive five initial AfCycDesign updates targeting
ipTM, after which $N_0<N$ candidates are retained probabilistically
using $Q$ from \eqref{eq:dc_search_quality}.
This empirical score allocates further refinement; early stopping instead
uses the composite score reported by AfCycDesign, with patience of five
soft-stage updates.
The final sequences undergo the common independent cofolding evaluation
in Appendix~\ref{app:cofolding_metrics}.
The geometric root selection, sequence branching, and predictor-based
retention thus operate at distinct checkpoints of one design workflow.

%% file: tables/app_data.tex
\subsection{Overview}
CPSea provides a large source of cyclic peptide-receptor complexes~\citep{yang2026cpsea}. Our training corpus combines the CPTrans, CPBind, CPCore, and CPSea-PDB source pools and applies an additional strict quality-control pipeline designed to remove obvious geometric failures while preserving useful structural diversity. 

The curation pipeline has two objectives: (i) to remove obvious geometric artifacts and conformations that are inconsistent with the intended chemical scope of the dataset, and (ii) to retain sufficient structural diversity for downstream learning tasks. To balance stereochemical conservatism against dataset coverage, we define two filtering regimes: a \textit{Strict Pipeline}, which favors conformations close to commonly observed low-strain geometries, and a \textit{Relaxed Pipeline}, which tolerates moderate ring-induced distortion while still excluding explicit geometric failures.

The qualitative locations of the preferred backbone, disulfide, and side-chain conformational regions are motivated by established peptide and protein structural chemistry. The exact numerical boundaries used below, however, are operational curation thresholds adopted for the present dataset and should not be interpreted as universal physical constants. The filtering strategy comprises four components: sequence-length constraints, head-to-tail amide-bond geometry, disulfide-bond quality audits, and the exclusion of side-chain isopeptide topologies.

\subsection{Sequence Length Constraints}

We restrict the working dataset to cyclic peptide chains with sequence length $\ell \ge 6$. This lower bound is used as an operational dataset-scope criterion rather than as a universal physical discontinuity at six residues. Classical conformational analyses show that very small cyclic peptides frequently require highly constrained backbone arrangements, unusual amide geometries, or sequence-specific turn-inducing features to achieve ring closure, whereas larger macrocycles more readily accommodate diverse turn-like and drug-relevant conformations \citep{kessler1982conformation,white2011contemporary,ji2024cyclic}. We therefore adopt $\ell \ge 6$ to reduce the representation of highly specialized, strongly strained microcycles and to focus the present dataset on the macrocyclic regime targeted by the downstream modeling tasks.

\subsection{Head-to-Tail Amide Bond Geometries}

The local geometry of the peptide backbone is governed in part by the amide torsion angle $\omega$. Because peptide bonds strongly favor approximately planar geometries, we use $\omega$ as a primary criterion for backbone plausibility. Non-proline and proline-containing peptide bonds are treated separately.

\paragraph{Non-Proline Peptide Bonds.}

For ordinary secondary amide bonds, the \textit{trans} state is strongly favored, although measurable nonplanarity and occasional substantial deviations are observed in high-resolution protein structures \citep{weiss1998peptide,berkholz2012nonplanar}. Non-proline \textit{cis} peptide bonds can occur in exceptional structural contexts, but they lie outside the canonical backbone regime targeted by this dataset. We therefore discard any conformation containing a non-proline \textit{cis} peptide bond.

For the remaining \textit{trans} non-proline peptide bonds, we measure the deviation from ideal trans geometry according to
\begin{equation}
\label{eq:omega_filter}
\left|\, |\omega| - 180^\circ \,\right| \le \Delta\omega .
\end{equation}
The values of $\Delta\omega$ are operational curation tolerances centered on the literature-supported preference for near-planar peptide bonds, rather than crystallographic pass/fail constants. We set $\Delta\omega = 10.0^\circ$ in the Strict Pipeline to favor near-planar amides, and relax the tolerance to $\Delta\omega = 15.0^\circ$ in the Relaxed Pipeline to preserve topologically closed conformers that remain chemically plausible but exhibit moderate ring-induced distortion \citep{weiss1998peptide,berkholz2012nonplanar}.

\paragraph{Proline-Containing Peptide Bonds.}

Peptide bonds preceding proline are a well-established exception to the behavior of ordinary secondary amides. The pyrrolidine ring reduces the free-energy difference between the \textit{cis} and \textit{trans} states of the preceding $X\text{-}\mathrm{Pro}$ bond, such that both isomers can be populated under biologically relevant conditions \citep{macarthur1991influence,wedemeyer2002proline,yao1994stabilization}. A \textit{cis} $X\text{-}\mathrm{Pro}$ bond can also support a tight local reversal, including type-VI-turn-like conformations, a feature that can be compatible with the conformational demands of cyclic peptide backbones \citep{kessler1982conformation,yao1994stabilization}. We therefore retain both \textit{cis} and \textit{trans} states for proline-containing peptide bonds.

To reduce conformational sampling bias in the curated dataset, we balance the \textit{cis} and \textit{trans} pools of proline-containing structures using the following length-conditioned target fractions:
\begin{equation}
\label{eq:cis_fraction}
f_{\text{cis}}(\ell) =
\begin{cases}
0.75 & \text{if } \ell = 6,\\
0.30 & \text{if } 7 \le \ell \le 10,\\
0.15 & \text{if } 11 \le \ell \le 16.
\end{cases}
\end{equation}
These values are dataset-balancing targets adopted for the present curation workflow; they are not intended to represent universal equilibrium \textit{cis} populations. For each balanced pool of \textit{cis} and \textit{trans} structures, redundant or energetically unfavorable states are pruned by ranking conformations according to their computed Rosetta all-atom final energy and retaining only the lowest-energy subset \citep{alford2017rosetta}. The energy function is used here as a physics-motivated ranking criterion for within-pool pruning rather than as an experimentally calibrated free energy.

\subsection{Disulfide Bond Geometry and Quality Audit}

Disulfide linkages ($\mathrm{Cys}\text{-}\mathrm{S}\text{-}\mathrm{S}\text{-}\mathrm{Cys}$) impose strong geometric constraints on peptide backbones and cysteine side chains. Protein disulfides occupy characteristic regions of bond-length, bond-angle, and dihedral space, although strained but physically realizable disulfides may deviate from the centers of these distributions \citep{thornton1981disulphide,katz1986crystallographically,dombkowski2003disulfide}. We therefore apply a four-part disulfide audit based on S-S distance, the two adjacent bond angles, the $\chi_3$ dihedral, and the remaining cystine torsions.

The literature motivates the centers and qualitative shapes of the accepted conformational regions. The precise \texttt{PASS}, \texttt{WARN}, and \texttt{FAIL} boundaries below are operational quality-control envelopes adopted for this dataset. Under the \textit{Strict Pipeline}, a disulfide bond is retained only if all four criteria receive \texttt{PASS}. Under the \textit{Relaxed Pipeline}, geometries in the broader but still bounded \texttt{WARN} region are retained, whereas explicit geometric failures are discarded.

\paragraph{S-S Distance.}

Protein disulfides cluster around an S-S bond length of approximately $2.03\,\text{\AA}$ \citep{thornton1981disulphide,katz1986crystallographically}. We classify the S-S distance $d_{\mathrm{S-S}}$ as
\begin{equation}
\label{eq:ss_distance}
\operatorname{Status}(d_{\mathrm{S-S}}) =
\begin{cases}
\text{PASS} & \text{if } 1.95\,\text{\AA} \le d_{\mathrm{S-S}} \le 2.12\,\text{\AA},\\
\text{WARN} & \text{if } 1.85\,\text{\AA} \le d_{\mathrm{S-S}} \le 2.20\,\text{\AA},\\
\text{FAIL} & \text{otherwise}.
\end{cases}
\end{equation}

\paragraph{Bond Angles.}

The $\mathrm{C}_\beta\text{-}\mathrm{S}\text{-}\mathrm{S}$ and $\mathrm{S}\text{-}\mathrm{S}\text{-}\mathrm{C}_\beta$ angles of protein disulfides are typically centered near $103^\circ$-$104^\circ$ \citep{thornton1981disulphide,dombkowski2003disulfide}. We evaluate each such angle $\theta$ according to
\begin{equation}
\label{eq:ss_angle}
\operatorname{Status}(\theta) =
\begin{cases}
\text{PASS} & \text{if } 95^\circ \le \theta \le 112^\circ,\\
\text{WARN} & \text{if } 90^\circ \le \theta \le 120^\circ,\\
\text{FAIL} & \text{otherwise}.
\end{cases}
\end{equation}

\paragraph{Dihedral Angle $\chi_3$.}

The central disulfide dihedral $\chi_3$ preferentially occupies near-perpendicular orientations, typically close to $\pm 90^\circ$, consistent with the stereoelectronic preferences of cystine linkages \citep{thornton1981disulphide,katz1986crystallographically}. Ring closure can broaden the accessible $\chi_3$ range, particularly in smaller macrocycles. To avoid disproportionately rejecting smaller rings, we use sequence-length-dependent operational PASS windows:
\begin{equation}
\label{eq:chi3_filter}
\operatorname{Status}(\chi_3) =
\begin{cases}
\text{PASS} &
\text{if }
\begin{cases}
|\chi_3| \in [50^\circ,110^\circ] & \text{for } \ell \le 8,\\
|\chi_3| \in [65^\circ,110^\circ] & \text{for } 9 \le \ell \le 10,\\
|\chi_3| \in [75^\circ,105^\circ] & \text{for } \ell > 10,
\end{cases}\\
\text{WARN} & \text{if outside PASS but within expanded boundaries},\\
\text{FAIL} & \text{otherwise}.
\end{cases}
\end{equation}
The length dependence and the numerical limits in \eqref{eq:chi3_filter} are curation choices centered on the preferred disulfide conformational region; they are not claimed to be universal boundaries for all naturally occurring or synthetically constrained disulfides.

\paragraph{Cystine Torsional States.}

The remaining cystine torsions $(\chi_1,\chi_2,\chi_2',\chi_1')$ are biased toward staggered conformational basins characteristic of protein side chains and disulfide linkages \citep{thornton1981disulphide,dunbrack1993backbone,lovell2000penultimate}. For each angle $\chi_{\text{rot}} \in \{\chi_1,\chi_2,\chi_2',\chi_1'\}$, we calculate the absolute difference to the nearest canonical staggered state $\theta_r \in \{g^+ \approx +60^\circ,g^- \approx -60^\circ,t \approx 180^\circ\}$:
\begin{equation}
\label{eq:rotamer_filter}
\operatorname{Status}(\chi_{\text{rot}}) =
\begin{cases}
\text{PASS} & \text{if } \displaystyle\min_r |\chi_{\text{rot}}-\theta_r| \le 30.0^\circ,\\
\text{WARN} & \text{if } \displaystyle\min_r |\chi_{\text{rot}}-\theta_r| \le 45.0^\circ,\\
\text{FAIL} & \text{otherwise}.
\end{cases}
\end{equation}
The $30^\circ$ and $45^\circ$ windows are operational tolerances used to distinguish core staggered states from progressively less typical torsional geometries. Any structure that triggers a \texttt{FAIL} status in any disulfide parameter is excluded.

\subsection{Exclusion of Side-Chain Isopeptide Cyclizations}

Cyclic peptides can be formed not only through head-to-tail backbone closure or disulfide formation, but also through side-chain-to-side-chain or side-chain-to-terminus isopeptide linkages, including Lys-to-Asp/Glu lactam bridges \citep{lundquist2002new,white2011contemporary,hayes2021approaches}. We exclude side-chain-to-side-chain lactam topologies as a dataset-scope decision intended to keep the present corpus focused on the two cyclization classes emphasized here: canonical head-to-tail macrocycles and disulfide-bridged systems. This restriction should not be interpreted as evidence that side-chain lactams lack chemical, biological, or therapeutic value.

The scope restriction is motivated by practical consistency. Head-to-tail and disulfide-linked peptides are major and widely used classes in cyclic-peptide discovery and development \citep{ji2024cyclic,hayes2021approaches}. By contrast, selective side-chain lactam formation commonly requires topology-specific protecting-group and deprotection strategies, often using orthogonal protection schemes, and may require additional route optimization to control competing coupling sites and cyclization outcomes \citep{lundquist2002new,white2011contemporary,fang2024recent}. These synthetic considerations can complicate uniform, high-throughput treatment across a heterogeneous structural database \citep{hayes2021approaches,fang2024recent,lombardi2025advances}. We therefore restrict the present dataset to head-to-tail and disulfide-bridged systems.

\subsection{Dataset Statistics and Filtering Outcomes}

To quantify the effect of the two quality regimes, we compare the retained dataset sizes under the Strict and Relaxed pipelines across all source subsets (Table~\ref{tab:filtering_results}).

\paragraph{Strict Filtering Outcomes.}

Under the Strict Quality Pipeline ($\Delta\omega = 10.0^\circ$, disulfide \texttt{PASS} only), we obtain a high-fidelity dataset of $73,519$ structurally optimized conformations. This population comprises $58,107$ head-to-tail cyclized peptides and $15,412$ disulfide-bridged structures. Within this regime, the \texttt{CPBind} subset contributes the largest share with $39,539$ conformations, followed by \texttt{CPTrans} with $30,380$, while \texttt{CPCore} and \texttt{CPSea\_PDB} provide $2,637$ and $963$ structures, respectively. This subset is intended for modeling settings that benefit from conservative stereochemical quality control and geometries close to commonly populated conformational regions.

\paragraph{Relaxed Filtering Outcomes.}

Under the Relaxed Quality Pipeline ($\Delta\omega = 15.0^\circ$, disulfide \texttt{PASS} or \texttt{WARN}), the dataset scale expands substantially to $129,198$ valid conformations, representing an approximate $75.7\%$ increase in overall data coverage. This expanded dataset contains $76,929$ head-to-tail and $52,269$ disulfide-bridged structures. Across source subsets, \texttt{CPBind} and \texttt{CPTrans} expand to $74,212$ and $48,541$ conformations, respectively, whereas \texttt{CPCore} and \texttt{CPSea\_PDB} scale to $4,724$ and $1,721$ conformations. These statistics show that allowing moderate, explicitly bounded backbone and side-chain strain recovers a substantial number of otherwise excluded topologically closed conformers, thereby increasing conformational coverage for data-intensive learning applications.

\begin{table}[ht]
\centering
\caption{Dataset statistics under the Strict and Relaxed quality filtering regimes.}
\label{tab:filtering_results}
\small
\begin{tabular}{lcccccc}
\toprule
& \multicolumn{3}{c}{\textbf{Strict Quality Pipeline}} & \multicolumn{3}{c}{\textbf{Relaxed Quality Pipeline}} \\
\cmidrule(r){2-4} \cmidrule(l){5-7}
\textbf{Sub-Dataset} & \textbf{Head-to-Tail} & \textbf{Disulfide} & \textbf{Total} & \textbf{Head-to-Tail} & \textbf{Disulfide} & \textbf{Total} \\
\midrule
CPTrans    & 26,189 & 4,191  & 30,380 & 34,391 & 14,150 & 48,541 \\
CPBind     & 29,149 & 10,390 & 39,539 & 38,897 & 35,315 & 74,212 \\
CPCore     & 1,997  & 640    & 2,637  & 2,583  & 2,141  & 4,724  \\
CPSea\_PDB & 772    & 191    & 963    & 1,058  & 663    & 1,721  \\
\midrule
\textbf{Total} & \textbf{58,107} & \textbf{15,412} & \textbf{73,519} & \textbf{76,929} & \textbf{52,269} & \textbf{129,198} \\
\bottomrule
\end{tabular}
\end{table}

\begin{figure}[ht]
    \centering
    \includegraphics[width=\linewidth]{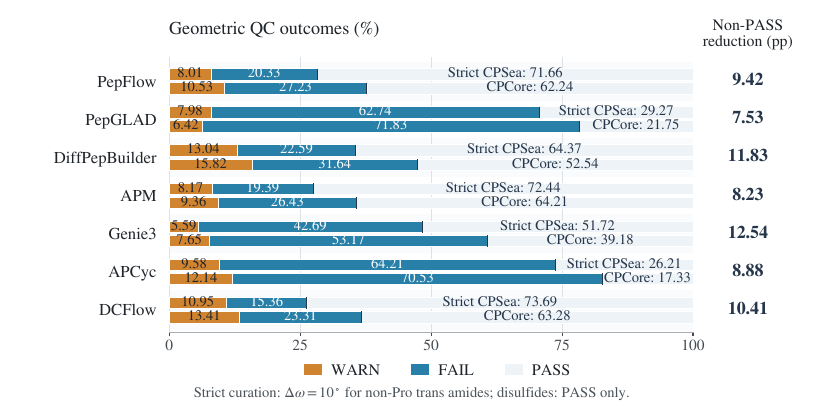}
    \caption{\textbf{Geometric QC by training corpus.} Models trained on Strict CPSea show lower FAIL and combined WARN+FAIL rates across all seven generators under the operational geometric audit described in the appendix. Right-hand values report WARN+FAIL reductions relative to CPCore in percentage points; WARN alone increases for PepGLAD.}
    \label{fig:app_data_ablation}
\end{figure}

\subsection{Ablation Study}

We compare the geometric QC outcomes of samples generated by seven models
trained on Strict CPSea or CPCore, using the audit described above
(Figure~\ref{fig:app_data_ablation}).
Training on Strict CPSea is associated with lower \texttt{FAIL} rates
for all seven generators and a reduction of 7.53--12.54 percentage points
in the combined \texttt{WARN}+\texttt{FAIL} rate.
For \method, the combined rate decreases from 36.72\% to 26.31\%.
The \texttt{WARN} rate also decreases for six generators; PepGLAD is
the exception, increasing from 6.42\% to 7.98\% despite a lower combined
non-\texttt{PASS} rate.

Here, \texttt{WARN} and \texttt{FAIL} serve solely as reference indicators
of geometric data quality under the adopted audit, not as measures of
overall model quality.
We therefore interpret these rates within each model across training
corpora rather than use them to rank different generators.
Since the corpora also differ in size and composition, this comparison
does not isolate the effect of filtering alone.

%% file: tables/app_implementation.tex
\label{app:implementation_metrics}

This appendix specifies the model implementation, generation and refinement
protocols, and metrics used in the experimental tables. Direct full-atom
evaluation assesses generated complexes after a common physical relaxation
procedure. Cofolding evaluation instead assesses the sequences produced by
each design pipeline through a common structure-prediction protocol.

\subsection{Model Training and Sampling}
\label{app:model_implementation}

\paragraph{Geometric representation and conditioning.}
\method{} represents each peptide residue by a rigid frame, a continuous
sequence carrier, and torsional variables. Local frames are constructed from
the N--C$\alpha$--C atoms. The angular representation contains one
N--C$\alpha$--C--O channel controlling carbonyl orientation and up to four
side-chain torsions. Residue templates reconstruct the full heavy-atom
structure from the predicted residue identities, frames, and angles.
Receptor coordinates and identities remain fixed throughout generation,
although receptor features participate in the shared attention network.
During training, the native peptide's identities and geometric features are
masked in the static conditioning encoder.

The shared trunk contains six geometric update blocks. Each block combines
invariant point attention, two Transformer layers, a residue transition,
and rigid-frame updates; pair features are updated between blocks.
Residue and pair representations have dimensions 128 and 64, respectively.
The attention module uses eight heads, eight query/key points and twelve
value points per head. The Transformer uses four heads, a feed-forward
dimension of 128, and zero dropout. Separate prediction heads produce
residue logits and angular endpoints.

\paragraph{Topology and modality activation.}
Cyclization conditions enter through residual node and pair features encoding
the linkage type, endpoint identities, graph distances, and closure edges.
These additions retain the original linear positional features. Topology
gates and the added clock projection are initialized to zero, allowing
initialization from a pretrained PepFlow model~\citep{li2024pepflow}.
The default activation thresholds are
$(\tbb,\tseq,\tsc)=(0,0.5,0.8)$.
Before activation, peptide sequence inputs use an unknown-residue token
and side-chain inputs use fixed zero-angle encodings. Backbone updates
continue after both later modalities become active.
For terminal disulfide generation, the designated endpoint identities are
restricted to cysteine.

A frozen ProteinMPNN~\citep{dauparas2022proteinmpnn} evaluates the current
intermediate backbone together with the observed receptor. Its centered
log-probabilities enter the native sequence logits through the global
gate $g$, initialized to $0.1$. The prior is evaluated without gradients
through either its parameters or coordinate inputs. Its contribution is
enabled when the sequence clock is positive; sequence supervision starts
at $t\geq\tseq$. The analogous side-chain loss threshold is $t\geq\tsc$.

\paragraph{Training configuration.}
Training uses the curated CPSea-derived corpus described in the dataset
appendix. The shared generative network is initialized from a PepFlow
checkpoint trained on the same corpus and optimized for 30 additional
epochs. We use four GPUs with 32 examples per GPU, giving a nominal global
batch size of 128. Adam uses coefficients $(0.9,0.999)$, zero weight decay,
and fixed learning rates of $10^{-5}$ for inherited parameters and $10^{-4}$
for the added clock, topology, and prior-gate parameters. Global gradient
norms are clipped at 100. ProteinMPNN remains frozen. The reported
configuration uses the epoch-30 checkpoint.

For each example, the interpolation time is sampled uniformly from
$[0.01,0.99)$. Translation sources are centered Gaussian samples, rotations
are sampled uniformly, torsions use uniform circular sources, and sequence
carriers use Gaussian noise with scale 5. The endpoint sequence encoding
assigns $+5$ to the native category and $-5$ to other categories.
The loss weights for translation, rotation, reconstructed backbone atoms,
sequence cross-entropy, angular matching, auxiliary endpoint torsions,
HT closure, and SS closure are, respectively,
$0.5,0.5,0.25,1.0,1.0,0.5,0.25,0.25$.
Sequence cross-entropy is averaged over peptide residues and then over
sequence-active examples. Angular supervision uses the observed and
residue-specific torsion masks. SS closure supervision is activated with
the side-chain modality.

\paragraph{Numerical generation.}
The evaluated sampler uses 200 time points spanning $[0.01,1]$, with
categorical sequence draws during the trajectory and a final endpoint
prediction. To specify the source-dependent updates, let $x_t$ denote
peptide C$\alpha$ translations and $u_t$ the continuous sequence carrier,
with fixed sources $x_0,u_0$ and predicted endpoints
$\widehat{x}_1,\widehat{u}_1$. For a step of length $\Delta t$, define
$\Delta s_{\mathrm{seq}}=s_{\mathrm{seq}}(t+\Delta t)-s_{\mathrm{seq}}(t)$.
The updates are
\begin{equation}
\begin{aligned}
x_{t+\Delta t}&=x_t+(\widehat{x}_1-x_0)\Delta t,\\
u_{t+\Delta t}&=u_t+(\widehat{u}_1-u_0)\Delta s_{\mathrm{seq}}.
\end{aligned}
\label{eq:app_source_updates}
\end{equation}
The sequence endpoint is obtained by sampling the predicted categorical
distribution and encoding the sampled identity. Rotations and angles use
their respective geometric updates, with modality clocks controlling the
angular increments. The updated sequence carrier is sampled again to
provide categorical input to the next network evaluation. Thus sequence
discretization occurs throughout sampling, rather than only at its end.

\subsection{Benchmark Conditions and Candidate Budgets}
\label{app:benchmark_budgets}

\paragraph{Direct full-atom generation.}
The LNR benchmark contains 93 receptors: 37 OOD and 56 ID targets under
the receptor-sequence-identity split in
Section~\ref{sec:experimental_setup}. All methods receive the same
target-specific receptor and peptide length. Each baseline generates
32 candidates per target. \method{} generates 32 HT candidates and
32 terminal-SS candidates per target, reported as separate configurations.
Each configuration therefore contains 2,976 candidates in total,
including 1,184 OOD and 1,792 ID candidates. HT and SS outputs are not
pooled to construct a single \method{} result.
DiffPepBuilder and Genie~3 outputs use Rosetta side-chain packing before
the common full-atom evaluation, as indicated in the tables.
The topology assigned during post-processing is described below;
baseline models are not assumed to have identical native topology controls.

\paragraph{Predictor-assisted design.}
The cofolding benchmark comprises the 12 published HFGuidedDesign
targets~\citep{hu2026hfguideddesign}. Each reported configuration contributes
32 final candidates per target, or 384 candidates in total.
The \emph{Guide} flag denotes generation-time oracle feedback;
\emph{AFCyc optim.} denotes subsequent sequence-initialized optimization;
and \emph{Search} denotes the complete cascaded policy. Equal final-candidate counts do not imply equal oracle costs:
guided proposals, intermediate branches, and refinement updates are
separate contributions to inference compute. For \method{} with inference-time search, the search procedure is configured
with parameters $K=16$, $K_0=8$, $M=8$ (yielding $N=64$), and $N_0=32$.

\subsection{Physical Relaxation and Topology Handling}
\label{app:rosetta_protocol}

\paragraph{Common relaxation protocol.}
Generated or independently predicted complexes are processed with Rosetta
FastRelax using \texttt{ref2015}~\citep{alford2017rosetta}.
For scoring, the peptide and receptor are standardized as chains A and B,
respectively, with receptor numbering and coordinates mapped to the
corresponding target. The same preparation is applied across methods.
FastRelax performs five repeats of repacking and minimization and returns
one relaxed output per input. Both peptide and receptor backbone and
side-chain degrees of freedom are permitted to relax: the protocol applies
neither receptor-specific movement restrictions nor receptor coordinate
restraints. The receptor is therefore fixed during neural generation but
flexible during this shared physical post-processing step.

\begin{table}[t]
\centering
\small
\caption{Settings of the common Rosetta relaxation and interface analysis
protocol. Closure restraints act during relaxation and are excluded from
the reported energy scores.}
\label{tab:rosetta_parameters}
\begin{tabular}{@{}lp{0.60\linewidth}@{}}
\toprule
Component & Setting\\
\midrule
Energy function & \texttt{ref2015}\\
FastRelax repeats & 5\\
Minimizer iteration cap & 1,000 per minimization stage\\
Outputs & One relaxed structure per input\\
Receptor mobility & Backbone and side chains permitted to relax\\
Closure-restraint weight & 1.0; constraint ramping disabled\\
Separated-state repacking & Disabled\\
Interface geometry & Hydrogen bonds, shape complementarity, and Packstat\\
\bottomrule
\end{tabular}
\end{table}

\paragraph{Chemical closure.}
HT structures use a declared bond between the first residue's nitrogen
and the last residue's carbonyl carbon. Terminal SS structures use a
declared disulfide between the terminal cysteine sulfur atoms, with
side-chain repacking enabled. Let $d$ be the closing-atom distance,
$d_0$ its target length, and $\sigma$ the restraint width.
The harmonic contribution is
\begin{equation}
E_{\mathrm{closure}}=\left(\frac{d-d_0}{\sigma}\right)^2.
\label{eq:app_harmonic_closure}
\end{equation}
For HT closure, $(d_0,\sigma)=(1.33,0.02)\,\text{\AA}$;
for SS closure, $(d_0,\sigma)=(2.03,0.03)\,\text{\AA}$.
The atom-pair restraint has weight 1.0 throughout relaxation.
No additional closure-specific angle, planarity, or dihedral restraints
are introduced; the remaining geometry is governed by the declared
chemical topology and the standard score function.

\paragraph{Topology fidelity and unsuccessful outputs.}
\label{app:scoring_failure_policy}
For \method{}, relaxation follows the requested HT or SS topology.
An output that changes the requested linkage type is counted as a failure.
For baseline configurations processed without a prescribed linkage type,
terminal-cysteine sequences are first attempted as SS structures;
unsuccessful SS attempts can fall back to HT closure. Other sequences are
processed as HT structures. The SS-to-HT fallback is triggered when
SS closure cannot be completed, including a relaxed sulfur distance
greater than $d_0+3\sigma$. A baseline output is accepted only if a
permitted closure attempt succeeds; failed outputs remain in the
generation denominator. This fallback policy is part of the evaluation
adapter, rather than a claim about each baseline's intrinsic capabilities.

\paragraph{Unrestrained scoring.}
After relaxation, closure-restraint contributions are excluded from
the reported \texttt{ref2015} scores. Total-complex energy supplies the
Stability metric. InterfaceAnalyzer evaluates the peptide--receptor
interface with separated-state repacking disabled; hydrogen-bond,
surface-area, shape-complementarity, and packing calculations are
enabled as required by the reported tables.
Its cross-interface score, \texttt{dG\_cross}, and separated-state score,
\texttt{dG\_separated}, are distinct quantities.
The former sums cross-interface energy contributions; the latter compares
the bound and separated states. These empirical energies are expressed
in Rosetta energy units (REU).

\subsection{Direct Full-Atom Generation Metrics}
\label{app:full_atom_metrics}

The full-atom tables report Cyc., Stability, $D_{\mathrm{seq}}$,
$D_{\mathrm{struct}}$, and Rosetta $dG$. Following the benchmark convention
adapted from APCyc~\citep{zhao2026apcyc}, all quantities are evaluated within
each receptor before aggregation over the specified All, OOD, or ID split.

\paragraph{Cyclization success candidates.}
A candidate is considered successfully cyclized if the prescribed cyclization bond is successfully formed under the topology-handling protocol described above in \eqref{eq:app_harmonic_closure}. Cyc. is defined as the percentage of all generated candidates satisfying this bond-length criterion, with all unsuccessful outputs included in the denominator.

\paragraph{Energetic metrics.}
Stability is the total \texttt{ref2015} score of the relaxed complex
without closure-restraint contributions. The Rosetta $dG$ column uses
the interface \texttt{dG\_cross} score. Lower values are preferred for
both metrics. Each is averaged over valid candidates for a receptor,
and the resulting receptor-wise means are averaged across the split.
These are conditional energy summaries of valid outputs, rather than
unconditional estimates of binding affinity.

\paragraph{Sequence and structural diversity.}
Valid candidates for a receptor are clustered separately in sequence
and structure space using single-linkage clustering. Sequence distance
is one minus sequence similarity, with threshold 0.4; structural distance
is peptide C$\alpha$ RMSD, with threshold $4.0\,\text{\AA}$.
For a receptor with $n_{\mathrm{valid}}>0$ valid candidates, let
$c_{\mathrm{seq}}$ and $c_{\mathrm{struct}}$ be the respective numbers
of clusters. The reported diversities are
\begin{equation}
D_{\mathrm{seq}}=\frac{c_{\mathrm{seq}}}{n_{\mathrm{valid}}},
\qquad
D_{\mathrm{struct}}=\frac{c_{\mathrm{struct}}}{n_{\mathrm{valid}}}.
\label{eq:app_full_atom_diversity}
\end{equation}
Higher values indicate less redundancy within the evaluated candidate set.
These cluster fractions are computed independently for each receptor
and then averaged over the reported split. Cyc. retains all attempted
samples in its denominator, whereas these diversity and energy summaries
use the valid subset.

\subsection{Cofolding Evaluation and Reported Metrics}
\label{app:cofolding_metrics}

\paragraph{Independent structure prediction.}
Final discrete sequences are independently evaluated with five
AlphaFold-Multimer-v3 models and three recycles per model using cyclic
peptide positional encoding. The input comprises the receptor, its MSA,
the peptide sequence, and the designated cyclization configuration.
Generated peptide coordinates are not supplied as structural templates.
Prediction metrics are first averaged over the five models for each
candidate. These independent evaluations, rather than scores from
intermediate optimization states, determine the confidence tables.

\paragraph{Prediction confidence.}
We report ipTM, the predicted interface TM-score, and peptide pLDDT,
the average per-residue confidence over the peptide chain on a 0--100
scale. Interface iPAE summarizes predicted alignment error across the
peptide--receptor interface; hotspot iPAE restricts the receptor-side
evaluation to the designated hotspot residues. Both raw iPAE values
are divided by 31 before reporting or inclusion in composite scores.
We denote these normalized quantities by
$\widetilde{\mathrm{iPAE}}_{\mathrm{interface}}$ and
$\widetilde{\mathrm{iPAE}}_{\mathrm{hotspot}}$.
Higher ipTM and peptide pLDDT, and lower iPAE, indicate greater predicted
confidence.

\paragraph{Composite scores.}
Let $S$ denote the interface-based score used in the main comparison,
and $S_{\mathrm{hotspot}}$ the hotspot-based score used in the search
comparison. They are defined by
\begin{equation}
\begin{aligned}
S &=
\mathrm{ipTM}
+\frac{\mathrm{pLDDT}_{\mathrm{pep}}}{100}
+1-\widetilde{\mathrm{iPAE}}_{\mathrm{interface}},\\
S_{\mathrm{hotspot}} &=
\mathrm{ipTM}
+\frac{\mathrm{pLDDT}_{\mathrm{pep}}}{100}
+1-\widetilde{\mathrm{iPAE}}_{\mathrm{hotspot}}.
\end{aligned}
\label{eq:app_composite_scores}
\end{equation}
Here $\mathrm{pLDDT}_{\mathrm{pep}}$ is peptide pLDDT.
The tables use these sums without an additional factor of $1/3$.
Hotspot iPAE is reported alongside the other metrics in the main
comparison but does not enter its interface-based score $S$.

\paragraph{Strict hits.}
A candidate is a strict hit if its model-averaged metrics satisfy all
three conditions:
\begin{equation}
\mathrm{ipTM}\geq0.7,\qquad
\mathrm{pLDDT}_{\mathrm{pep}}\geq70,\qquad
\widetilde{\mathrm{iPAE}}_{\mathrm{interface}}\leq0.3.
\label{eq:app_strict_hit}
\end{equation}
The count and percentage are computed over all 384 final candidates.
The same interface-iPAE threshold is used in the search comparison,
even when its ranking score is $S_{\mathrm{hotspot}}$.

\paragraph{Mean and Best summaries.}
Mean averages the candidate-level metrics over all 384 sequences.
For Best, one candidate is selected from the 32 candidates for each
target using the relevant composite score: $S$ in the main comparison
and $S_{\mathrm{hotspot}}$ in the search comparison. All Best columns
then summarize these same 12 selected candidates. In particular,
Best does not select a different sequence for each metric or a
different AlphaFold model for each candidate.
The main and search comparisons use separate random-seed sets;
their baseline values need not coincide.

\paragraph{Rosetta interface metrics.}
Independent cofolded structures are relaxed with the common protocol
in Appendix~\ref{app:rosetta_protocol}. Let $E_{\mathrm{bound}}$ be
the relaxed complex score and $E_{\mathrm{separated}}$ the score after
separating the peptide and receptor without repacking. The cofolding
energy column reports
\begin{equation}
\Delta G_{\mathrm{sep}}
=E_{\mathrm{bound}}-E_{\mathrm{separated}}
=\texttt{dG\_separated}.
\label{eq:app_separated_energy}
\end{equation}
Buried solvent-accessible surface area is denoted
$\Delta\mathrm{SASA}$ (the \texttt{dSASA\_int} output, in
$\text{\AA}^2$). Area-normalized energy is
$100\Delta G_{\mathrm{sep}}/\Delta\mathrm{SASA}$,
in REU per $100\,\text{\AA}^2$.
Lower raw and normalized energies are preferred.
The remaining table columns are interface hydrogen-bond count
(\texttt{hbonds\_int}), buried unsatisfied polar groups
(\texttt{delta\_unsatHbonds}, labeled \emph{Unsat.\ H-bonds},
measuring the change upon interface formation),
surface shape complementarity (SC), and interface packing quality
(Packstat). Fewer buried unsatisfied polar groups, and higher
hydrogen-bond counts, SC, and Packstat are preferred in the reported
comparisons. Buried surface area measures interface extent and is
reported separately from its energy density.

\paragraph{Physical-score aggregation.}
The Rosetta analysis uses four sampled candidates per target for
the Mean summaries and the AF-selected Top-1 candidate for Best.
Each candidate's Rosetta metrics are averaged across its five
independently predicted and relaxed models, followed by aggregation
within and across targets. AF Top-1 selection uses the corresponding
confidence composite, rather than selecting the lowest Rosetta
energy retrospectively. Thus the Rosetta summaries and the confidence
tables have different evaluation sample counts, as stated in the
interface-analysis figure. Predictor confidence and Rosetta scores
serve as computational design criteria, not measurements of
experimental affinity.

\subsection{Generation-Time Guidance and AfCycDesign Refinement}
\label{app:afcycdesign_refinement_protocol}

\paragraph{Oracle-guided sequence proposals.}
The external guide used by \method{} and HFGuidedDesign~\citep{hu2026hfguideddesign} evaluates current peptide
sequences and their single-residue alternatives with a cyclic
AfCycDesign oracle~\citep{rettie2025afcycdesign}, following the
sequence-guidance principle of HFGuidedDesign.
For efficiency, this online oracle uses
\texttt{model\_1\_multimer\_v3}, one recycle, cyclic positional offsets
of type 2, the receptor MSA, and the designated hotspots.
It receives sequences rather than the generated peptide coordinates,
and its predicted coordinates do not overwrite the generative state.

Guidance is applied at four events during the 200-point trajectory:
sampling steps 125, 150, 175, and 199, corresponding approximately
to global times $0.627,0.751,0.876,0.995$.
At each event and for each candidate, the three positions with
highest categorical entropy are considered. At each selected position,
up to three alternative amino acids are selected from the model's
high-probability categories, excluding the current identity.
Every proposal changes only one residue. Thus an event evaluates
at most nine alternatives, together with the unchanged sequence;
the setting does not denote three simultaneous mutations.
The guidance base scale is
$\lambda_{\mathrm{guide}}=50.0$.

\begin{table}[t]
\centering
\small
\caption{Generation-time AfCycDesign guidance settings for \method{}.
Position and alternative limits apply per candidate and per event.}
\label{tab:app_guide_settings}
\begin{tabular}{@{}lp{0.53\linewidth}@{}}
\toprule
Setting & Value\\
\midrule
Base guidance scale & 50.0\\
Maximum probed positions & 3\\
Alternatives per position & 3\\
Number of events & 4\\
Event steps & 125, 150, 175, 199 of the 200-point schedule\\
Maximum evaluations per event & 9 single-residue alternatives plus 1 unchanged sequence\\
Online predictor & One AlphaFold-Multimer-v3 model, one recycle\\
Bias retention at each event & 0.9\\
\bottomrule
\end{tabular}
\end{table}

\paragraph{Score-dependent probability bias.}
For the online guide, $S_{\mathrm{hotspot}}$ in
\eqref{eq:app_composite_scores} is computed from the online oracle.
Let $e_{\mathrm{base}}=|3-S_{\mathrm{hotspot}}^{\mathrm{base}}|$
and $e_{\mathrm{mut}}=|3-S_{\mathrm{hotspot}}^{\mathrm{mut}}|$
be the score gaps for the current sequence and a single-residue
alternative. At global time $t$, the proposed residue receives
the additive bias
\begin{equation}
\Delta b=
\begin{cases}
\displaystyle
\lambda_{\mathrm{guide}}t
\frac{e_{\mathrm{base}}-e_{\mathrm{mut}}}
     {\max(e_{\mathrm{base}},0.1)},
&e_{\mathrm{mut}}<e_{\mathrm{base}},\\[6pt]
\displaystyle
-0.3\lambda_{\mathrm{guide}}t
\frac{e_{\mathrm{mut}}-e_{\mathrm{base}}}
     {\max(e_{\mathrm{mut}},0.1)},
&\text{otherwise}.
\end{cases}
\label{eq:app_online_bias}
\end{equation}
Here $\Delta b$ is accumulated at the tested position and amino-acid
category. Existing accumulated biases are multiplied by 0.9 at each
guidance event before new contributions are added.
The base scale therefore controls score-dependent logit adjustments;
it is neither an optimizer learning rate nor a direct multiplier
of backbone coordinates.

The accumulated bias modifies subsequent sequence logits, the
updated continuous sequence carrier, and the final sequence prediction.
Its addition to the sequence carrier is not multiplied by the local
time increment. The biased categorical sequence re-enters the shared
trunk at subsequent steps and can thereby affect later backbone and
side-chain updates. This is a discrete proposal-and-evaluation guide,
without differentiation through the online folding predictor.

\paragraph{Post-generation optimization.}
In the fixed-budget comparisons, configurations marked
\emph{AFCyc optim.=Yes} initialize AfCycDesign with the generated
sequence and run the specified 160-update refinement protocol using
ipTM and hotspot-iPAE objectives.
Configurations marked \emph{AFCyc optim.=No} proceed directly to
independent final evaluation. This optimization stage is distinct
from the four generation-time guidance events, and the generated
coordinates are not used as pose initializations.
Final discrete sequences are reevaluated with the five-model,
three-recycle protocol above before confidence aggregation.

\paragraph{Cascaded allocation and early stopping.}
The search configuration uses the integrated policy of
Section~\ref{sec:cascaded_search}: $K$ backbone prefixes are
probabilistically reduced to $K_0$ roots, each root produces $M$
guided continuations, and $N=K_0M$ proposals receive five initial
AfCycDesign updates targeting ipTM. Of these, $N_0<N$ are
probabilistically retained using the empirical screening score
$Q=0.75(1-\widetilde{\mathrm{iPAE}}_{\mathrm{hotspot}})
+0.25\,\mathrm{ipTM}$.
The retained candidates receive further optimization, with
early stopping after five consecutive soft-stage updates without
improvement in the composite score reported by AfCycDesign.
Thus early retention uses $Q$, whereas stopping uses the optimizer's
reported composite score. The resulting discrete sequences undergo
the same independent final evaluation as other configurations.
Search does not require every proposal to exhaust the fixed-budget
refinement schedule. Oracle calls during guidance, the five-update
screen, subsequent refinement, and independent evaluation all
contribute to its total computational cost.

%% file: tables/app_lnr.tex
% Preamble:
% \usepackage{booktabs,multirow}

\begin{table}[h]
\centering
\caption{\textbf{Direct full-atom cyclic peptide generation on LNR.}
\method (SS) achieves the highest cyclization success across All/OOD/ID,
reaching 99.1\% overall, while both \method{} variants reach 100\% on ID receptors.
\method (HT) achieves the best overall Stability and Rosetta $dG$ scores,
together with the highest sequence diversity on All and OOD.
\method (SS) leads structural diversity across all three splits.
Bold and underlined values indicate the best and second-best results,
respectively, within each split.}
\label{tab:app_full_atom_results}

\begingroup
\normalsize
\setlength{\tabcolsep}{2.5pt}
\renewcommand{\arraystretch}{1.15}

% ---------- Cyclization, stability, and binding energy ----------
\begin{tabular*}{\textwidth}{@{\extracolsep{\fill}}l*{9}{r}@{}}
\toprule
\multirow{2}{*}{Method}
& \multicolumn{3}{c}{Cyc. (\%) $\uparrow$}
& \multicolumn{3}{c}{Stability $\downarrow$}
& \multicolumn{3}{c}{Rosetta $dG\downarrow$} \\
\cmidrule(lr){2-4}
\cmidrule(lr){5-7}
\cmidrule(lr){8-10}
& All & \textsc{ood} & \textsc{id}
& All & \textsc{ood} & \textsc{id}
& All & \textsc{ood} & \textsc{id} \\
\midrule

PepFlow
& 94.0 & 92.9 & 94.6
& \underline{-737.76} & \textbf{-961.10} & -590.27
& -20.23 & -24.78 & -17.24 \\

PepGLAD
& 91.4 & \underline{97.5} & 87.4
& -635.44 & -720.20 & -579.45
& -14.33 & -15.59 & -13.50 \\

DiffPepBuilder$^{\dagger}$
& 93.7 & 92.2 & 94.6
& -719.75 & -939.25 & -574.79
& -19.09 & -20.66 & -18.05 \\

APM
& 98.9 & 97.3 & \textbf{100.0}
& -645.42 & -869.51 & -497.36
& -19.90 & \underline{-25.06} & -16.50 \\

Genie~3$^{\dagger}$
& 82.4 & 83.6 & 81.6
& -530.93 & -726.29 & -399.28
& -11.73 & -14.07 & -10.16 \\

APCyc
& 82.5 & 95.4 & 73.9
& -654.07 & -713.21 & \textbf{-616.07}
& \underline{-21.07} & -22.79 & \textbf{-19.91} \\

\midrule
\textbf{\method (HT)}
& \underline{99.0} & \underline{97.5} & \textbf{100.0}
& \textbf{-737.87} & \underline{-940.05} & \underline{-604.28}
& \textbf{-22.82} & \textbf{-28.37} & \underline{-19.14} \\

\textbf{\method (SS)}
& \textbf{99.1} & \textbf{97.6} & \textbf{100.0}
& -627.74 & -807.36 & -509.06
& -17.69 & -21.17 & -15.39 \\

\bottomrule
\end{tabular*}

\par\vspace{0.7em}

% ---------- Sequence and structural diversity ----------
\begin{tabular*}{\textwidth}{@{\extracolsep{\fill}}l*{6}{r}@{}}
\toprule
\multirow{2}{*}{Method}
& \multicolumn{3}{c}{$D_{\mathrm{seq}}\uparrow$}
& \multicolumn{3}{c}{$D_{\mathrm{struct}}\uparrow$} \\
\cmidrule(lr){2-4}
\cmidrule(lr){5-7}
& All & \textsc{ood} & \textsc{id}
& All & \textsc{ood} & \textsc{id} \\
\midrule

PepFlow
& 0.646 & 0.686 & 0.619
& 0.781 & 0.948 & 0.670 \\

PepGLAD
& \underline{0.881} & \underline{0.895} & \textbf{0.872}
& \underline{0.862} & 0.932 & \underline{0.817} \\

DiffPepBuilder$^{\dagger}$
& 0.349 & 0.446 & 0.286
& 0.830 & \underline{0.979} & 0.732 \\

APM
& 0.141 & 0.173 & 0.120
& 0.827 & 0.938 & 0.754 \\

Genie~3$^{\dagger}$
& 0.247 & 0.253 & 0.242
& 0.807 & 0.823 & 0.796 \\

APCyc
& 0.845 & 0.867 & 0.831
& 0.762 & 0.849 & 0.705 \\

\midrule
\textbf{\method (HT)}
& \textbf{0.883} & \textbf{0.940} & \underline{0.845}
& 0.648 & 0.934 & 0.459 \\

\textbf{\method (SS)}
& 0.389 & 0.668 & 0.205
& \textbf{0.962} & \textbf{0.991} & \textbf{0.943} \\

\bottomrule
\end{tabular*}

\par\vspace{0.35em}
\parbox{\textwidth}{\footnotesize
$^{\dagger}$ DiffPepBuilder and Genie~3 require the Rosetta side-chain packer.
}
\endgroup

\end{table}

%% file: tables/app_cofold.tex
% Appendix tables for the existing ICLR document.
% Required preamble packages: \usepackage{booktabs,multirow,amsmath}
% Insert after \appendix with: \input{appendix_cofolding_tables}
% Tables use full \textwidth, normal body size, and page floats.
% The four labels are separate from any tables retained in the main text.
% Scaling baselines, Search/FKS names, and score definitions follow their
% respective source tables; they are not silently pooled with cofolding.

\begin{table}[p]
\centering
\caption{\textbf{Detailed cofolding performance across 12 targets.} Each configuration contains 384 candidates (32 per target), independently evaluated with five AlphaFold-Multimer-v3 models and three recycles. Mean aggregates all candidates; Best averages the target-wise Top-1 candidates selected by composite score.}
\label{tab:app_cofold_main}
\begingroup
\normalsize
\setlength{\tabcolsep}{2pt}
\renewcommand{\arraystretch}{1.12}
% Source: cofold_main.tex
\noindent\makebox[\textwidth][l]{\textbf{(a) Predicted complex quality}}\par
\vspace{3pt}
\begin{tabular*}{\textwidth}{@{\extracolsep{\fill}}lcc*{6}{r}@{}}
\toprule
\multirow{2}{*}{Method} & \multirow{2}{*}{Guide} & \multirow{2}{*}{\shortstack{AFCyc\\optim.}} & \multicolumn{2}{c}{ipTM $\uparrow$} & \multicolumn{2}{c}{Peptide pLDDT $\uparrow$} & \multicolumn{2}{c}{Interface iPAE $\downarrow$} \\
\cmidrule(lr){4-5}
\cmidrule(lr){6-7}
\cmidrule(lr){8-9}
 &  &  & Mean & Best & Mean & Best & Mean & Best \\
\midrule
\multirow{2}{*}{RFPeptides} & -- & No & 0.516 & 0.763 & 47.97 & 71.42 & 0.434 & 0.262 \\
 & -- & Yes & 0.666 & 0.877 & 61.21 & 84.90 & \underline{0.334} & 0.170 \\
\midrule
\multirow{4}{*}{HFGuidedDesign} & No & No & 0.470 & 0.627 & 41.65 & 57.39 & 0.477 & 0.365 \\
 & No & Yes & 0.662 & 0.886 & 60.27 & 85.57 & 0.341 & 0.167 \\
 & Yes & No & 0.519 & 0.720 & 46.38 & 67.36 & 0.442 & 0.295 \\
 & Yes & Yes & \underline{0.674} & \underline{0.889} & \underline{61.60} & \underline{86.57} & \underline{0.334} & \underline{0.161} \\
\midrule
\multirow{4}{*}{\method{}} & No & No & 0.448 & 0.634 & 39.76 & 58.29 & 0.493 & 0.367 \\
 & No & Yes & 0.648 & 0.886 & 58.81 & 85.99 & 0.352 & 0.165 \\
 & Yes & No & 0.502 & 0.705 & 44.69 & 65.60 & 0.454 & 0.306 \\
 & Yes & Yes & \textbf{0.683} & \textbf{0.903} & \textbf{62.50} & \textbf{87.26} & \textbf{0.326} & \textbf{0.146} \\
\bottomrule
\end{tabular*}

\par\vspace{0.7em}

% Source: cofold_main.tex
\noindent\makebox[\textwidth][l]{\textbf{(b) Hotspot alignment, composite score, and strict hits}}\par
\vspace{3pt}
\begin{tabular*}{\textwidth}{@{\extracolsep{\fill}}lcc*{6}{r}@{}}
\toprule
\multirow{2}{*}{Method} & \multirow{2}{*}{Guide} & \multirow{2}{*}{\shortstack{AFCyc\\optim.}} & \multicolumn{2}{c}{Hotspot iPAE $\downarrow$} & \multicolumn{2}{c}{Composite $S\uparrow$} & \multicolumn{2}{c}{Strict hits $\uparrow$} \\
\cmidrule(lr){4-5}
\cmidrule(lr){6-7}
\cmidrule(lr){8-9}
 &  &  & Mean & Best & Mean & Best & Count & Rate (\%) \\
\midrule
\multirow{2}{*}{RFPeptides} & -- & No & 0.318 & 0.182 & 1.562 & 2.215 & 17/384 & 4.4 \\
 & -- & Yes & \underline{0.234} & 0.115 & 1.944 & 2.556 & 112/384 & 29.2 \\
\midrule
\multirow{4}{*}{HFGuidedDesign} & No & No & 0.369 & 0.283 & 1.410 & 1.836 & 2/384 & 0.5 \\
 & No & Yes & 0.238 & 0.105 & 1.923 & 2.575 & \underline{128}/\underline{384} & \underline{33.3} \\
 & Yes & No & 0.336 & 0.223 & 1.541 & 2.099 & 9/384 & 2.3 \\
 & Yes & Yes & \underline{0.234} & 0.104 & \underline{1.957} & \underline{2.593} & \underline{128}/\underline{384} & \underline{33.3} \\
\midrule
\multirow{4}{*}{\method{}} & No & No & 0.376 & 0.268 & 1.353 & 1.851 & 4/384 & 1.0 \\
 & No & Yes & 0.248 & \underline{0.103} & 1.885 & 2.581 & 93/384 & 24.2 \\
 & Yes & No & 0.332 & 0.212 & 1.495 & 2.055 & 14/384 & 3.6 \\
 & Yes & Yes & \textbf{0.228} & \textbf{0.091} & \textbf{1.982} & \textbf{2.629} & \textbf{141/384} & \textbf{36.7} \\
\bottomrule
\end{tabular*}
\par\vspace{4pt}
\parbox{\textwidth}{\footnotesize\raggedright
\textit{Notes.} Bold and underlined values indicate the best and second-best results across the ten configurations, respectively, separately for each numeric column; ties share the same rank. Best uses the same 12 candidates for every metric: for each target, select the candidate with the highest composite score among 32, then average across targets. These are not independently selected metric-wise extrema or a global maximum. Each candidate's metrics are averaged over five models. Guide denotes generation-time guidance; AFCyc optim. denotes subsequent sequence-initialized gradient optimization. \method{} uses hotspot persistent-transition guidance. Both iPAE metrics are normalized by 31; hotspot iPAE is not included in $S$. Strict hits are counted over all 384 candidates and require $\mathrm{ipTM}\geq0.7$, peptide pLDDT $\geq70$, and interface iPAE $\leq0.3$. Optimization-trajectory scores and pose-initialized arms are excluded.
}
\endgroup
\end{table}

\begin{table}[p]
\centering
\caption{\textbf{Detailed Rosetta interface metrics for cofolding across 12 targets.} Mean and Best are shown separately for each of the ten generation and optimization configurations.}
\label{tab:app_cofold_rosetta}
\begingroup
\normalsize
\setlength{\tabcolsep}{2pt}
\renewcommand{\arraystretch}{1.12}
% Source: cofold_rosetta.tex
\noindent\makebox[\textwidth][l]{\textbf{(a) Interface energetics and buried surface area}}\par
\vspace{3pt}
\begin{tabular*}{\textwidth}{@{\extracolsep{\fill}}lcc*{6}{r}@{}}
\toprule
\multirow{3}{*}{Method} & \multirow{3}{*}{Guide} & \multirow{3}{*}{\shortstack{AFCyc\\optim.}} & \multicolumn{2}{c}{$\Delta G_{\mathrm{sep}}\downarrow$} & \multicolumn{2}{c}{$100\Delta G_{\mathrm{sep}}/\Delta\mathrm{SASA}\downarrow$} & \multicolumn{2}{c}{$\Delta\mathrm{SASA}\uparrow$} \\
 &  &  & \multicolumn{2}{c}{(REU)} & \multicolumn{2}{c}{(REU/$100\,\text{\AA}^{2}$)} & \multicolumn{2}{c}{($\text{\AA}^{2}$)} \\
\cmidrule(lr){4-5}
\cmidrule(lr){6-7}
\cmidrule(lr){8-9}
 &  &  & Mean & Best & Mean & Best & Mean & Best \\
\midrule
\multirow{2}{*}{RFPeptides} & -- & No & -30.289 & -32.745 & -2.824 & -3.079 & 1062.6 & 1071.8 \\
 & -- & Yes & \underline{-35.891} & -40.984 & \underline{-3.170} & \underline{-3.494} & 1123.4 & 1174.1 \\
\midrule
\multirow{4}{*}{HFGuidedDesign} & No & No & -31.437 & -34.868 & -2.758 & -2.942 & 1128.1 & 1166.5 \\
 & No & Yes & -34.938 & -40.634 & -3.037 & -3.412 & 1138.1 & 1186.4 \\
 & Yes & No & -34.007 & -37.760 & -2.952 & -3.150 & \textbf{1149.0} & 1198.8 \\
 & Yes & Yes & -35.882 & \textbf{-42.806} & -3.121 & -3.447 & \underline{1145.7} & \textbf{1249.5} \\
\midrule
\multirow{4}{*}{\method{}} & No & No & -29.865 & -37.771 & -2.831 & -3.435 & 1045.1 & 1101.8 \\
 & No & Yes & -35.083 & -40.063 & -3.078 & -3.365 & 1133.3 & \underline{1202.0} \\
 & Yes & No & -32.730 & -36.181 & -2.982 & -3.338 & 1087.8 & 1082.9 \\
 & Yes & Yes & \textbf{-38.386} & \underline{-42.298} & \textbf{-3.396} & \textbf{-3.524} & 1091.5 & 1154.6 \\
\bottomrule
\end{tabular*}

\par\vspace{0.7em}

% Source: cofold_rosetta.tex
\noindent\makebox[\textwidth][l]{\textbf{(b) Hydrogen bonding and interface geometry}}\par
\vspace{3pt}
\begin{tabular*}{\textwidth}{@{\extracolsep{\fill}}lcc*{8}{r}@{}}
\toprule
\multirow{2}{*}{Method} & \multirow{2}{*}{Guide} & \multirow{2}{*}{\shortstack{AFCyc\\optim.}} & \multicolumn{2}{c}{\shortstack{Interface\\H-bonds $\uparrow$}} & \multicolumn{2}{c}{\shortstack{Unsat.\\H-bonds $\downarrow$}} & \multicolumn{2}{c}{SC $\uparrow$} & \multicolumn{2}{c}{Packstat $\uparrow$} \\
\cmidrule(lr){4-5}
\cmidrule(lr){6-7}
\cmidrule(lr){8-9}
\cmidrule(lr){10-11}
 &  &  & Mean & Best & Mean & Best & Mean & Best & Mean & Best \\
\midrule
\multirow{2}{*}{RFPeptides} & -- & No & 4.083 & 5.100 & 7.367 & \underline{6.183} & 0.653 & 0.685 & 0.635 & 0.640 \\
 & -- & Yes & 4.604 & 5.267 & \underline{6.650} & 6.567 & 0.684 & 0.689 & 0.644 & 0.654 \\
\midrule
\multirow{4}{*}{HFGuidedDesign} & No & No & 3.488 & 4.550 & 7.779 & 6.950 & 0.663 & 0.668 & 0.638 & 0.642 \\
 & No & Yes & \underline{4.729} & \underline{5.550} & 7.250 & 8.683 & 0.685 & 0.701 & 0.646 & \textbf{0.666} \\
 & Yes & No & 3.992 & 4.767 & 7.742 & 8.267 & 0.661 & 0.683 & 0.641 & 0.658 \\
 & Yes & Yes & 4.592 & 5.383 & 7.142 & 6.467 & \underline{0.686} & \textbf{0.709} & \underline{0.648} & \underline{0.661} \\
\midrule
\multirow{4}{*}{\method{}} & No & No & 3.358 & 4.667 & 7.875 & 7.550 & 0.664 & 0.698 & 0.644 & 0.655 \\
 & No & Yes & 4.354 & 5.200 & 7.454 & 7.483 & 0.681 & \underline{0.702} & 0.641 & 0.657 \\
 & Yes & No & 3.946 & 4.700 & 7.650 & 6.550 & 0.672 & 0.698 & 0.635 & 0.645 \\
 & Yes & Yes & \textbf{5.658} & \textbf{6.050} & \textbf{5.433} & \textbf{5.200} & \textbf{0.702} & 0.699 & \textbf{0.649} & \textbf{0.666} \\
\bottomrule
\end{tabular*}
\par\vspace{4pt}
\parbox{\textwidth}{\footnotesize\raggedright
\textit{Notes.} Bold and underlined values indicate the best and second-best results across the ten configurations, respectively, separately for each numeric column; ties share the same rank. Guide denotes generation-time guidance; AFCyc optim. denotes subsequent sequence-initialized gradient optimization.
}
\endgroup
\end{table}

%% file: tables/app_ablation_lnr.tex
% Preamble:
% \usepackage{booktabs,multirow}

\begin{table}[h]
\centering
\caption{\textbf{Ablations on LNR.}
The default $\tseq=0.5$ achieves the highest overall cyclization success
within each topology and the best HT Stability and Rosetta $dG$ scores.
Freezing backbone motion substantially worsens Stability for both topologies,
while removing the MPNN prior weakens interface scores and removing the delay
produces topology-dependent effects.
Earlier activation at $\tseq=0.2$ improves overall sequence and structure diversity,
indicating a trade-off between diversity and interface scores.
All/OOD/ID correspond to 93/37/56 receptors.}
\label{tab:app_ablation}

\begingroup
\normalsize
\setlength{\tabcolsep}{2pt}
\renewcommand{\arraystretch}{1.15}

% ---------- Cyclization, stability, and binding energy ----------
\begin{tabular*}{\textwidth}{@{\extracolsep{\fill}}ll*{9}{r}@{}}
\toprule
\multirow{2}{*}{Inference setting}
& \multirow{2}{*}{Topo.}
& \multicolumn{3}{c}{Cyc. (\%) $\uparrow$}
& \multicolumn{3}{c}{Stability $\downarrow$}
& \multicolumn{3}{c}{Rosetta $dG\downarrow$} \\
\cmidrule(lr){3-5}
\cmidrule(lr){6-8}
\cmidrule(lr){9-11}
& & All & \textsc{ood} & \textsc{id}
& All & \textsc{ood} & \textsc{id}
& All & \textsc{ood} & \textsc{id} \\
\midrule

\multirow{2}{*}{w/o coupling} & HT
& 98.5 & 97.1 & 99.5
& -109.23 & 112.60 & -255.79
& -9.59 & -4.54 & -12.92 \\

& SS
& 98.6 & 97.1 & 99.6
& -217.19 & -167.87 & -249.78
& -10.97 & -7.97 & -12.95 \\

\midrule
\multirow{2}{*}{w/o MPNN prior} & HT
& 98.8 & 97.3 & 99.8
& -715.25 & \underline{-914.12} & -583.86
& \underline{-21.42} & \underline{-27.56} & \underline{-17.36} \\

& SS
& 98.7 & 97.1 & 99.7
& -624.71 & -808.71 & -503.13
& -17.37 & -18.88 & -16.37 \\

\midrule
\multirow{2}{*}{w/o delay, $\tseq=0$} & HT
& 98.4 & 97.1 & 99.3
& -716.55 & -905.72 & -591.57
& -18.46 & -20.40 & -17.17 \\

& SS
& 98.7 & 97.2 & 99.7
& -657.66 & -838.82 & -537.96
& -17.99 & -19.61 & -16.92 \\

\midrule
\multirow{2}{*}{$\tseq=0.2$} & HT
& 98.9 & \textbf{97.6} & 99.8
& \underline{-720.54} & -913.84 & \underline{-592.83}
& -19.12 & -23.38 & -16.30 \\

& SS
& 98.9 & \underline{97.5} & \underline{99.9}
& -632.35 & -812.85 & -513.08
& -17.64 & -19.59 & -16.36 \\

\midrule
\multirow{2}{*}{$\tseq=0.5$} & HT
& \underline{99.0} & \underline{97.5} & \textbf{100.0}
& \textbf{-737.87} & \textbf{-940.05} & \textbf{-604.28}
& \textbf{-22.82} & \textbf{-28.37} & \textbf{-19.14} \\

& SS
& \textbf{99.1} & \textbf{97.6} & \textbf{100.0}
& -627.74 & -807.36 & -509.06
& -17.69 & -21.17 & -15.39 \\

\bottomrule
\end{tabular*}

\par\vspace{0.7em}

% ---------- Sequence and structural diversity ----------
\begin{tabular*}{\textwidth}{@{\extracolsep{\fill}}ll*{6}{r}@{}}
\toprule
\multirow{2}{*}{Inference setting}
& \multirow{2}{*}{Topo.}
& \multicolumn{3}{c}{$D_{\mathrm{seq}}\uparrow$}
& \multicolumn{3}{c}{$D_{\mathrm{struct}}\uparrow$} \\
\cmidrule(lr){3-5}
\cmidrule(lr){6-8}
& & All & \textsc{ood} & \textsc{id}
& All & \textsc{ood} & \textsc{id} \\
\midrule

\multirow{2}{*}{w/o coupling} & HT
& 0.869 & \underline{0.937} & 0.824
& 0.636 & 0.933 & 0.439 \\

& SS
& 0.382 & 0.639 & 0.213
& 0.957 & 0.989 & \underline{0.936} \\

\midrule
\multirow{2}{*}{w/o MPNN prior} & HT
& 0.875 & 0.933 & 0.836
& 0.634 & 0.948 & 0.427 \\

& SS
& 0.417 & 0.639 & 0.271
& 0.941 & 0.948 & \underline{0.936} \\

\midrule
\multirow{2}{*}{w/o delay, $\tseq=0$} & HT
& 0.871 & 0.934 & 0.829
& 0.647 & 0.940 & 0.453 \\

& SS
& 0.398 & 0.668 & 0.219
& 0.947 & 0.982 & 0.924 \\

\midrule
\multirow{2}{*}{$\tseq=0.2$} & HT
& \textbf{0.918} & \textbf{0.940} & \textbf{0.903}
& 0.727 & 0.934 & 0.590 \\

& SS
& 0.391 & 0.672 & 0.205
& \textbf{0.966} & \textbf{1.000} & \textbf{0.943} \\

\midrule
\multirow{2}{*}{$\tseq=0.5$} & HT
& \underline{0.883} & \textbf{0.940} & \underline{0.845}
& 0.648 & 0.934 & 0.459 \\

& SS
& 0.389 & 0.668 & 0.205
& \underline{0.962} & \underline{0.991} & \textbf{0.943} \\

\bottomrule
\end{tabular*}
\endgroup

\end{table}

%% file: tables/app_cofold_scaling.tex
\begin{table}[p]
\centering
\caption{\textbf{Detailed performance of \method{} with and without search across 12 targets.} Both configurations use generation-time guidance and AFCyc optimization. This experiment reports the hotspot-based composite score $S_{\mathrm{hotspot}}$.}
\label{tab:app_scaling_main}
\begingroup
\normalsize
\setlength{\tabcolsep}{2pt}
\renewcommand{\arraystretch}{1.12}
% Source: scaling_main.tex
\noindent\makebox[\textwidth][l]{\textbf{(a) Prediction confidence}}\par
\vspace{3pt}
\begin{tabular*}{\textwidth}{@{\extracolsep{\fill}}lccc*{4}{r}@{}}
\toprule
\multirow{2}{*}{Method} & \multirow{2}{*}{Guide} & \multirow{2}{*}{\shortstack{AFCyc\\optim.}} & \multirow{2}{*}{Search} & \multicolumn{2}{c}{ipTM $\uparrow$} & \multicolumn{2}{c}{Peptide pLDDT $\uparrow$} \\
\cmidrule(lr){5-6}
\cmidrule(lr){7-8}
 &  &  &  & Mean & Best & Mean & Best \\
\midrule
\multirow{2}{*}{\method{}} & Yes & Yes & No & \underline{0.651} & \underline{0.904} & \underline{57.90} & \underline{87.92} \\
 & Yes & Yes & Yes & \textbf{0.849} & \textbf{0.937} & \textbf{79.91} & \textbf{94.69} \\
\bottomrule
\end{tabular*}

\par\vspace{0.7em}

% Source: scaling_main.tex
\noindent\makebox[\textwidth][l]{\textbf{(b) Hotspot score and strict hits}}\par
\vspace{3pt}
\begin{tabular*}{\textwidth}{@{\extracolsep{\fill}}lccc*{6}{r}@{}}
\toprule
\multirow{2}{*}{Method} & \multirow{2}{*}{Guide} & \multirow{2}{*}{\shortstack{AFCyc\\optim.}} & \multirow{2}{*}{Search} & \multicolumn{2}{c}{Hotspot iPAE $\downarrow$} & \multicolumn{2}{c}{$S_{\mathrm{hotspot}}\uparrow$} & \multicolumn{2}{c}{Strict hits $\uparrow$} \\
\cmidrule(lr){5-6}
\cmidrule(lr){7-8}
\cmidrule(lr){9-10}
 &  &  &  & Mean & Best & Mean & Best & Count & Rate (\%) \\
\midrule
\multirow{2}{*}{\method{}} & Yes & Yes & No & \underline{0.241} & \underline{0.091} & \underline{1.989} & \underline{2.693} & \underline{125/384} & \underline{32.6} \\
 & Yes & Yes & Yes & \textbf{0.132} & \textbf{0.061} & \textbf{2.516} & \textbf{2.823} & \textbf{290/384} & \textbf{75.5} \\
\bottomrule
\end{tabular*}
\par\vspace{4pt}
\parbox{\textwidth}{\footnotesize\raggedright
\textit{Notes.} Bold and underlined values indicate the best and second-best results across the two configurations, respectively, separately for each numeric column; ties share the same rank. Best averages 12 candidates: select the highest-composite candidate among 32 for each target, then average its metrics across targets. Strict hits are counted over all 384 candidates and require $\mathrm{ipTM}\geq0.7$, peptide pLDDT $\geq70$, and interface iPAE $\leq0.3$.
}
\endgroup
\end{table}

\begin{table}[p]
\centering
\caption{\textbf{Detailed Rosetta interface metrics for \method{} with and without Search across 12 targets.} Both configurations use generation-time guidance and AFCyc optimization.}
\label{tab:app_scaling_rosetta}
\begingroup
\normalsize
\setlength{\tabcolsep}{2pt}
\renewcommand{\arraystretch}{1.12}
% Source: scaling_rosetta.tex
\noindent\makebox[\textwidth][l]{\textbf{(a) Interface energetics and buried surface area}}\par
\vspace{3pt}
\begin{tabular*}{\textwidth}{@{\extracolsep{\fill}}lccc*{6}{r}@{}}
\toprule
\multirow{3}{*}{Method} & \multirow{3}{*}{Guide} & \multirow{3}{*}{\shortstack{AFCyc\\optim.}} & \multirow{3}{*}{Search} & \multicolumn{2}{c}{$\Delta G_{\mathrm{sep}}\downarrow$} & \multicolumn{2}{c}{$100\Delta G_{\mathrm{sep}}/\Delta\mathrm{SASA}\downarrow$} & \multicolumn{2}{c}{$\Delta\mathrm{SASA}\uparrow$} \\
 &  &  &  & \multicolumn{2}{c}{(REU)} & \multicolumn{2}{c}{(REU/$100\,\text{\AA}^{2}$)} & \multicolumn{2}{c}{($\text{\AA}^{2}$)} \\
\cmidrule(lr){5-6}
\cmidrule(lr){7-8}
\cmidrule(lr){9-10}
 &  &  &  & Mean & Best & Mean & Best & Mean & Best \\
\midrule
\multirow{2}{*}{\method{}} & Yes & Yes & No & \underline{-38.386} & \underline{-42.298} & \underline{-3.396} & \underline{-3.524} & \underline{1091.5} & \underline{1154.6} \\
 & Yes & Yes & Yes & \textbf{-38.841} & \textbf{-43.514} & \textbf{-3.567} & \textbf{-3.724} & \textbf{1107.4} & \textbf{1180.5} \\
\bottomrule
\end{tabular*}

\par\vspace{0.7em}

% Source: scaling_rosetta.tex
\noindent\makebox[\textwidth][l]{\textbf{(b) Hydrogen bonding and interface geometry}}\par
\vspace{3pt}
\begin{tabular*}{\textwidth}{@{\extracolsep{\fill}}lccc*{8}{r}@{}}
\toprule
\multirow{2}{*}{Method} & \multirow{2}{*}{Guide} & \multirow{2}{*}{\shortstack{AFCyc\\optim.}} & \multirow{2}{*}{Search} & \multicolumn{2}{c}{\shortstack{Interface\\H-bonds $\uparrow$}} & \multicolumn{2}{c}{\shortstack{Unsat.\\H-bonds $\downarrow$}} & \multicolumn{2}{c}{SC $\uparrow$} & \multicolumn{2}{c}{Packstat $\uparrow$} \\
\cmidrule(lr){5-6}
\cmidrule(lr){7-8}
\cmidrule(lr){9-10}
\cmidrule(lr){11-12}
 &  &  &  & Mean & Best & Mean & Best & Mean & Best & Mean & Best \\
\midrule
\multirow{2}{*}{\method{}} & Yes & Yes & No & \textbf{5.658} & \textbf{6.050} & \textbf{5.433} & \textbf{5.200} & \underline{0.702} & \underline{0.699} & \underline{0.649} & \underline{0.666} \\
 & Yes & Yes & Yes & \underline{5.156} & \underline{5.500} & \underline{5.807} & \underline{5.583} & \textbf{0.729} & \textbf{0.730} & \textbf{0.664} & \textbf{0.678} \\
\bottomrule
\end{tabular*}
\par\vspace{4pt}
\parbox{\textwidth}{\footnotesize\raggedright
\textit{Notes.} Bold and underlined values indicate the best and second-best results across the two configurations, respectively, separately for each numeric column; ties share the same rank. 
}
\endgroup
\end{table}

%% file: tables/app_tl1a.tex
We study TL1A as a challenging target for cyclic-peptide design against
an oligomeric protein--protein interface.
TL1A (TNFSF15) is a TNF-superfamily ligand recognized by the signaling
receptor DR3 and the soluble decoy receptor DcR3~\citep{migone2002tl1a}.
The TL1A--DcR3 crystal structure reveals a trimeric TL1A core bound by
three DcR3 molecules~\citep{zhan2011decoy}.
Our objective is to design cyclic peptides that occupy the DcR3-binding
surface of trimeric TL1A, with the intended function of competitively
blocking TL1A--DcR3 association.
Given the trimeric target and designated interface hotspots, the task is
to generate a peptide sequence and cyclic conformation that can engage
this surface in the assembled complex.

This case combines several geometric and modeling challenges.
The designated hotspots span two TL1A chains, so recognizing the target
requires a peptide pose compatible with a composite surface whose geometry
depends on the relative arrangement of neighboring subunits.
The peptide must establish contacts across this surface while avoiding
steric clashes with the surrounding trimer.
Meanwhile, the long target chains produce a large total residue context,
requiring precise peptide localization within a substantially larger assembly.
Cyclization adds a further constraint: the peptide must adopt a closed
conformation while positioning its interacting residues to engage hotspots
on both chains.
Together, these requirements make the case a demanding test of whether
cofolding can jointly resolve peptide conformation, interface placement,
and compatibility with an oligomeric target.

For all evaluated configurations, we design cyclic peptides of length 12
residues and generate 12 candidate cyclic peptides per model.
For \method{} with inference-time search, the search procedure is configured
with parameters $K=8$, $K_0=4$, $M=6$ (yielding $N=24$), and $N_0=12$.

Table~\ref{tab:tl1a_case} compares RFPeptides, HFGuidedDesign, \method{},
and \method{} with inference-time search.
Across all methods, ipTM values remain closely clustered at a high level.
However, this uniformity obscures substantial uncertainty in peptide folding:
all three baseline configurations without search yield consistently low peptide
pLDDT, indicating that a favorable target-wide ipTM alone does not guarantee a
well-resolved peptide structure.
In contrast, incorporating inference-time search leads to pronounced improvements
across the board, yielding the highest peptide pLDDT, the lowest interface and
hotspot iPAE, and the strongest hotspot scores under both Mean and AF Top-1
summaries.

These trends demonstrate that inference-time search effectively guides the model
toward more confident and geometrically consistent solutions on this multi-chain
interface.
Nevertheless, the absolute peptide confidence remains modest even after search,
reflecting the inherent difficulty of the oligomeric target.
While these computational gains indicate improved structural plausibility,
experimental binding and competition assays will be essential to verify
functional TL1A--DcR3 blockade.

\begin{table}[!ht]
\centering
\small
\caption{\textbf{Cofolding results on the challenging TL1A-trimer case.}
Each cell reports Mean / AF Top-1.
Search improves peptide confidence, interface iPAE, and hotspot-specific
metrics, while mean ipTM remains similar across methods.}
\label{tab:tl1a_case}

\begingroup
\setlength{\tabcolsep}{3pt}
\renewcommand{\arraystretch}{1.18}

\begin{tabular*}{\textwidth}{
    @{\extracolsep{\fill}}l*{5}{r}@{}
}
\toprule
Method
& ipTM $\uparrow$
& \shortstack{Peptide\\pLDDT $\uparrow$}
& \shortstack{Interface\\iPAE $\downarrow$}
& \shortstack{Hotspot\\iPAE $\downarrow$}
& \shortstack{Hotspot\\$S_{\mathrm{hotspot}}\uparrow$} \\
\midrule

RFPeptides
& \textbf{0.868} / \underline{0.873}
& 26.82 / 30.38
& \underline{0.808} / 0.725
& 0.751 / 0.623
& \underline{1.385} / 1.555 \\

HFGuidedDesign
& \underline{0.867} / 0.868
& 25.52 / 27.44
& 0.818 / 0.790
& \underline{0.749} / 0.705
& 1.374 / 1.438 \\

\midrule

\textbf{\method}
& \textbf{0.868} / \underline{0.873}
& \underline{26.96} / \underline{33.42}
& 0.810 / \underline{0.689}
& 0.753 / \underline{0.583}
& 1.384 / \underline{1.624} \\

\textbf{\method{} (Search)}
& 0.864 / \textbf{0.877}
& \textbf{38.41} / \textbf{49.79}
& \textbf{0.677} / \textbf{0.632}
& \textbf{0.576} / \textbf{0.524}
& \textbf{1.673} / \textbf{1.851} \\

\bottomrule
\end{tabular*}

\par\vspace{3pt}
\parbox{\textwidth}{\footnotesize\raggedright
Bold and underline indicate the best and second-best values,
respectively, separately for each metric and summary;
ties share the same rank.
AF Top-1 follows the AF composite-score selection used above
and is not selected independently for each metric.
}
\endgroup
\end{table}